\documentclass[12pt,letterpaper,fleqn]{article}
\usepackage{etex,authblk,amsmath,fullpage,amsfonts,titletoc,amssymb,theorem}
\usepackage[onehalfspacing]{setspace}
\usepackage{graphicx,hyperref,pgfplots,dcolumn,ctable,floatpag}

\hypersetup{colorlinks=true,
            linkcolor=blue,
            citecolor=blue,
            urlcolor=blue}

\usetikzlibrary{pgfplots.groupplots}
\usetikzlibrary{patterns, arrows}
\usetikzlibrary{arrows,positioning,shadows,shapes,calc,fit,backgrounds}
\tikzset{>=stealth}

\reversemarginpar

\floatpagestyle{empty}

\makeatletter
\def\pgfplots@drawtickgridlines@INSTALLCLIP@onorientedsurf#1{}
\makeatother

\usepackage[font=footnotesize,labelfont=bf,labelsep=period,justification=justified]{caption}

\usepackage{natbib}
         \makeatletter
          \renewcommand{\bibsection}{
          \begin{center}
  \section*{\refname\@mkboth{\MakeUppercase{\refname}}
  {\MakeUppercase{\refname}}}
              \end{center}
              }
          \makeatother

\newtheorem{theorem}{Theorem}

\newtheorem{assumption}{Assumption}

\newtheorem{corollary}{Corollary}

\newtheorem{lemma}{Lemma}

\newtheorem{proposition}{Proposition}
\newtheorem{theorem-app}{Theorem}[section]
\newtheorem{lemma-app}[theorem-app]{Lemma}
\newtheorem{proposition-app}[theorem-app]{Proposition}
\theorembodyfont{\upshape}

\theoremstyle{definition}
\newtheorem{definition}{Definition}

\newenvironment{proof}[1][\proofname]{
                \par\normalfont\trivlist\item[\hskip\labelsep\textbf{#1}.]\ignorespaces}
                {\hfill Q.E.D.\endtrivlist}
\newcommand{\proofname}{Proof}

\newcommand{\gD}{\Delta}
\newcommand{\ga}{\alpha}

\newcommand{\gl}{\lambda}
\newcommand{\eps}{\varepsilon}
\newcommand{\gd}{\delta}

\newcommand{\R}{{\mathbb R}}

\newcommand{\N}{{\mathbb N}}

\newcommand{\Z}{{\mathbb Z}}

\newcommand{\gO}{\Omega}
\newcommand{\go}{\omega}

\newcommand{\diag}{\operatorname{diag}}

\newcommand{\supp}{{\rm supp}}
\newcommand{\cI}{{\mathcal I}}

\newcommand{\cP}{\mathcal{P}}
\newcommand{\cD}{\mathcal D}
\newcommand{\cR}{\mathcal{R}}

\newcommand{\cA}{{\mathcal A}}
\newcommand{\cG}{{\mathcal G}}
\newcommand{\cK}{{\mathcal K}}
\newcommand{\cB}{{\mathcal B}}

\newcommand{\Supp}{{\operatorname{Supp}}}
\newcommand{\Prob}{\operatorname{Prob}}
\newcommand{\cC}{\mathcal C}

\usepackage{bbm}
\usepackage{color}

\usepackage{mathtools}
\mathtoolsset{showonlyrefs}

\newcommand{\be}{\begin{equation}}

\newcommand{\ee}{\end{equation}}
\newcommand{\bea}{\begin{eqnarray}}
\newcommand{\eea}{\end{eqnarray}}
\newcommand{\bee}{\begin{equation*}}
\newcommand{\eee}{\end{equation*}}

\begin{document}


\title{Persuasion by Dimension Reduction\thanks{We thank Philip Bond, Darrell Duffie, Piotr Dworczak, Egemen Eren, Bart Lipman,  Jean-Charles Rochet, and Stephen Morris (AEA discussant) as well as seminar participants at Caltech, UBC, SUFE, SFI, and conference participants at the 2020 AEA meeting in San Diego  for their helpful comments. Parts of this paper were written when Malamud visited the Bank for International Settlements (BIS) as a research fellow. The views in this article are those of the authors and do not necessarily represent those of BIS.}}

\author{Semyon Malamud\thanks{%
Swiss Finance Institute, EPF Lausanne, and CEPR; E-mail: \texttt{semyon.malamud@epfl.ch%
}} and Andreas Schrimpf\thanks{Bank of International Settlements (BIS) and CEPR; E-mail: \texttt{andreas.schrimpf@bis.org%
} }\\
 }

\date{This version: \today}

\maketitle


\begin{abstract} \noindent  
How should an agent (the sender) observing multi-dimensional data (the state vector) persuade another agent to take the desired action? We show that it is always optimal for the sender to perform a (non-linear) dimension reduction by projecting the state vector onto a lower-dimensional object that we call the ``optimal information manifold." We characterize geometric properties of this manifold and link them to the sender's preferences. Optimal policy splits information into ``good" and ``bad" components. When the sender's marginal utility is linear,  it is always optimal to reveal the full magnitude of good information. In contrast, with concave marginal utility, optimal information design conceals the extreme realizations of good information and only reveals its direction (sign). We illustrate these effects by explicitly solving several multi-dimensional Bayesian persuasion problems.

\vspace{5pt}
\noindent
\textbf{Keywords}:  Bayesian Persuasion, Information Design, Signalling, Learning\\
\textbf{JEL}: D82, D83, E52, E58, E61

\vspace{5pt}
\end{abstract}

\renewcommand{\thefootnote}{\number\value{footnote}}

\pagenumbering{arabic}
\def\baselinestretch{1.617}\small\normalsize%

\clearpage

\section{Introduction}

Bayesian persuasion -- that is, the optimal information design when the sender has full commitment power -- can be used to model communication and disclosure policies in numerous economic settings. Applications include the design of school grades, credit ratings, processes of criminal investigations, disclosure rules, risk monitoring, central bank communication, and even traffic regulations.\footnote{See, e.g., \cite{kamenica2019Bayesian} for an overview.} In fact, according to recent studies, 30\% of the US GDP is accounted for by persuasion.\footnote{See, \cite{mccloskey1995one} and \cite{antioch2013persuasion}.} In most of the practical applications, the sender has access to high-dimensional data that she needs to transform into a signal. For example, a school compresses the vector of a student's test results into a single grade; a prosecutor needs to communicate numerous dimensions of the collected evidence to the judge; a firm may disclose multiple characteristics of a good as well as results of various tests; a policymaker needs to compress vast amounts of macroeconomic information into simple, digestible signals. 

Although Bayesian persuasion for small, finite state spaces is well understood (see, \cite{KamGenz2011}), little is known about the nature of optimal policies for the realistic case of large, multi-dimensional state spaces. 
In this paper, we develop a novel, geometric approach to Bayesian persuasion when the state space is continuous and multi-dimensional. We show how the continuity assumption brings tractability into Bayesian persuasion in the same way as the continuous-time assumption brings tractability into discrete-time models. Many optimality conditions become explicit and have a clear interpretation due to our ability to take derivatives along different directions in the continuous state space. 

We show that it is always optimal for the sender to perform a (non-linear) dimension reduction by projecting the state onto a lower-dimensional object that we call the ``optimal information manifold." We characterize the geometric properties of this manifold and link them to the sender's preferences. 

Consider first the case when the state observed by the sender is continuous and one-dimensional. If a full revelation is not optimal, the sender conceals information by pooling multiple values into a single signal value $s.$ This set of states is the ``pool" of the signal $s.$ Under technical conditions (see, e.g., \cite{rayo2013monopolistic} and \cite{DworczakMartini2019}), the pool of every signal $s$ is an interval: The line can be partitioned into a union of intervals such that each interval is either fully pooled into a single value, or is fully revealed. In particular, nontrivial persuasion policies always involve pooling sets of {\it positive measure} into a single signal. Recent results (see \cite{Arielietal2020}) show that one-dimensional persuasion always involves pooling intervals of positive measure. 

This intuition, however, no longer holds in multiple dimensions.\footnote{There is an interesting parallel between these effects that those in the multi-dimensional signaling problems, where the behavior differs drastically between one- and multi-dimensional cases. See, \cite{rochet1998ironing}.} For example, in two dimensions, when the sender's {\it utility is quadratic} in the receiver's action, and the latter is given by the expected state, \cite{rochet1994insider}, \cite{RayoSegal2010}, \cite{tamura2018Bayesian}, and \cite{DworzakKolotilin2019} provide examples showing that the set of possible signal values (the {\it support} of the optimal policy) is a {\it curve} in $\R^2.$ In fact, this support curve is the graph of a monotone function, and the pool of every signal is a subset of a line (\cite{RayoSegal2010}) and, hence, has Lebesgue measure zero.\footnote{See, Figure \ref{fig:curve} in Section \ref{sec:moment}.} 

In the language of our paper, this curve is a one-dimensional optimal information manifold. The monotonicity properties of this manifold have important implications for the nature of optimal disclosure. First, it means that signals are always ordered. Second, the policy effectively splits directions of information into ``good" and ``bad," and better signals correspond to states with a larger good information component. This is the essence of information compression through dimension reduction. Good directions are those tangent to the support curve, and the magnitude of good information is always fully revealed in the sense that a state with a large magnitude of the good information component corresponds to a distant point on the curve (see, Figure \ref{fig:curve} in Section \ref{sec:moment}). Bad directions, in contrast, are parallel to the pool lines; these lines are downward sloping: Less
relevant but more attractive states are pooled with more relevant but less attractive ones (\cite{RayoSegal2010}). 
Although limited to a single, specific, two-dimensional, quadratic example, these observations raise a question: Are these properties characteristic for optimal information design in multiple dimensions? In this paper, we study this question and derive properties of optimal information design for general utility functions and distributions that are absolutely continuous with respect to the Lebesgue measure. This opens up a road to numerous applications to real-world persuasion problems. 

Our paper has three core findings. {\it First,} we show that the set of possible signal values (the support of an optimal policy) is a multi-dimensional analog of a curve that we refer to as the optimal information manifold.\footnote{A curve in $\R^2$ is a one-dimensional manifold.} The dimension of this manifold equals the number of directions in which the sender is risk-loving.\footnote{Formally, it is the number of positive eigenvalues of the Hessian (the matrix of second-order partial derivatives) of sender's utility.} 

{\it Second,} in stark contrast with the one-dimensional case, the pool of every signal has measure zero and is itself a manifold that we characterize explicitly. When the receivers' actions are given by the expected state, each pool is a line segment or a convex subset of a hyper-plane. The sign of the slope of this hyperplane determines the nature of states that are pooled together. Contrary to the simple quadratic case studied in the previous literature, this sign is determined by the geometry of the optimal information manifold and may exhibit nontrivial patterns that we describe in the paper. 

{\it Third}, there is a strong difference between optimal persuasion for linear (considered in the vast majority of papers on Bayesian persuasion) and non-linear marginal utilities of the sender. With linear marginal utilities, the optimal information design {\it always} reveals the magnitude of ``good" information. By contrast, with concave marginal utilities (e.g., when the utility function is a fourth-order, concave polynomial), 
it is always optimal to conceal extreme situations. 

To deal with the general persuasion problem, we develop novel mathematical techniques. We start by solving a coarse communication problem in which the sender is constrained to a finite set of signals. We prove a purification result, showing that a pure optimal policy always exists, given by a partition of the state space. When the receivers' actions are functions of the expected state, we derive sufficient conditions for the partition to be given by convex polygons, a natural analog of a monotone partition in many dimensions.\footnote{\cite{KleinbergMullainathan2019} argue that clustering (partitioning the state space into discrete cells) is the most natural way to simplify information processing in complex environments. Our results provide a theoretical foundation for such clustering. Note that, formally, in a Bayesian persuasion framework, economic agents (signal receivers) would need to use (potentially complex) calculations underlying the Bayes rule to compute the conditional probabilities. One important real-world problem arises when the receivers do not know the "true" probability distribution, in which case methods from robust optimization need to be used. See \cite{DworzakPavan2020}.}
Our characterization of optimal partitions allows us to take the continuous limit and show that these partitions converge to a solution to the unconstrained problem. We establish a surprising connection between optimal information design and the Monge-Kantorovich theory of optimal transport whereby the sender effectively finds an optimal way of ``transporting information" to the receiver, with an endogenous {\it information transport cost}. 

In the case when receivers' actions are a function of expectations about (multiple, arbitrary) functions of the state, optimal information design is given by an explicit {\it projection onto the optimal information manifold}. Each state is projected onto the point on the manifold with the minimal information transport cost. We use metric geometry and the theory of the Hausdorff dimension to show that the manifold is ``sufficiently rich" so that formal first-order conditions can be used to characterize optimal pools. This characterization allows us to describe analytically, which states are optimally pooled together, and derive the monotonicity properties of these pools. We then apply our general results to several classic persuasion models and show how multiple dimensions lead to surprising findings with no analogs in the one-dimensional case. 

The paper is organized as follows. Section \ref{sec:lit} reviews the literature. Section \ref{sec:model} describes the model. Section \ref{sec:pure pol} proves the existence of pure policies and links them to optimal transport. Section \ref{sec:moment} studies moment persuasion. Section \ref{sec:oim} introduces optimal information manifolds and their geometric properties, and characterizes optimal pools. Section \ref{sec-applications} contains applications. Section \ref{sec:concl} concludes the paper. 
\vspace{5pt}

\section{Literature Review}\label{sec:lit}

We study the general problem of optimal information design, phrased as a ``Bayesian persuasion" problem
between a sender and a receiver in the influential paper by \cite{KamGenz2011}. A large literature fits this topic, including important contributions of \cite{aumann1995}, \cite{pavan2006}, \cite{brocas2007influence}, \cite{RayoSegal2010}, and \cite{ostrovsky2010}.
The term ``information design" was introduced in \cite{RePEc:edn:esedps:256} and \cite{bergemann2016information}. See \cite{BergMor2017} and \cite{kamenica2019Bayesian} for excellent reviews. 
Most existing papers on Bayesian persuasion with continuous signals consider the case of a one-dimensional signal space. See, for instance, \cite{rayo2013monopolistic}, \cite{gentzkow2016rothschild}, \cite{Kolotilin2018}, \cite{Hopenhayn2019}, \cite{DworczakMartini2019}, \cite{Arielietal2020} and \cite{kleiner2020extreme}. For example, \cite{DworczakMartini2019} derive necessary and sufficient conditions for the optimality of monotone partitions when the utility of the sender is a function of the expected state. \cite{Arielietal2020} show that, in general, monotone partitions are not sufficient, but any one-dimensional persuasion problem admits a solution in a class of ``bi-pooling" policies that involve a small degree of randomization. 

However, little is known about optimal information design in the multidimensional case. The only results we are aware of are due to \cite{rochet1994insider}, \cite{RayoSegal2010}, \cite{DworczakMartini2019}, \cite{DworzakKolotilin2019}, and \cite{tamura2018Bayesian}.  
\cite{rochet1994insider}, \cite{RayoSegal2010}, and \cite{kramkov2019optimal}\footnote{In the growing literature on martingale optimal transport in mathematical finance, the classic Monge-Kantorovich optimal transport is studied under the constraint that the target is a conditional expectation of the origin. See, \cite{beiglbock2016problem} and \cite{ghoussoub2019structure}.  As \cite{rochet1994insider} show, such problems are also related to a special class of signaling games.}
 consider the case when the sender's utility is quadratic, the state is two-dimensional, and the action is given by the expected state. They show that the pool of every signal is a (potentially non-convex) subset of a line, and the support of the signal distribution is a curve; in fact, it is the graph of a monotone function. \cite{DworzakKolotilin2019} consider a special case of this problem when this curve is a line. None of these papers give general insights about the structure of optimal solutions beyond the special, quadratic example. It is not even clear what the multi-dimensional analog of monotone partitions would be,  and whether an analog of the results of \cite{DworczakMartini2019} exists in many dimensions. We show that such a natural analog is given by partitions into convex signal pools, and derive sufficient conditions for this convexity. In stark contrast to the one-dimensional case, pools have Lebesgue measure zero, meaning that, in multiple dimensions, pooling strategies are much more granular. 
 
To the best of our knowledge, \cite{tamura2018Bayesian} is the only existing paper solving a persuasion problem in more than two dimensions. He
considers Bayesian persuasion with quadratic preferences and a Gaussian prior. He shows the existence of a linear optimal information design, given by a linear projection onto a linear subspace. The results of \cite{tamura2018Bayesian} depend crucially on the assumptions of quadratic preferences and a Gaussian prior. We show that the linear policy from \cite{tamura2018Bayesian} is always optimal when the prior is elliptic and that, in fact, this optimal policy is unique. This non-existence of non-linear policies is an important implication of our general results. 

One of the key applications of Bayesian persuasion is to the problem of firms supplying product information to customers, considered in both \cite{KamGenz2011} and \cite{RayoSegal2010}. In our paper, we focus on the case in \cite{RayoSegal2010} where the value of the product for the seller is correlated with that for the buyer.  \cite{RayoSegal2010} assume that the state space is discrete and characterize several important properties of an optimal information design. In particular, they show that the pool of every signal is a discrete subset of a line in the space of two-dimensional prospects (expected profitability for the sender and the receiver). Many of the results in \cite{RayoSegal2010} only hold when the distribution of the opportunity cost, $G$, is uniform; for example, in this case the pools are downward sloping. They emphasize that {\it ``little can be said about the optimal pooling
graph for arbitrary $G$."} We illustrate the power of our approach and the convenience of continuous state spaces by fully characterizing the optimal policy for any $G$ and showing that the intuition from discrete state spaces can be misleading. In their conclusion section, \cite{RayoSegal2010} suggest two natural multi-dimensional extensions of their framework: (i) different receiver types and (ii) a sender endowed with multiple prospects. To the best of our knowledge, no progress has been made in these directions thus far, perhaps because tackling these problems with discrete states seems to be a daunting task. We illustrate the power of our approach by characterizing solutions to both problems (i) and (ii) and deriving the monotonicity properties of optimal signal pools and signals' support (the optimal information manifold). In particular, we compute the ``dimension of information revealed" and link it to the convexity properties of customers' acceptance functions. We show how conventional wisdom may break down\footnote{ \cite{RayoSegal2010} write: {\it ``we would expect an additional reason
for hiding information, which occurs in models of optimal bundling
with heterogeneous consumers"}. Our results imply that, on the contrary,  multiple customer types often increase incentives to reveal information.} and how almost full revelation may be optimal with many customer types.  

Like \cite{KamGenz2011}, \cite{RayoSegal2010} find that allowing the sender to randomize information when sending signals is crucial for tractability. They write: 
{\it ``We believe that such randomization would become unnecessary with a continuous,
convex-support distribution of prospects, but the full analysis of such a case is considerably
more challenging." } Our results confirm this intuition. First, perhaps surprisingly, convex support turns out to be unnecessary: Pure policies always exist when the prior is absolutely continuous with respect to the Lebesgue measure. Second, the analysis is indeed challenging and has required developing novel mathematical techniques from several areas of mathematics, including {\it real analytic functions, differential geometry and partial differential equations, metric geometry, optimal transport, and Hausdorff dimension.} However, although our proofs are sometimes involved, the final outcome is an explicit characterization that can be directly used in almost any multi-dimensional persuasion problem. 

Finally, we note that our solution to the problem with a finite signal space relates this paper to the literature on optimal rating design (see, e.g., \cite{Hopenhayn2019}). Indeed, in practice, most ratings are discrete. For example, credit rating agencies use discrete rating buckets (e.g., above BBB-); restaurant and hotel ratings take a finite number of values. Imposing finiteness of the signal space is natural for many real-world applications. \cite{aybas2019persuasion} study persuasion with coarse communication and a large but discrete state space and derive bounds for utility loss due to finiteness of the signal space. We hope that our results related to the optimality of partitions with coarse communication will find more applications in this literature.

\section{Model}\label{sec:model}

The state space $\gO$ is a (potentially unbounded) open subset of $\R^L$. We use $\Delta(\gO)$ to denote the set of Borel probability measures on $\gO$. The prior distribution has a density $\mu_0(\go)$ with respect to the Lebesgue measure on $\gO,$ with $\int_\gO \mu_0(\go)d\go=1.$ The information designer (the sender) observes $\go$ and sends a signal to the receivers. The receivers use the Bayes rule to form a posterior $\mu$ after observing the signal of the sender, and then take an action $a\in \R^M.$ Conditional on $\go$ and $a,$ the sender's utility is given by $W(a,\go).$ We use $D_a$ to denote the derivative (gradient) with respect to $a$ and, similarly, $D_{aa}$ is the second order derivative (Hessian). 

We will make the following technical assumption about the link between the posterior and the actions of the receivers.\footnote{In Appendix \ref{der-g}, we show how to derive the map $G$ from a utility maximization problem of the receivers, in which case $G$ is the vector of receivers' marginal utilities.}

\begin{assumption}\label{ac} There exists a function $G:\ \R^{M}\times \gO\to\R^{M}$ 
such that the optimal action $a\ =\ a(\mu)$ of the receivers with a posterior $\mu$ satisfies 
\begin{equation}\label{sys-3}
\int G(a(\mu),\go)d\mu(\go)\ =\ 0
\end{equation}
Furthermore, $G$ satisfies the following conditions:  
\begin{itemize}
\item $G$ is continuously differentiable in $a$. 

\item $G$ is uniformly monotone in $a$ for each $\go$ so that $\eps\|z\|^2\ \le -z^\top D_aG(a,\go)z\le\eps^{-1}\|z\|^2$ for some $\eps>0$ and all $z\in \R^M;$\footnote{Strict monotonicity is important here. Without it, there could be multiple equilibria.} 

\item the unique solution $a_*(\go)$ to $G(a_*(\go),\go)=0$ is square integrable: $E[\|a_*(\go)\|^2]<\infty.$
\end{itemize}
\end{assumption}

Assumption \ref{ac} implies that the following is true:

\begin{lemma} \label{existence} For any posterior $\mu$, there exists a unique action $a=a(\mu)$ satisfying \eqref{sys-3} and 
$\|a(\mu)\|^2\ \le\ \kappa \int_\gO \|a_*(\go)\|^2d\mu(\go)$ for some universal $\kappa>0.$
\end{lemma}

We will also need the following technical assumption about the sender's utility.

\begin{assumption}\label{integrability} $W$ is jointly continuous in $(a,\go)$ and is continuously differentiable with respect to $a.$ Furthermore, there exists a function $g:\gO\to \R_+$ such that $g(\go)\ge \|a_*(\go)\|^2$ and the set $\{\go:g(\go)\le A\}$ is compact for all $A>0,$ and a convex, increasing function $f\ge 1$ such that  $|W(a,\go)|+\|D_aW(a,\go)\|\ \le\ g(\go) f(\|a\|^2)$ and $E[g^2(\go) f(g(\go))]<\infty.$
\end{assumption}

Following \cite{KamGenz2011}, we assume that the sender is able to \emph{commit} to an information design before the state $\go$ is realized. As \cite{KamGenz2011} show, the {\it persuasion} (optimal information design) problem of the sender is equivalent to choosing a distribution $\tau$ of posterior beliefs, $\tau\in \Delta(\Delta(\gO)).$\footnote{For example, if the sender sends one of the three signals, $s_1,s_2,s_3,$ with probabilities $p_1,p_2,p_3,$ let $\mu_i\in \gD(\gO)$ be the posterior distribution of $\go$ conditional on $s_i.$ This information design is then equivalent to the distribution $\tau$ on $\gD(\gO),$ with a support of three points, $\mu_1,\mu_2,\mu_3\in \gD(\gO),$ with $\mu_i$ occurring with probability $p_i.$ Hence, $\tau$ is a distribution on posterior distributions, $\tau\in \gD(\gD(\gO)).$ }
We formally state the optimal information design problem in the following definition. 

\begin{definition}\label{dfn1} Let 
\[
\bar W(\mu)\ =\ \int_\gO W(a(\mu),\go)d\mu(\go)\,.
\]
be the expected utility of the sender conditional on a posterior $\mu$, with $a(\mu)$ defined in \eqref{sys-3} 
The optimal Bayesian persuasion (optimal information design) problem is to maximize 
\[
\int_{\gD(\gO)} \bar W(\mu)d\tau(\mu)
\]
over all distributions of posterior beliefs $\tau\in\Delta(\Delta(\gO))$ satisfying 
\[
\int_{\gD(\gO)} \mu d\tau(\mu)\ =\ \mu_0\,. 
\]
We denote the value of this problem by $V(\mu_0)\,.$ A solution $\tau$ (a distribution of posterior beliefs) to this problem is called an optimal information design. We say that an information design is pure (does not involve randomization) if there exists a map $a:\ \gO\to \R^M$ such that $\tau$ is induced by this map. That is, the distribution of $\tau$ coincides with that of $\{\mu_a:\ a\in \R^M\}$ where 
\[
\mu_a(\go)\ =\ \Prob(\go|\ a(\go)\ =\ a),\ a\in \R^M
\]
is the posterior after observing the realization of the signal $a(\go)$.  A pure information design where the signal $a(\go)$ coincides with the optimal action of the receivers, such that
\[
\int G(a,\go)d\mu_a(\go)\ =\ E[G(a,\go)|\ a(\go)=a]\ =\ 0\,
\]
for all $a\in \R^M,$ will be referred to as an optimal policy.
\end{definition}

A pure information design is an intuitive form of Bayesian persuasion whereby the signal sent by the sender is a deterministic function $a(\go)$ of the state.\footnote{For example, $\go=\binom{\go_1}{\go_2}$, and the sender reveals only $a(\go)=0.5(\go_1+\go_2).$} In the case when $\gO$ is finite, it is straightforward to show that, without loss of generality, we may always assume that $a(\go)$ coincides with the action taken by the receivers: The sender can just directly recommend the desired action for each realization of $\go.$\footnote{If randomization is optimal, the sender randomly selects a recommended action from a set.} See \cite{KamGenz2011}. However, when the state space is large (as in our setting), proving this result requires additional effort. 

In the Appendix,  we prove a ``purification" result: Pure policies always exist. This existence will be crucial for our subsequent analysis, significantly limiting the search space for optimal policies.\footnote{In general, pure policies are often preferable. For example, as \cite{kamenica2021bayesian} argue, ``A commitment to randomized messages is difficult to verify and enforce."}  Our argument is non-trivial and proceeds as follows. 
First, we consider a discretization\footnote{This discretization corresponds to the case of coarse communication. See, for example, \cite{aybas2019persuasion}.} of the basic problem of Definition \ref{dfn1}, imposing the constraint that the support of $\tau$ in Definition \ref{dfn1} is finite, and show that a discrete pure optimal policy always exists. Second, we take the continuous limit and prove convergence.  Our proof of the existence of pure policies in the discrete case is non-standard and is based on the theory of real analytic functions. 

\section{Pure Optimal Policies} \label{sec:pure pol}

We will use $\Supp(a)$ to denote the support of any map $a:\R^L\to\R^M:$ 
\[
\Supp(a)\ =\ \{x\in \R^M: \mu_0(\{\go:\ \|a(\go)-x\|<\eps\})>0\ \forall\ \eps>0\}\,. 
\]
We will need the following definition. 

\begin{definition}\label{full-def} Recall that $a_*(\go)$ is the unique solution to $G(a_*(\go),\go)=0$ (see Assumption \ref{ac}). For any map $x:\R^M\to \R^M,$ we define 
\begin{equation}\label{c-function}
c(a,\go;x)\ \equiv\ W(a_*(\go),\go)-W(a,\go)\ +\ x(a)^\top G(a,\go)\,.
\end{equation}
Everywhere in the sequel, we refer to $c$ as the {\it cost of information transport.}
\end{definition}

To gain some intuition behind the cost $c,$ we note that $W(a_*(\go),\go)$ is the sender's utility attained by revealing that the true state is $\go.$ Thus, $W(a,\go)-W(a_*(\go),\go)$ is the utility  gain from inducing a different (preferred) action $a$ and $x(a)^\top G(a,\go)$ is the corresponding shadow cost of agents' participation constraints. The total cost of information transport is the sum of the true and the shadow costs of ``transporting" information from $a_*(\go)$ to $a.$ 

As the utility attained by full revelation, $E[W(a_*(\go),\go)],$ is independent of the information design, maximizing expected utility  $E[W(a(\go),\go)]$ is equivalent to minimizing $E[W(a_*(\go),\go)-W(a(\go),\go)].$ From now on, we will be considering this equivalent formulation of the problem. Note that, for any policy $a$ satisfying $E[G(a(\go),\go)|a(\go)=a]=0$ and any well-behaved $x,$ we always have 
\begin{equation}\label{key}
\underbrace{E[W(a(\go),\go)-W(a_*(\go),\go)]}_{gain\ from\ concealing\ information}\ =\ -\underbrace{E[c(a(\go),\go;x)]}_{expected\ cost}
\end{equation}
Thus, the problem of maximizing $E[W(a(\go),\go)]$ over all admissible policies $a(\go)$ is equivalent to the problem of {\it minimizing the expected cost of information transport},  $E[c(a(\go),\go;x)]\,.$ Recall that $D_aW\in \R^M$ is the gradient of $W$ and $D_aG\in \R^{M\times M}$ is the Jacobian of $G.$ In the Appendix, we prove the following result. 

\begin{theorem}\label{mainth-limit} There always exists a Borel-measurable pure optimal policy $a(\go)$ solving the problem of Definition \ref{dfn1}. Furthermore, if we define the optimal information manifold 
$\Xi\ =\ \Supp(a)$
and 
\begin{equation}\label{x-top}
x(a)^\top\ =\ E[D_aW(a,\go)|a]\,E[D_aG(a,\go)|a]^{-1}\,,
\end{equation}
then
\begin{equation}\label{c<0}
c(a(\go),\go;x)\le 0\qquad{(transporting\ information\ from\ \go\ to\ a(\go)\ has\ a\ negative\ cost)}
\end{equation}
and
\begin{equation}\label{olicy}
a(\go)\ \in\ \arg\min_{b\in \Xi}c(b,\go;x)\qquad \qquad{(transporting\ information\ to\ a(\go)\ is\ optimal)}\,,
\end{equation}
and the function $c(a(\go),\go;x)$ is Lipschitz continuous in $\go.$ 
Furthermore, any optimal information design satisfies \eqref{c<0} and \eqref{olicy}.
\end{theorem}

Theorem \ref{mainth-limit} characterizes some important properties that are necessary for an optimal information design. We now establish an interesting connection between Theorem \ref{mainth-limit} and optimal transport theory. We first recall the classical optimal transport problem of Monge and Kantorovich (see, e.g., \cite{mccann2011five}). 

\begin{definition} Consider two probability measures, $\mu_0(\go)d\go$ (distribution of mines) on $\gO$ and $\nu$ on $\Xi$ (distribution of factories). The optimal map problem (the Monge problem) is to find a map $X:\gO\to \Xi$ that minimizes 
$\int c(X(\go),\go)\mu_0(\go)d\go$ 
under the constraint that the random variable $\chi=X(\go)$ is distributed according to $\nu.$
The Kantorovich problem is to find a probability measure $\gamma$ on $\Xi\times \gO$ that minimizes 
$\int c(\chi,\go)d\gamma(\chi,\go)$ 
over all $\gamma$ whose marginals coincide with $\mu_0(\go)d\go$ and $\nu$, respectively. 
\end{definition}

It is known that, under very general conditions, the Monge problem and its Kantorovich relaxation have identical values, and an optimal map exists. 
It turns out that any optimal policy of Theorem \ref{mainth-limit} solves the Monge problem.\footnote{This result is a direct analog of Corollary 2.5 in \cite{kramkov2019optimal} that was established there in the special case of $W(a)=a_1a_2,$ $g(\go)=\binom{\go_1}{\go_2}$ and $L=M=2.$ Its proof is also completely analogous to that in \cite{kramkov2019optimal}.}

\begin{theorem}\label{monge} Any optimal policy $a(\go)$ solves the Monge problem with $\nu$ being the distribution of the random vector $a(\go)\in \R^M.$ 
\end{theorem}

As each $\go$ is already coupled with the ``best" $a(\go)$ by \eqref{olicy}, it is clear that $a(\go)$ is $c$-cyclically monotone.\footnote{That is, all $k\in \N$ and $\go_1,\cdots,\go_k\in \gO$ satisfy 
\[
\sum_{i=1}^k c(a(\go_i),\go_i;x)\ \le\ \sum_{i=1}^k c(a(\go_i),\go_{\sigma(i)};x)
\]
for any permutation $\sigma$ of $k$ letters.}
Cyclical monotonicity plays an important role in the theory of optimal demand. See, for example, \cite{rochet1987necessary}. In our setting, this result has a similar flavour: In order to induce an optimal action, the sender optimally aligns actions $a$ with the state $\go$ to minimize the cost of information transport, $c.$ 
However, there is a major difference between Bayesian persuasion and classic optimal transport. In the Monge-Kantorovich problem, factories are in fixed locations and we need to design the transport plan. In contrast, in Bayesian persuasion the ``location of factories" is endogenous: It is the support $\Xi$ of the map $a(\go)$. The optimal choice of $\Xi$ is the key element of the solution to the Bayesian persuasion problem. 
Once this support and the distribution of ``factories" along $\Xi$ are fixed, Theorem \ref{monge} tells us that there is no way to improve the sender's utility {\it even through policies that do not satisfy receivers' optimality conditions \eqref{sys-3},} as long as we do not change the location of factories (that is, the optimal information manifold). 

In general, we do not know if the conditions of Theorem \ref{mainth-limit} are also sufficient for the optimality. As we show in the next section, such sufficiency can be established when receivers' action is the expectation of a function of $\go$. In this case, we can also characterize the set $\Xi$ and derive its geometric properties. 

\section{Moment Persuasion} \label{sec:moment}

In this section, we consider a setup where
$G(a,\go)\ =\ a-g(\go)$ 
for some continuous functions $g(\go)=(g_i(\go))_{i=1}^M:\ \gO\to\R^M.$ \cite{DworzakKolotilin2019} refer to this setup as ``moment persuasion." Everywhere in the sequel, we make the following assumption. 

\begin{assumption} We have $G(a,\go)=a-g(\go),\ a_*(\go)=g(\go),\ W(a,\go)=W(a)$ and $|W(a)|+\|D_aW\|\le f(\|a\|^2)$ for some convex function $f$ satisfying $E[\|g(\go)\|^2f(\|g(\go)\|^2)]<\infty.$
\end{assumption}

In the case of moment persuasion, equation \eqref{x-top} reduces to $x(a)^\top=D_aW(a).$ Thus, a key simplification comes from the fact that the Lagrange multiplier $x(a)$ of receivers' optimality conditions \eqref{sys-3} is independent of the choice of information design. We will slightly abuse the notation and introduce a modified definition of the function $c.$ Namely, we define 
\begin{equation}\label{cab}
c(a,b)\ =\ W(b)\ -\ W(a)\ +\ D_aW(a)\,(a-b)\,.
\end{equation}
As one can see from \eqref{cab}, the cost of information transport, $c,$ coincides with the classic Bregman divergence that plays an important role in convex analysis (see, e.g., \cite{rockafellar1970convex}). In particular, as the graph of a convex function always lies above a tangent hyperplane, $c(a,b)\ge 0$ when $W$ is convex, and hence $c(a,b)$ can be interpreted as ``distance". However, in our setting $W$ is generally not convex and hence $c$ can take negative values. 

We define the {\it Bregman Projection} $\cP_\Xi$ onto a set $\Xi$ via 
\begin{equation}\label{bregman-proj}
\cP_\Xi(b)\ =\ \arg\min_{a\in\Xi} c(a,b)\,\qquad(definition\ of\ a\ Bregman\ Projection).
\end{equation}
In other words, $\cP_\Xi$ projects $b$ onto the point $a\in \Xi$ that attains the lowest Bregman divergence. As neither $\Xi$ nor $W$ are convex, standard results about Bregman projections do not apply.\footnote{Note that $\arg\min_{a\in\Xi} c(a,b)$ might be a set of cardinality higher than one. In particular, it is generally not true that $\cP_\Xi$ is a true projection in the sense that $\cP_\Xi(\cP_\Xi(a))\ =\ \cP_\Xi(a)$ for any $a.$ We show below that the projection property always holds for an optimal information manifold $\Xi$ due to special geometric properties of such manifolds. } 
Our key objective here is to understand the structure of the support set $\Xi$ of an optimal policy. 

\begin{definition}\label{max-def} Let $conv(X)$ be the closed convex hull of a set $X\subset\R^M.$ A set $\Xi\subset \R^M$ is $X$-maximal if $\inf_{a\in \Xi}c(a,b)\le 0$ for all $b\in X.$ A set $\Xi$ is $W$-monotone if $c(a_1,a_2)\ge 0$ for all $a_1,a_2\in \Xi.$ A set $\Xi$ is $W$-convex if $W(ta_1+(1-t)a_2)\le t W(a_1)+(1-t)W(a_2)$ for all $a_1,a_2\in\Xi,\ t\in [0,1].$ We also define $\phi_\Xi(x)\ =\ \inf_{a\in \Xi} c(a,x)\,.$\footnote{Note that we are again slightly abusing the notation so that the function $\phi_\Xi$ from the previous section corresponds to $\phi_{\Xi}(g(\go))$ in this section.}
\end{definition}

We now state the first important result of this section: Any pure policy is a (Bregman) projection onto an optimal information manifold. 

\begin{theorem}[Optimal Policies are Projections] \label{cor-moment} There always exists a pure optimal policy. Furthermore:
\begin{itemize}
\item  Each such policy is a Bregman projection onto an optimal information manifold $\Xi:\ a(\go)\ \in\ \cP_\Xi(g(\go))$ for Lebesgue-almost every $\go.$  

\item Any optimal information manifold is $conv(g(\gO))$-maximal, $W$-convex, and $W$-monotone. 

\item The pool of every signal value $a_1$ is $Pool(a_1)\ =\ \{\go\in \gO:\ a(\go)=a_1\}\ =\ a^{-1}(a_1).$ 
\end{itemize}
\end{theorem}

Theorem \ref{cor-moment} highlights the key property of Bayesian persuasion: The optimal policy is always a projection, minimizing the ``distance", $c(a,b)$, to the optimal information manifold. Importantly, $W$-monotonicity implies the {\it projection property:} Further transporting $b$ along the manifold is costly, and it is already in the right place. Formally, it means that if $g(\go)\in \Xi,$ then it is optimal to reveal the true value of $g(\go)$ instead of sending a signal corresponding to a different point on the manifold.\footnote{As $c(a,b)\ge 0$ for all $a,b\in \Xi,$ we have $c(\cP_\Xi(b),a)\ge 0$ for all $a$ and, by direct calculation, $c(\cP_\Xi(b), \cP_\Xi(b))=0$. Thus, we have the projection property: for any $x\in \cP_\Xi(b),$ we have $x\in \cP_\Xi(x).$ If $\cP_\Xi(b)$ is always a singleton, this boils down to the ``true" projection property: $\cP_\Xi(b)=\cP_\Xi(\cP_\Xi(b)).$}
 
We now discuss the two key properties of an optimal information manifold: {\it monotonicity and maximality.}\footnote{By direct calculation, $W$-convexity is equivalent to $W$-monotonicity when $W$ is quadratic.} Suppose for simplicity that $M=2,$ $W(a)=a_1a_2$, and $g(\go)=\binom{\go_1}{\go_2},$ as in  \cite{rochet1994insider} and \cite{RayoSegal2010}. Then, $c(a,b)\ =\ (a_1-a_2)(b_1-b_2),$ and $W$-monotonicity means that $(a_1-a_2)(b_1-b_2)\ge 0$. Thus, in the language of \cite{RayoSegal2010}, for any two signals $s_1\not=s_2,$ expected prospects $a(i)=E[\go|s_i]=\binom{E[\go_1|s_i]}{E[\go_2|s_i]}$ are {\it ordered:} A better signal reveals that both expected dimensions of the prospect are better. This ordering immediately implies the existence of a monotone increasing function $f:\R\to\R$ such that $\Xi\subset Graph(f)\ =\ \{\binom{f(a_2)}{a_2}:\ a_2\in\R\}.$ Clearly, this graph is a one-dimensional object (a curve) and, thus, so is the optimal information manifold, $\Xi.$\footnote{The exact implementation of a signalling policy is not unique. For example, the sender can just commit to the function $\go\to a_2(\go_1,\go_2),$ and then reveal its value. Alternatively, the sender recommends that receivers take the action $a(\go)$ and commits to its functional form. As a result, receivers know that $a_1=f(a_2)$ and simply ignore $a_1$ and only use $a_2(\go)$ to perform Bayesian updating.}

Therefore, {\it optimal persuasion is achieved by dimension reduction}  and {\it $W$-monotonicity imposes a lower-dimensional structure on $\Xi.$} Figure \ref{fig:curve} illustrates this. We now turn to $W$-maximality. Maximality means that we cannot extend $\Xi$ while preserving $W$-monotonicity. This means that $\Xi$ must extend over the whole support $\gO$; otherwise, we can extend the graph of the monotone function $f$ to the left or to the right. See Figure \ref{fig:curve}.  Intuitively, all information is transported to some (optimal) location on $\Xi$ and hence it is optimal to extend $\Xi$ as much as possible, until either (1) it hits the boundary of the support of $\mu_0$; or (2) $\Xi$ closes on itself like a circle.\footnote{The exact shape of $\Xi$ may depend in a complex fashion on the prior $\mu_0(\go).$ See Section \ref{sec:conceal}.}

\begin{figure}[ht!]
  \begin{center}
  \includegraphics[width=1\linewidth]{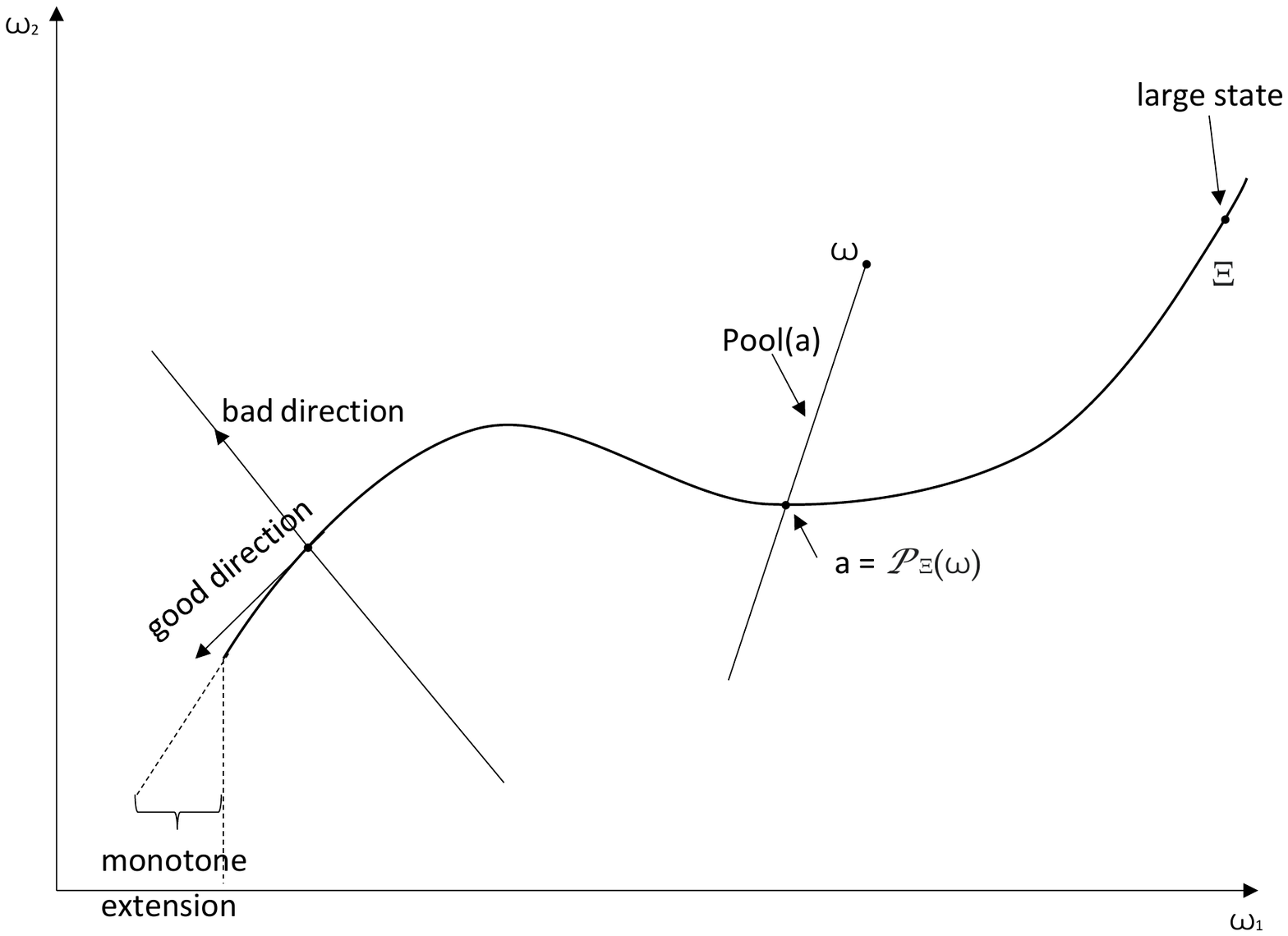}
  \end{center}
  \vskip-7cm
   \caption{The figure illustrates the properties of an optimal information manifold, $\Xi$ (the curve on the plot). If $c(a,b)\ =\ (a_1-a_2)(b_1-b_2),$ the curve $\Xi$ is the graph of a monotone increasing function so that, along $\Xi,$ we have $\go_2=\varphi(\go_1).$ In this case, the maximality of $\Xi$ means that $\Xi$ must extend to the full support of $\mu_0:$ Otherwise, the dashed line in the left part of the figure shows a monotone extension of $\Xi,$ implying that $\Xi$ cannot be maximal. Line segments transversal to $\Xi$ correspond to signal pools. A large state corresponds to a distant point on $\Xi.$}
  \label{fig:curve}
\end{figure}

For a discrete state space, \cite{RayoSegal2010} show that the pool of every signal is a discrete subset of a line segment. Therefore, we intuitively expect that, in the continuous case, {\it the pool of every signal is convex} (i.e., there are no gaps in the segment).
The convexity of pools is an intuitive and important property: It means that the sender bundles states that are close to one another. Convex pools correspond to monotone partitions when $M=1.$\footnote{Many real-world persuasion mechanisms pooling only happens between contiguous types. This is the case for most ranking mechanisms, such as school grades, credit, and restaurant and hotel ratings; see, for example, \cite{ostrovsky2010} and \cite{Hopenhayn2019}. Monotone partitions also appear as equilibria in communication models without commitment, such as the cheap talk model of \cite{crawford1982strategic}. } When $g(\go)$ is non-linear, the convexity of pools cannot be guaranteed.\footnote{See, for example, \cite{Arielietal2020}.} However, as \cite{DworczakMartini2019} show, when $g(\go)=\go$ and $M=1$, monotone partitions are indeed optimal when $W(a)$ is affine-closed.\footnote{Roughly speaking, an affine-closed function is such that $W+q$ has at most one local maximum for any affine function $q.$ See, \cite{DworczakMartini2019} for details.} 
Extending the monotone partitions result of \cite{DworczakMartini2019} to multiple dimensions is far from trivial. The following is true.

\begin{proposition}[Monotone partitions (Convexity of pools)] \label{main convexity} 

Suppose that $g(\go)=\go$ and $\gO$ is convex.  Then, there exists a pure optimal policy $a(\go)$ such that the map $\go\to D_aW(a(\go))$ is monotone increasing on $\gO$\footnote{In fact, $c(a(\go),\go)$ is convex on $\gO$ and $D_aW(a(\go))$ is a subgradient of $c(a(\go),\go).$} and the set 
\begin{equation}\label{pool-union}
\{\go\in\gO:\ D_aW(a(\go))=a\}\ =\ \cup_{b\in (D_aW)^{-1}(a)}Pool(b)
\end{equation}
is always convex. If the map $a\to D_aW(a)$ is injective, then the pool of every signal is convex (up to a set of measure zero)\footnote{The last claim follows because level sets for a monotone map are convex.} and 
$a(\go)$ is an idempotent: $a(a(\go))\ =\ a(\go).$ 
\end{proposition}

We believe that condition \eqref{pool-union} is the strongest result one can hope for in a general, multi-dimensional setting. Note that this result is also new in the one-dimensional case, complementing the results of \cite{DworczakMartini2019} and \cite{Arielietal2020}. It emphasizes the basic intuition that the degree of non-convexity of pools is closely linked to the number of ``peaks" of the value function $W.$ The more the derivative of $W$ changes its sign, the more disconnected the pools can potentially become, whereby the sender attempts to partition the state space to induce actions that come as close as possible to the peaks (local maxima of the utility function). Furthermore, Proposition \ref{main convexity} also implies the following novel result: If $(D_aW)^{-1}(a)$ is a singleton for some $a$, then $Pool(a)$ is convex. 

The requirement of a globally injective gradient for $W$ is much stronger than the affine-closeness assumption of \cite{DworczakMartini2019}. In particular, in the one-dimensional case it implies that $W$ is either concave (when $D_aW$ is decreasing) or convex (when $D_aW$ is increasing). However, in multiple dimensions Proposition \ref{main convexity} applies to a large class of non-trivial functions. For example, any linear-quadratic utility function $W(a)=a^\top H a+h^\top a$ with a non-degenerate $H$ satisfies this assumption because $D_aW(a)=2Ha+h$ is injective. 

By Theorem \ref{cor-moment}, any optimal information manifold has two key properties: maximality (ensuring that any state $g(\go)$ can be optimally transported to some point $b\in \Xi$ at a non-positive cost) and $W$-convexity. In general, verifying $W$-convexity (or, even, $W$-monotonicity) of a candidate optimal manifold $\Xi$ is hard. It turns out that, perhaps surprisingly, it is possible to provide explicit and verifiable necessary and sufficient conditions for the optimality  of a given manifold without verifying $W$-convexity. Recall that, by Definition \ref{max-def}, $\Xi$ is $conv(g(\gO))$-maximal if $\inf_{a\in \Xi}c(a,b)\le 0$ for all $b\in conv(g(\gO)).$ The following is true.   

\begin{theorem}[Maximality is both necessary and sufficient] \label{converse} Let $\Xi$ be a $conv(g(\gO))$-maximal subset of $\R^M$. Suppose that a map $a(\go)$ satisfies  $a(\go)\ \in\ \cP_\Xi(g(\go))$ and $a(\go)\ =\ E[g(\go)|a(\go)]$ for Lebesgue-almost all $\go.$ Then, $a$ is an optimal policy. 
\end{theorem}

Theorems \ref{converse} and \ref{cor-moment} imply that $conv(g(\gO))$-maximality is both necessary and sufficient for the optimality . This fact drastically simplifies the search for optimal information designs. It implies that finding an optimal policy reduces to two steps: (1) find a set of candidates  $\Xi$ that are $conv(g(\gO))$-maximal and (2) solve the following {\it integro-differential equation}:\footnote{The fact this is indeed an integro-differential equation is shown in Proposition \ref{prop-change-variables} in the Appendix.}
\begin{equation}\label{integro-dif}
\cP_\Xi(g(\go))\ =\ E[g(\go)|\cP_\Xi(g(\go))]\,.
\end{equation}
In Section \ref{sec-applications}, we illustrate how Theorem \ref{converse} can be used as a tool for solving persuasion problems. The arguments in the proof of Theorem \ref{converse} can also be used to shed some light on the uniqueness of optimal policies, as is shown in Proposition \ref{uniqueness} in the Appendix. 

We complete this section with a discussion of the link between Theorems  \ref{cor-moment} and \ref{converse} and Theorem 3 in \cite{DworzakKolotilin2019} (the persuasion duality). \cite{DworzakKolotilin2019} show that when (i) $\gO$ is compact and (ii) $W(a)$ is uniformly Lipschitz continuous on $\gO$, there exists a globally convex function $p(a)$ such that (1) $p(a)\ge W(a)$ for all $a\in conv(g(\gO))$; and (2) $p(a)$ coincides with $W(a)$ on the support of optimal policy. Conversely, given a candidate policy $a(\go)$ satisfying $a(\go)\ =\ E[g(\go)|a(\go)],$ this policy is optimal if there exists a globally convex function $p(a)$ satisfying $p(a)\ge W(a)$ and $p(a)=W(a)$ for all $a$ in the support of the signal distribution. \cite{DworzakKolotilin2019} provide no recipes for finding a candidate policy $a(\go)$ and the convex majorante $p(a).$ A key insight from our analysis is that imposing regularity (namely, smoothness of $W$ and Lebesgue-absolute continuity of $\mu_0$) implies that (i) pure policies always exist and (ii) the support is a lower-dimensional manifold. Although this support does have complex properties (its $W$-convexity is equivalent to the convexity of the majorante $p(a)$ on $\Xi$), Theorem \ref{converse} circumvents this: All we need is to verify maximality, which is straightforward in many applications. See Section \ref{sec:conceal}.\footnote{Note also that the proofs in \cite{DworzakKolotilin2019} rely heavily on the compactness of $\gO$ and uniform Lipschitz continuity of $W.$ Even in the simple quadratic setting in their Proposition 2, $W=a_1a_2$ is not uniformly Lipschitz continuous when $\gO$ is not compact. Many applications (such as the Gaussian setting of \cite{tamura2018Bayesian}) require $\gO=\R^L,$ and hence the results of \cite{DworzakKolotilin2019} are not directly applicable.} 

\section{Optimal Information Manifold} \label{sec:oim}

To the best of our knowledge, the only well-understood example of multi-dimensional persuasion has been studied in \cite{rochet1994insider} and \cite{RayoSegal2010}, with $W=a_1a_2$ and $g(\go)=\binom{\go_1}{\go_2}.$ In this case, the optimal information manifold is a well-behaved curve in $\R^2$ (the graph of a monotone function), while the pool of every signal is a subset of a line segment. \cite{RayoSegal2010} also show that the ``pools are subsets of line segments" result also holds when $W=a_1G(a_2)$ for a strictly increasing function $G.$ How universal are these properties? Is there an intuitive way to characterize whether (and how) signals are ordered? Which states get pooled together and how big are these pools? Why do the pools in \cite{rochet1994insider} and \cite{RayoSegal2010} always have Lebesgue measure zero? Is pooling achieved through a form of monotone partition? Why can some pools not be open subsets of $\R^M,$ like they are in \cite{DworczakMartini2019} for $M=1$? 

In this section, we provide answers to these (and many other) questions. To this end, we characterize the optimal information manifold and the associated optimal signal pools. 
For any symmetric matrix $H,$ let $\nu_+(H)$ be the number of strictly positive eigenvalues and $\nu(H)\ge \nu_+(H)$ the number of nonnegative eigenvalues. Recall that the dimension, $dim(X),$ of a convex set $X$ is the dimension of the smallest linear manifold containing it. We start with the following result, characterizing the dimension and the ``position" of signal pools. 

\begin{corollary}[Pools are low-dimensional sets]\label{dimension-pool} We have $dim(conv(g(Pool(a))))\ \le\ M-\nu_+(D_{aa}W(a)).$ In particular, if $D_{aa}W(a)$ has at least one strictly positive eigenvalue for any $a,$ then the convex hull of the image of the pool of any signal, $conv(g(Pool(a)))$, has measure zero. If $g$ is locally injective and bi-Lipschitz,\footnote{$g$ is locally bi-Lipschitz if both $g$ and its local inverse, $(g|_X)^{-1}$ are Lipschitz for any compact set $X.$}then $Pool(a)$ also has Lebesgue measure zero for each $a.$ 
\end{corollary}

Corollary \ref{dimension-pool} implies that the observations in \cite{rochet1994insider} and \cite{RayoSegal2010} are, in fact, typical for $M>1:$ Pools have measure zero, and can only have positive measure if $a$ is a local maximum, so that $D_{aa}W(a)$ is negative semi-definite. In the \cite{RayoSegal2010} setting, $W(a)=a_1G(a_2)$ and $D_{aa}W(a)=\begin{pmatrix}
0&G'(a_2)\\
G'(a_2)&a_1G''(a_2)
\end{pmatrix}
$ 
with $G'(a_2)>0,$ and hence $det(D_{aa}W(a))=-(G'(a_2))^2<0$, so that $D_{aa}W(a)$ always has exactly one positive and one negative eigenvalue. Thus,  each action profile $a$ in \cite{RayoSegal2010} is a saddle point. This is driven by two features: First, the value function is linear in $a_1$ because the sender is risk-neutral; and second, expected payoffs of the sender and the receiver are substitutes because a higher expected payoff increases the likelihood of acceptance. 

We next turn to the geometric properties of the optimal information manifold that will be crucial for deriving an explicit characterization of signal pools. 
We start our analysis with a formal definition of a manifold that we include for the readers' convenience. 

\begin{definition}\label{dfn-man} A $\nu$-dimensional (topological) manifold $\Xi\subset \R^M$ is a set such that every point $a\in \Xi$ has a neighborhood homeomorphic to $\R^\nu$.\footnote{Two sets  $X$ and $Y$ are homeomorphic if there exists a homeomorphism $\varphi$ (a continuous map with a continuous inverse) such that $\varphi(X)=Y.$} The respective homeomorphism is called a (local) coordinate map. 
A $\nu$-dimensional  Lipschitz (respectively, $C^k$)-manifold is such that the respective homeomorphism and its inverse are Lipschitz-continuous (respectively, $k$-times continuously differentiable). 
\end{definition}

A manifold can typically be defined in two ways: through a coordinate map or through a system of equations. For example, the unit circle $\Xi\ =\ \{(a_1,a_2)\in \R^2:\ a_1^2+a_2^2=1\}$ is a smooth 1-dimensional manifold defined by one equation $\Psi(a)=0$ with $\Psi(a)=a_1^2+a_2^2-1.$ Clearly, such a $\Psi$ is not unique: Any monotonic transformation of $\Psi$ can be used to define the same manifold. This manifold is {\it locally} homeomorphic to $[0,2\pi)$ (and hence also to $\R^1$ or any other interval) with a coordinate map given by $f(\theta)\ =\ \binom{\cos\theta}{\sin\theta}:\ [0,2\pi)\to \Xi,$ which shows that the circle ``locally looks like an interval." The coordinate map is also not unique. For example, $\tilde f(\theta)=\binom{\sin\theta}{\cos\theta}$ is another coordinate map. 

Our goal is to show that $\Xi$ is a lower-dimensional set. In order to gain intuition for the origins of lower-dimensionality of an optimal information manifold, consider first the case when $W$ is linear-quadratic, $W(a)=0.5 a^\top H a+h^\top a.$ In this case, by direct calculation, $c(a,b)\ =\ (a-b)^\top H (a-b)$ and hence $W$-monotonicity of $\Xi$ implies that for any $a_1,\ a_2\in \Xi,$ we have $(a_1-a_2)^\top H (a_1-a_2)\ge 0.$ As we will now show, this imposes a low-dimensional structure on $\Xi.$ Let $H=VDV^\top$ be the eigenvalue decomposition of $H$ where $D=\diag(\gl_1,\cdots,\gl_M)$ with $\gl_1\ge \gl_2\ge\cdots\ge \gl_M.$ Let $\nu$ be the number of nonnegative eigenvalues of $H.$ Let also $\tilde a_i=  |D|^{1/2}V^\top a_i\ =\ \binom{\tilde a_{i,-\nu}}{\tilde a_{i,\nu+}},$ where we have split the vector into two components corresponding to nonnegative and negative eigenvalues, respectively.\footnote{$|D|^{1/2}=\diag(|\gl_i|^{1/2})_{i=1}^M.$}  

In the sequel, we refer to $\tilde a_{-\nu}$ as {\it good information}, and to $\tilde a_{\nu+}$ as {\it bad information.} The eigenvectors of $H$ for positive (negative) eigenvalues are the {\it directions of good (bad) information}: For the former, the sender is risk-loving, and wants to reveal as much information as possible; for the latter, the sender is risk-neutral and wants to conceal information. The monotonicity condition $(a_1-a_2)^\top H (a_1-a_2)\ge 0$ now takes the form 
\begin{equation}\label{good-bad}
\underbrace{\|\tilde a_{1,-\nu}-\tilde a_{2,-\nu}\|^2}_{change\ in\ good\ information}\ \ge\ \underbrace{\|\tilde a_{1,\nu+}-\tilde a_{2,\nu+}\|^2}_{change\ in\ bad\ information}\ for\ all\ a_1,\ a_2\in\Xi\,.
\end{equation}
This is a general, multi-dimensional analog of the phenomenon discovered by \cite{rochet1994insider} and \cite{RayoSegal2010} : {\it The set of policies is ordered} in the sense that for any two signals $s_1\not=s_2,$ with policies $a_1=E[g(\go)|s_1],\ a_2=E[g(\go)|s_2]$, the change in good information, $\|\tilde a_{1,-\nu}-\tilde a_{2,-\nu}\|^2$,  must be larger than the change in bad information, $\|\tilde a_{1,\nu+}-\tilde a_{2,\nu+}\|^2.$ The condition \eqref{good-bad} immediately implies the existence of a map $f:\ \R^\nu\to\R^{M-\nu}$ such that $\tilde a_{\nu+}=f(\tilde a_{-\nu})$ because the coincidence of $\tilde a_{-\nu,1}$ with $\tilde a_{-\nu,2}$ always implies the coincidence of $\tilde a_{\nu+,1}$ with $\tilde a_{\nu+,2}.$ Furthermore, \eqref{good-bad} implies that $f$ is Lipschitz-continuous with the Lipschitz constant of one.\footnote{The classic \cite{kirszbraun1934zusammenziehende} theorem implies that $f$ can always be extended to the whole $\R^\nu.$} Thus, {\it $\Xi$ is a subset of a $\nu(H)$-dimensional Lipschitz manifold,} justifying the name ``optimal information manifold." In general, $H$ is replaced by $D_{aa}W$ and the following is true. 

\begin{theorem}[$\Xi$ is a lower-dimensional manifold] \label{thm-unif} Let $\Xi$ be an optimal information manifold (the support of an optimal policy) and $\nu(a)=\nu(D_{aa}W(a))$ be the local degree of convexity of $W$. Then, for any open set $B,$ $\Xi\cap B$ is a subset of a Lipschitz manifold of dimension at most $\sup_{a\in B} \nu(a).$ 
\end{theorem}

Theorem \ref{thm-unif} confirms our intuition: {\it Optimal persuasion is achieved by dimension reduction.} This reduction happens by a (Bregman) projection onto a lower-dimensional object: the optimal information manifold $\Xi.$ The dimension of this manifold is controlled by the number of directions of convexity of sender's utility $W$-- that is, the number of positive eigenvalues of $D_{aa}W.$ When $D_{aa}W(a)$ does not vary too much across $a,$ then $\Xi$ globally behaves like $\R^\nu.$ Otherwise, only local coordinates can be guaranteed, while globally $\Xi$ may behave like, for example, a circle, as we show in Section \ref{sec:conceal} below. 

We would now like to understand the fine properties of signal pools. Which states get pooled together? Is there a sense in which (as in \cite{RayoSegal2010}) less relevant but more attractive states are pooled with more relevant but less attractive ones? Is there an analytical way to describe $Pool(a)$ for a given $a?$
Theorem \ref{thm-unif} only implies that we can characterize $\Xi$ as $\Xi=\{a\in \R^M:\ a=f(\theta),\ \theta\in \Theta\},$ where $\Theta \subset\R^\nu$ is a lower-dimensional subset with {\it unknown properties.} Rewriting \eqref{bregman-proj} as
\begin{equation}\label{attempt-aa}
a(\go)\ =\ \arg\min_{\theta\in \Theta} c(f(\theta), g(\go))\,,
\end{equation}
one might be tempted to differentiate \eqref{attempt-aa} with respect to $\theta.$ Indeed, as $f$ is Lipschitz continuous, it is differentiable Lebesgue-almost everywhere by the Rademacher Theorem.\footnote{See, e.g., \cite{cheeger1999differentiability}.} However, differentiation in \eqref{attempt-aa} is only possible if the set $\Theta$ is ``sufficiently rich", extending ``in all possible directions."  Establishing richness is extremely difficult. In the Appendix, we use techniques from geometric measure theory to achieve this goal. Intuitively, Corollary \ref{dimension-pool} tells us that each $Pool(a)$ has dimension $M-\nu$ and hence $\Xi$ ought to have dimension $\nu$ because $\gO\ =\ \cup_{a\in \Xi}Pool(a).$ In the Appendix, we show that this is indeed the case using the theory of the Hausdorff dimension. Namely, $\Xi$ has a Hausdorff dimension $\nu.$ To the best of our knowledge, this is the first application of the Hausdorff dimension in economics.\footnote{See Proposition \ref{regularity} in the Appendix.}
This Hausdorff dimension result gives enough richness to perform differentiation in \eqref{attempt-aa}. The following is true. 

\begin{corollary}[Characterization of Pools]\label{char-pools} Let $a(\go)$ be a pure optimal policy and $\Xi$ the corresponding optimal information manifold. Suppose that $D_{aa}W(a)$ is non-degenerate and that, for any $\eps>0,$ $conv(g(Pool(\Xi\cap B_\eps(a))))\subset\R^M$ has positive Lebesgue measure. Let $f:\R^\nu\to \Xi$ be local coordinates from Theorem \ref{thm-unif} in a small neighborhood of $a,$ and let $\Theta=f^{-1}(\Xi\cap B_\eps(a)).$ Then, for Lebesgue-almost every  $\theta\in \Theta,$ $f$ is differentiable, with a Jacobian $Df(\theta)\in \R^{M\times\nu},$ and we have 
\begin{itemize}
\item[(1)] the matrix $Df(\theta)^\top D_{aa}W(f(\theta)) Df(\theta)\in \R^{\nu\times\nu}$ is Lebesgue-almost surely symmetric and positive semi-definite. 

\item[(2)] Lebesgue-almost every $\go$ satisfies 
\begin{equation}\label{downward-sloping-pools} 
Df(\theta)^\top D_{aa}W(f(\theta))(f(\theta)-g(\go))\ =\ 0\,\,,
\end{equation}
\end{itemize}
when $\go\in Pool(f(\theta)).$
\end{corollary}

Corollary \ref{char-pools} provides a general, analytic characterization of optimal pools. Item (1) describes the monotonicity properties of the optimal information manifold, $\Xi,$ while item (2) describes how optimal pools are linked to the ``slope" of $\Xi,$ as captured by $Df(\theta).$ As an illustration, consider the case $W(a)=a_1a_2,\ g(\go)=\binom{\go_1}{\go_2}$ of \cite{rochet1994insider} and \cite{RayoSegal2010}. In this case, $f(\theta)=\binom{\varphi(\theta)}{\theta}$ and item (1) takes the form $Df(\theta)^\top D_{aa}W(f(\theta)) Df(\theta)=(\varphi'(\theta),1)\begin{pmatrix}0&1\\1&0\end{pmatrix}\binom{\varphi'(\theta)}{1}=2\varphi'(\theta)\ge 0,$ confirming that $\Xi$ is the graph of a monotone increasing function, while the pool equation \eqref{downward-sloping-pools} takes the form 
\[
\go_1\ =\ \varphi'(\theta)(\theta-\go_2)+\varphi(\theta)\,. 
\]
Thus, pools are (subsets of) lines orthogonal to the manifold $\Xi$ (see Figure \ref{fig:curve}). Furthermore, {\it the slope of each line is determined by the slope of the manifold:} The steeper the slope, the stronger the separation is between different signals on $\Xi.$ Thus, already in this simple setting, formula \eqref{attempt-aa} provides a novel insight: an equality between the degree of separation and the steepness of pools.\footnote{However, $\varphi'(\theta)$ depends on the exact properties of the prior $\mu_0$ and has to be determined by solving an integro-differential equation. See, Proposition \ref{prop-change-variables} in the Internet Appendix.}

\section{Concealing the Tails} 

In the preceding example, $\Xi$ is the graph of a monotone function. If $\gO=\R$ is unbounded, the maximality of $\Xi$ implies that $\Xi$ extends all the way to infinity; see Figure \ref{fig:curve}. This has important implications for the nature of signals for extreme (tail) state realizations--specifically, the optimal policy always reveals (some) information about the tails. In particular, there will always be states $\go=\binom{\varphi(\theta)}{\theta}$ with arbitrarily large $\theta$ that are revealed.  However, this property cannot be true in a general persuasion problem. 
To gain some intuition, suppose first that $W$ is concave for large $a,$ meaning the sender becomes risk averse when $a$ is large: $D_{aa}W$ is negative semi-definite for all $a$ with $\|a\|>K$ for some $K>0.$ Then, by Corollary \ref{dimension-pool}, any optimal information manifold satisfies $\Xi\subset \{a:\ \|a\|\le K\}$ and is therefore bounded. 
Of course, concavity is a very strong condition. It turns out that the boundedness of optimal information manifolds can be established under much weaker conditions. We will need the following definition. 

\begin{definition}\label{conc-rays} Let $\cC(a,\eps)=\{b\in \R^M:\ b^\top a/(\|a\|\,\cdot\|b\|)>1-\eps\}$ be the $\eps$-cone around $a:$ the set of vectors $b$ that point in approximately the same direction as $a.$  We say that the value function $W$ is concave along rays for large $a$ if there exists a small $\eps>0$ and a large $K>0$ such that $b^\top D_{aa}W(a)b\ <\ 0$ for all $a$ with $\|a\|>K$ and all $b\in \cC(a,\eps)).$ 

We also say that a set $\Theta\subset \R^\nu$ extends indefinitely in all directions if the projection of $\Theta $ on any ray from the origin is unbounded. 
\end{definition}

Note that if $W$ has linear marginal utilities, $W(a)=a^\top H a+h^\top a$, we have $D_{aa}W(a)=2H$ and hence $a^\top D_{aa}W(a)a=2a^\top H a.$ Thus, $W$ is concave along rays if and only if $W$ is globally concave, implying that it is optimal not to reveal any information. As we show below, quadratic preferences represent a knife-edge case, as even slight deviations from linear marginal utilities may drastically alter the nature of optimal policies. The following is true. 

\begin{proposition}[Concealing Tail Information] \label{conv-bound} Suppose that $conv(g(\gO))=\R^M$ and let $\Xi$ be an optimal information manifold.
\begin{itemize}
\item If $W(a)=a^\top H a+h^\top a$ with $\det H\not=0$, then $\Xi=f(\Theta)$ for some Lipschitz $f:\R^{\nu(H)}\to\R^M,$ where $\Theta$ extends indefinitely in all directions;

\item If $W$ is concave along rays for large $a,$ then there exists a constant $K$ independent of the prior $\mu_0$, such that any optimal information manifold satisfies $\Xi\subset B_K(0).$
\end{itemize}

\end{proposition}

Proposition \ref{conv-bound} shows how a weak form of the sender's aversion of large risks makes it optimal to conceal information about large state realizations. Instead of assuming concavity (risk aversion) occurs in all directions, it is enough to assume it exists along rays according to Definition \ref{conc-rays}. The following claim follows by direct calculation from Proposition \ref{conv-bound}. 

\begin{corollary}[Concave Marginal Utility Implies Concealing the Tails]\label{concave-bounded} Let $H$ be a non-degenerate, $M\times M$ positive-definite matrix. Suppose that $W(a)=\varphi(a'Ha)$ for some $\varphi$ with $-\varphi''(x)/|\varphi'(x)|>\eps$ for some $\eps>0$ and all sufficiently large $x.$ Then, $\nu(D_{aa}(W(a)))\ge M-1$ for all $a.$ Yet, $W$ is concave along rays for large $a$ and, hence, any optimal information manifold is bounded, contained in a ball of radius $K$ that is independent of the prior $\mu_0.$ 
\end{corollary}

Corollary \ref{concave-bounded} shows explicitly how non-linear marginal utility alters the nature of optimal information manifolds, leading to a phenomenon that we call ``information compression", whereby potentially unbounded information is compressed into a bounded signal. Consider as an illustration $W(a)=\varphi(a_1^2+\gl a_2^2).$ First, let $\varphi(x)=x,$ corresponding to a linear marginal utility for the sender. If $\gl<0,$ the first item of Corollary \ref{conv-bound} applies, and we get that $\Xi$ is the graph of a Lipschitz function that extends indefinitely in all directions. Making $\gl$ more negative will lead to a rotation of the optimal information manifold, but will not alter its shape. Consider now a case when $\gl>0.$ If $\varphi(x)=x,$ $W(a)$ is convex and full information revelation is optimal. However, even a slight degree of concavity for $\varphi$ leads to a {\it compressed information revelation.} The optimal information manifold, $\Xi,$ is bounded and, hence, cannot be a graph of a function extending indefinitely. Instead, $\Xi$ is a bounded curve in $\R^2$ (e.g., a circle). 

\section{Applications} \label{sec-applications}

In this section, we illustrate the power of our approach by characterizing solutions to several multi-dimensional persuasion problems. 

\subsection{Supplying Product Information}

The question of how much (and what kind of) information firms should supply to their potential customers has received a lot of attention in the literature. See, for example, \cite{lewis1994supplying}, \cite{anderson2006advertising}, and \cite{johnson2006simple}. In this section, we consider an extension of the \cite{lewis1994supplying} model studied in \cite{RayoSegal2010}. 

In this model, the sender is endowed with a prospect randomly drawn from $\mu_0(\pi,v).$ Each prospect is characterized by $\go=\binom{\pi}{v}$ where $\pi$ is the prospect’s profitability to the sender and $v$ is its value to the receiver.\footnote{For example, $\pi$ is (future, random) revenue net of costs and $v$ is the private value of the good for the customer.} After observing the signal of the sender, the receiver decides whether to accept the prospect. Whenever the receiver accepts the prospect, she
forgoes an outside option worth $r,$ which is a random variable independent of $\go$ and drawn from a c.d.f. G over $\R.$\footnote{\cite{KamGenz2011} consider a similar model, but assume that the outside option is non-random, and so is the value of the product for the sender.} Thus, the sender and receiver obtain payoffs, respectively, equal to $q\pi$ and $q (v-r)$ where $q=1$ if the prospect is accepted and zero otherwise. Defining 
$g(\go)\ =\ \binom{\pi}{v}\,,$ we get by direct calculation that the sender's and receiver's expected utilities are respectively given by $
W(a)\ =\ a_1\,G(a_2),\ U(a)\ =\ \int_\R \max\{a_2-r,\ 0\}dG(r)\ =\ \int_{-\infty}^{a_2}(a_2-r)dG(r)\,.$
By \eqref{cab}, $c(a,b)\ =\ b_1(G(b_2)-G(a_2))\ -\ a_1\,G'(a_2)(b_2-a_2)\,.$ 
Consider first the case when $G(b)=b$ (uniform acceptance rate). As \cite{RayoSegal2010} show in a discrete state space setting, the set of possible signals'
payoffs $\binom{a_1}{a_2}=\binom{E[\pi|a]}{E[v|a]}$ (that is, the optimal information manifold $\Xi$) is ordered: For any two possible signals' payoffs $a,\ \tilde a\in \Xi,$ we always have $(a_1-\tilde a_1)(a_2-\tilde a_2)\ge 0$. As we explain above, this is a direct consequence of $W$-monotonicity because $c(a_1,\tilde a_1)=(a_1-\tilde a_1)(a_2-\tilde a_2)\,.$ As \cite{RayoSegal2010} explain, pooling two different prospects is only optimal if it preserves
the expected acceptance rate while at the same time shifting it from the more valuable to
the less valuable prospect. When two prospects are ordered, it does not make sense to pool them; hence, they correspond to different points on the support $\Xi.$ This observation  implies that $\Xi$ is in fact a graph of a monotone increasing function $a_1=f(a_2)$. However, the monotonicity of $f(a_2)$ (i.e., the ordering of prospects across possible signal realizations) depends crucially on the assumption of a uniform acceptance rate (i.e., $G(b)=b.$) \cite{RayoSegal2010} write: ``{\it when G is allowed to have an
arbitrary shape, not much can be said in general about the optimal
rule."} As we now show, continuous state space introduces analytical tractability that allows us to characterize the monotonicity properties of support of $\Xi$ for any $G.$ Throughout this subsection, we assume that $G$ is continuously differentiable with $G'(b)>0$ on its support. The following is a direct consequence of Corollary \ref{char-pools}. 

\begin{proposition}\label{Rayo-Segal} There always exists a pure optimal policy $a(\go)\ =\ \binom{a_1(\go)}{a_2(\go)}.$ For each such policy, there exists a function $f(a_2)$ such that $a_1(\go)=f(a_2(\go))$ for all $\go$ and, hence, the optimal information $\Xi$ is the graph $\{(f(a_2),a_2)\}.$  The function $f(a_2) (G'(a_2))^{1/2}$ is monotone increasing in $a_2.$ For each $a_2,$ $Pool(a_2)$ is a convex segment of the line $\pi\ =\ \kappa_1(a_2)\,v\ +\ \kappa_2(a_2)$ with $\kappa_1(a_2)\ =\ -(f(a_2)G'(a_2))'/G'(a_2),\ \kappa_2(a_2)\ =\ f(a_2)-a_2 \kappa_1(a_2).$
\end{proposition}

Proposition \ref{Rayo-Segal} shows explicitly how the shape of $G$ influences the degree two which prospects can be un-ordered along $\Xi.$ If $G$ is concave, the marginal gain of sending a positive signal to receivers decreases with 
its level. Hence, providing incentives requires the strong ordering of signals: $f(a_2)$ must increase faster than $(G'(a_2))^{1/2}$ decays. The opposite happens when $G$ is convex, and $f'(a_2)$ cannot be too small, as $f$ cannot decay faster than $1/(G'(a_2))^{1/2}.$ The pool of Lebesgue-almost every signal is a line segment, consistent with the findings of \cite{RayoSegal2010} (in addition, pools are convex in the continuous limit and thus have no ``gaps" -- a result that has no discrete counterpart). They also show that these pool lines are downward sloping when $G(b)=b$ (uniform acceptance rate) because, as previously discussed, pooling two different prospects is optimal when we pool prospects of ``comparable" attractiveness (i.e., prospects with low expected $v$ and high expected $\pi$ are pooled with those with high expected $v$ and low expected $\pi$). However, this result does not hold for general $G.$ \cite{RayoSegal2010} write: {\it ``Beyond these results, little can be said about the optimal pooling
graph for arbitrary G, given that its curvature can greatly influence the
outcome. More can be said, however, when the curvature of G is mild.
For example, if G is everywhere concave and its curvature is not strong
enough to lead to pooling of strictly ordered prospects, then all the
additional characterization results in Section IV continue to hold."}
Our results imply that this intuition does not extend to the case of continuous states. As $f(a_2)(G'(a_2))^{1/2} $ is monotone increasing, we get that  $Pool(a_2)$ is downward sloping if $G''(a_2)>0$. In particular, we can characterize the monotonicity of slopes locally, {\it for each value of the signal $a_2.$}

In their conclusions section, \cite{RayoSegal2010} suggest extending their results in two key directions: multiple customer types and multiple product types. Both extensions require developing techniques for tackling complex, multi-dimensional persuasion problems, which -- to the best of our knowledge -- was impossible until now. We illustrate the power of our approach by providing solutions to both extensions. 

\subsection{Supplying Product Information with Multiple Customer Types}

We assume that there are $N$ types of receivers, with independent outside options $r_i$ with different distributions $G_i(r_i).$ The sender is endowed with a prospect randomly drawn from a density $\mu_0(\pi,v).$ Each prospect is characterized by $\go=\binom{\pi}{v}\in\gO\subset\R^{N+1},$ where $\pi\ge 0$ is the prospect’s profitability to the sender and $v=(v_i)_{i=1}^N$ is the vector of its values to the receivers: Receiver $i$ accepts the prospect if and only if the expectation of $v_i$ is higher than $r_i,\ i=1,\cdots,N.$ As a result, the sender's utility is given by 
$W(a)\ =\ a_1\,\sum_{i=1}^N G(a_{i+1})\,,$ where $a_1=E[\pi|s]$ and $a_{i+1}=E[v_i|s],\ i=1,\cdots,N.$ For simplicity, we will assume that, for each $i$, $G_i$ is either strictly convex or strictly concave. Furthermore, we will also assume that the function 
$q(x)\ =\ -\sum_i (G_i'(x_i))^2/G''(x_i)$ does not change the sign for $x\in conv(\gO_{1+}).$\footnote{$\gO_{1+}\in\R^N$ is the projection of $\gO$ onto the last $N$ coordinates.} For example, this is the case when all of $G_i''$ have the same sign. Under these assumptions, $D_{aa}W$ is non-degenerate and Corollary \ref{char-pools} and Proposition \ref{main convexity} allow us to characterize signal pools as well as the local structure of $\Xi.$ 

\begin{proposition}\label{Rayo-Segal1} Let $\nu$ be the number of $G_i$ with $G_i''>0.$ There always exists a pure optimal policy $a(\go)$ such that:

\begin{itemize}
\item The optimal information manifold is a $(\nu+ {\bf 1}_{q(a)>0})$-dimensional Lipschitz manifold, while pools are at most $(N+1-(\nu+ {\bf 1}_{q(a)>0}))$-dimensional. If all $G_i''$ have the same sign, then all pools are convex. 

\item if $G_i''(a)>0$ for all $i,$ then $\nu=N, {\bf 1}_{q(a)>0}=0$ and for each each pure policy there exists a function $f(a_{1+}):\R^N\to \R$ such that $a_1(\go)=f(a_{1+}(\go))$ for all $\go$ and, hence, the optimal information manifold $\Xi$ is an $N$-dimensional subset of the graph $\{(f(a_{1+}),a_{1+})\}.$ For each $i,$ the function $f(a_{1+}) (G_i'(a_{i+1}))^{1/2}$ is monotone increasing in $a_{i+1}$ (i.e., prospects are ordered for each customer type), and there exist functions $\kappa_i(a)$ such that the pool of Lebesgue-almost every signal $a_{1+}$ is a convex subset (a segment) of the one-dimensional line 
\begin{equation}\label{optimal-rs-1}
Pool(a_{1+})\ \subset\ \{\binom{\pi}{v}:\ a_i-v_i\ =\ (\pi-f(a))\kappa_i(a)\ for\ all\ i>1\,\}\ \subset\ R^{N+1}\,. 
\end{equation}
Furthermore, these lines are downward sloping on average in the following sense:
\[
\sum_i\kappa_i(a)G_i'(a)\ \ge\ 0
\]

\item if $G_i''(a)<0$ for all $i,$ then $\nu=0,\ {\bf 1}_{q(a)>0}=1,$ and hence $\Xi$ is a one-dimensional curve. For each pure policy there exists a map $f(a_{1})=(f_i(a_1))_{i=1}^N:\R\to \R^N$ such that $a_{1+i}(\go)=f_i(a_{1}(\go)).$ The function $\sum_i G_i(f_i(a_1))$ is monotone increasing in $a_1$ (i.e., prospects are only ordered on average across customers) and there exists a map $\kappa(a_1):\R\to\R^N$ such that the pool of Lebesgue-almost every signal $a_{1}$ is a convex subset of the $N$-dimensional hyperplane
\begin{equation}\label{optimal-rs-1-1}
Pool(a_{1+})\ =\ \{\binom{\pi}{v}:\ \pi-f(a)\ =\ \kappa(a)^\top (a-v)\}\ \subset\ R^{N+1}\,. 
\end{equation}
\end{itemize}
\end{proposition}

A striking implication of Proposition \ref{Rayo-Segal1} is the difference between the nature of optimal policies with one and multiple receiver types. In the former case (Proposition \ref{Rayo-Segal}), optimal information manifold is always a curve, independent of $G$, and pools are line segments. In contrast, with multiple types, the amount of information revealed (as captured by the dimension of $\Xi$) depends crucially on the curvature of $G.$ If all $G$'s are convex, only one dimension (that of $\pi$) is compressed; meanwhile, if all $G$'s are concave, only one dimension of information is revealed, and $N$ dimensions are compressed. In particular, a simple law or large numbers argument implies that, if the joint distribution of $(v,\pi)$ is symmetric in $v$ and covariances do not go to zero, approximately full revelation is optimal when $N\to\infty$ when all $G$ are convex. Yet when all $G$ are concave, the amount of information revealed does not depend on $N$: it is always optimal to compress $N$ out of $N+1$ dimensions of information. 
In their conclusions section, \cite{RayoSegal2010} argue: {\it ``in this case, we would expect an additional reason
for hiding information, which occurs in models of optimal bundling
with heterogeneous consumers.} Our results show how standard intuition may break down. An approximately full revelation might be optimal with a large number of receiver types. The mechanism underlying this surprising result is clear: Convex acceptance rates make the sender risk-loving, making it optimal for the sender to ``gamble" on large realizations of $v$ by committing to (almost) full revelation ex-ante.  

\subsection{Supplying Information about Multiple Products}

We now consider the case of multiple prospects and assume that each prospect has a different value $(\pi_i,v_i)$. In this case, $
W(a)\ =\ \sum_{i=1}^N a_i\,G_i(a_{i+N})\,,$ 
where $a_i=E[\pi_i|s]$ is the expected payoff for the sender and $a_{i+N}=E[v_i|s]$ is the expected payoff for the customer, and $G_i$ is the probability that the customer accepts prospect $i.$ In this case, $D_{aa}W(a)$ is block-diagonal as there are no cross-effects across different pairs $\binom{a_i}{a_{i+N}}:$ only $a_i$ and $a_{i+N}$ are substitutes in the sender's utility. As a result, $D_{aa}W(a)$ always has exactly $N$ positive eigenvalues, independent of the properties of $G_i.$  In the Appendix (see Proposition \ref{arbitraryGi}), we characterize optimal policies for arbitrary $G_i$. Here, for simplicity we only consider the case of uniform acceptance rates: $G_i(b)=b$ for all $i$. 

Recall that a map $f:\R^N\to \R^N$ is monotone increasing  if $(a-b)^\top (f(a)-f(b))\ge 0$ for any $a,\ b\,\in\R^N.$ Let $a_{N+}=(a_{N+i})_{i=1}^N\in \R^N,\ a_{-N}=(a_{i})_{i=1}^N$ be the vectors of expected payoffs of the different products for the customer and the sender, respectively.  The following is true. 

\begin{proposition} \label{Rayo-Segal2} There always exists a pure optimal policy $a(\go).$ For each such policy, there exists a map $f=(f_i)_{i=1}^N:\ \R^N\to\R^N$ such that $a_{-N}(\go)=f(a_{N+}(\go))$ for all $\go$ and, hence, the optimal information manifold $\Xi$ is an $N$-dimensional subset of the graph $\{(f(a_{N+}),a_{N+})\}$ of the map. Furthermore, the map $f(a_{N+})$ is monotone increasing. The pool of Lebesgue-almost every signal $a_{1+}$ is given by the $N$-dimensional hyperplane 
\begin{equation}\label{optimal-rs-2}
Pool(a_{N+})\ =\ \{\binom{\pi}{v}:\ \pi\ =\ \kappa_1(a_{N+})\,v\ +\ \kappa_2(a_{N+})\ for\ all\ i\,\}\ \subset\ R^{2N}\,. 
\end{equation}
where $\kappa_1(a_{N+})\ =\ -(D_af)^\top \in \R^{N\times N}$ and $\kappa_2(a_{N+})\ =\ \diag(f ){\bf 1}\ -\ \kappa_1(a_{N+})a_{N+}\ \in\ \R^N\,.$
Furthermore, the matrix $\kappa_1(a_{N+})$ is negative semi-definite. 
\end{proposition}

The monotonicity of the map $f$ is the multi-dimensional analog of simple coordinate-wise monotonicity of Propositions \ref{Rayo-Segal} and \ref{Rayo-Segal1}. The intuition behind this monotonicity is clear: It is not optimal to pool two strictly ordered prospects, as pooling only makes sense to create substitution between projects of a similar acceptance rate. In our setting, however, the nature of prospects ordering along the optimal information manifold is more subtle. The monotonicity of $f$ implies that 
$0\ \le\  \sum_i (a_i-b_i)(f_i(a)-f_i(b))$. In other words, for any signal $s,$ the vectors 
$\binom{a_i}{a_{N+i}}=\binom{E[\pi_i|s]}{E[v_i|s]}$ are aligned {\it on average across prospects} whereas individual prospects can be mis-aligned. 

Proposition \ref{Rayo-Segal2} also implies that the linearity of pools is preserved in this multiple prospects case. However, the ``downward sloping" property becomes more subtle. The kind of prospects that get pooled together depends on the substitutability between prospects: For any two prospects $\binom{\pi}{v},\ \binom{\tilde\pi}{\tilde v}\in Pool(a_{N+})$ we have 
$0\ \ge\ \sum_i (\pi_i-\tilde\pi_i) (v_i-\tilde v_i).$ 
Thus, it is optimal to bundle multiple prospects as long as at least some of them are sufficiently mis-aligned. This cross-compensation across different prospects differs from the ``single-prospect-with-multiple-receivers" case of Proposition \ref{Rayo-Segal1} where pooling only happens when the sender's payoff is mis-aligned with every single payoff of the receivers. In the Appendix (see, Proposition \ref{arbitraryGi}), we show how this condition is influenced by the shape of $G$ when $G$ is non-uniform.

\subsection{Concealing the Tails}\label{sec:conceal}
 
In the previous sections, we have characterized the general properties of optimal pools and information manifolds. However, their exact structure depends in a non-trivial way on the prior density $\mu_0(\go).$\footnote{See Proposition \ref{prop-change-variables} in the Appendix for a characterization in terms of an intergro-differential equation.}  In this subsection, we assume that $\mu_0(\go)$ is elliptical. We are particularly interested in the phenomenon described in Proposition \ref{conv-bound}: the emergence of compact optimal information manifolds (e.g., when $\Xi$ is a circle) whereby the sender conceals tail risks from the receiver. We start with an explicit solution to the case with quadratic preferences.

\begin{proposition}[$\Xi$ is a hyper-plane]\label{tamura} Suppose that $W(a)=a^\top H a$ and $g(\go)=\go.$ Define $P_+$ to be the orthogonal projection onto the span of eigenvectors associated with all positive eigenvalues of $V$.  Then,  
$a(\go)\ =\ \Sigma^{1/2}P_+\Sigma^{-1/2}\go\,$ is an optimal policy. In particular, 
\begin{itemize}
\item the optimal information manifold is if the $\nu_+(H)$-dimensional hyperplane $\Xi\ =\ \Sigma^{1/2}P_+\Sigma^{-1/2}\R^M;$ 

\item The pool of every signal is an $(M-\nu(H))$-dimensional hyperplane, 
\[
Pool(a)\ =\ \{\go\ =\ a\ +\ (Id-\Sigma^{1/2}P_+\Sigma^{-1/2})y:\ y\in (Id-\Sigma^{1/2}P_+\Sigma^{-1/2})\R^M\}\cap\gO
\]
\end{itemize}
Furthermore, if $\det(H)\not=0,$ then the optimal policy is unique. In particular, there are no non-linear optimal policies. 
\end{proposition}

Proposition \ref{tamura} is a particularly clean illustration of our key results: $\Xi$ is a $\nu$-dimensional manifold (Theorem \ref{thm-unif}), and pools have dimension $M-\nu(H)$ (Corollary \ref{dimension-pool}) and are convex (Proposition \ref{main convexity}). One interesting observation is that maximality (Theorem \ref{converse}) takes the form of the requirement that $\Xi$ must be spanned by {\it all} eigenvectors with positive eigenvalues. Finally, Proposition \ref{uniqueness} ensures that the policy is unique. 

\cite{tamura2018Bayesian} was the first to show that linear optimal policies of the form described in Proposition \ref{tamura} are optimal when $\mu_0$ is Gaussian. Proposition \ref{tamura} extends his results to general elliptic distributions and establishes the uniqueness of optimal policies. The key simplification in Proposition \ref{tamura} comes from the assumption that $\mu_0$ is elliptic, implying that the optimal policy and the optimal information manifold are linear. 
Consistent with Proposition \ref{conv-bound}, the linear manifold $\Xi$ {\it extends indefinitely in all directions}, implying that tail information (large realizations of good information, $\Sigma^{1/2}P_+\Sigma^{-1/2}\go$) is always revealed. As we know from Corollary \ref{concave-bounded}, the situation changes when we abandon the assumption of quadratic preferences. The following is true.

\begin{corollary}[$\Xi$ is a sphere]\label{sphere} Suppose that $g(\go)=\go\,\psi(\|\go\|^2)$ for some function $\psi\ge 0$ and $\mu_0(\go)=\mu_*(\|\go\|^2),$ and $W(a)\ =\ \varphi(\|a\|^2)$. Let $\beta\ =\ E[\|\go\|].$
If $\varphi'(\beta^2)>0$ and 
\begin{equation}\label{dwor1}
\max_{\|b\|\le \sup_{x\ge 0} (x\psi(x^2))}(\varphi(\|b\|^2)-\varphi(\beta^2)+2\varphi'(\beta^2)\beta (\beta-\|b\|))\ \le\ 0\,, 
\end{equation}
then: (1) $a(\go)=\beta \go/\|\go\|$ is an optimal policy; (2) the optimal information manifold is the sphere $\{\go:\ \|\go\|=\beta\}$; and (3) pools are rays from the origin. The optimal policy is unique if the maximum in \eqref{dwor1} is attained only when $\beta=\|b\|.$ 
\end{corollary}

Condition \eqref{dwor1} means that the graph of the function $\varphi(x^2)$ lies below its tangent at $x=\beta.$\footnote{When this condition is violated, one can consider the affine closure of $\varphi(x^2)$ as in  \cite{DworczakMartini2019}. In this case, the tangent will touch the graph of $\varphi(x^2)$ in several points $r_i$ and the optimal policy will be to project $\go_1$ onto one of the spheres $\|\go_1\|=r_i.$ It is then possible to extend the results of  \cite{DworczakMartini2019} to this nonlinear setting.} In particular, this is the case when the function $\varphi(x^2)$ is concave. As an application, consider a multi-dimensional version of model from \cite{KamGenz2011}, whereby a lobbying group commissions an investigation in order to influence a benevolent politician. The politician (receiver) chooses a multi-dimensional policy $a\in \R^M.$ The state $\go\in \R^M$ and the socially optimal policy is $\go\psi(\|\go\|^2).$ The lobbyist (the sender) is employed by an interest group whose preferred action is $a^*(\go)=\ga \go \psi(\|\go\|^2) +(1-\ga)\go^*.$ The parameter $\ga$ represents the bias of the group towards the preferred action, $\go^*.$ For simplicity, we normalize $\go^*=0.$  The politician's payoff is $u=-\|a-\go\psi(\|\go\|^2)\|^2.$ The lobbyist's payoff is $W(a,\go)=-\|a-a^*(\go)\|^2-\gamma\|a-\go^*\|^4.$ The second term captures the fact that the lobbyist may feel particularly strong about large (tail) deviations of $a$ from $\go^*=0.$ This model specification allows for a specific form of conflict of interest: {\it Mis-alignment of preferences about the tails.}  

Similarly to \cite{KamGenz2011}, upon observing a signal $s,$ the optimal action of the politician is $a\ =\ E[\go\psi(\|\go\|^2)|s],$ and, by direct calculation, the expected payoff of the lobbyist is $E[W(a,\go)|s]\ =\ -\ga^2E[\|\go\|^2\psi^2(\|\go\|^2)|s] +E[-\gamma\|a\|^4\  +\ (2\ga-1)\|a\|^2|s].$ As in  \cite{KamGenz2011}, for $\ga\le 0.5,$ the lobbyist's and politician's preferences are so mis-aligned that no disclosure is uniquely optimal. However, when $\ga>0.5,$ things change. Absent tail mis-alignment, $\gamma=0$ and full revelation is optimal. Yet a sufficient degree of mis-alignment  changes the picture. By direct calculation, if $\beta=E[\|\go\|]$ is sufficiently large ($\beta\ge \left(\frac{2\ga-1}{2\gamma}\right)^{1/2}),$ condition \eqref{dwor1} holds; therefore, it is optimal to fully conceal information about the magnitude of $\go$ and {\it only reveal the direction of $\go,$} as given by $\go/\|\go\|.$ A real world counter-part of such a policy could be revealing that ``$\go$ is good" or ``$\go$ is bad", without revealing how good/bad it is. 

\section{Conclusion} \label{sec:concl}

Bayesian persuasion is growing in popularity as a model of optimal communication with commitment. Most of its real-life applications are in informationally complex environments, where the signal is high-dimensional. We show that the optimal way to communicate a high-dimensional signal is through dimension reduction, achieved by projecting the signal onto a lower-dimensional {\it optimal information manifold}. We derive several analytical results regarding the shape and the geometry of the optimal information manifold and the corresponding optimal pools. We show how our results can be used to shed new light on several classic persuasion problems. In particular, we show when it is optimal for the sender to conceal the tails and project the signal onto a compact manifold. We see two potentially important directions for future research. First, it would be great to extend our analysis to the case where the sender is uncertain about the underlying distribution and the actions of the receivers, as in \cite{DworzakPavan2020}. Second, it would be very interesting to understand how our results change when extended to a dynamic setting where information is gradually revealed over time. 

\newpage

{\bf \Huge Internet Appendix}

\bibliographystyle{aer}
\bibliography{bibliography}

\newpage

\appendix

\section{Discrete Approximation}

We call an information design $K$-finite if the support $\tau$ has cardinality $K.$ An optimal $K$-finite information design is the one attaining the highest utility for the sender {\it among all $K$-finite designs.} A pure $K$-finite design corresponds to an optimal policy $a(\go)$ that only takes $K$ different values $a_1,\cdots,a_K.$ In this case, $\gO_k=\{\go:\ a(\go)=a_k\}$ defines a partition of $\gO$ and one possible implementation of the design for the sender is to tell receivers to which $\gO_k$ the state $\go$ belong. 
  
The following is the main result of this section. 

\begin{theorem}[Optimal $K$-finite information design]\label{mainth1} There always exists an optimal $K$-finite information design which is a partition. 
\end{theorem}

While Theorem \ref{mainth1} may seem intuitive, its proof is non-trivial and is based on novel techniques (see the appendix) using the theory of real analytic functions. The reason is that {\it the set of $K$-finite designs is not convex} and hence the duality approach of \cite{KamGenz2011} does not apply. Indeed, a convex combination of two $K$-finite designs is a $2K$-finite design. First, we prove that when $W$ and $G$ are real analytic and satisfy additional regularity conditions, then {\it any optimal $K$-finite design is a partition and hence randomization is never optimal.} Then, the general claim follows by a simple approximation argument because any function can be approximated by a real analytic function. Note that the regularity of both $W$ and $G$ (ensured by Lemma \ref{existence}) are crucial for the partition result. Without such regularity, classic examples of Bayesian persuasion (see, e.g., \cite{KamGenz2011}) show that randomization can be optimal. 

In general, the structure of optimal partitions in Theorem \ref{mainth1} can be complex. However, in many real-world persuasion mechanisms pooling only happens between contiguous types.\footnote{This is the case for most ranking mechanisms, such as school grades, credit, restaurant, and hotel rating; see, for example, \cite{ostrovsky2010} and \cite{Hopenhayn2019}.} In the one dimensional case, such signal structures are commonly known as monotone partitions; they also appear as equilibria in communication models without commitment, such as the cheap talk model of \cite{crawford1982strategic}. To the best of our knowledge, the most general result currently available guaranteeing existence of monotone partitional signals is that of \cite{DworczakMartini2019}. They show that when (i) the utility function $W$ only depends on the receivers' actions; (ii) the state space is one-dimensional, $\go\in \R;$ (iii) the actions are given by the expected state, a monotone partitional optimal information design exists for any continuous prior if and only if $W$ is affine-closed. Roughly speaking, an affine-closed function is such that $W+q$ has at most one local maximum for any affine function $q.$ 

Extending the monotone partitions result of \cite{DworczakMartini2019} to multiple dimensions is far from trivial. 
Here, we prove an analog of the monotone partitions result of \cite{DworczakMartini2019} for multiple dimensions. In this case, the natural analog of a monotone partition is a partition into a union of convex subsets, which then naturally have to be polygons. We will need the following definition.

\begin{definition} We say that a function $W=W(a)$ has a unique tangent hyperplane property is tangent hyperplanes to the graph of $W(a)$ at any two distinct points $a_1\not=a_2.$ are distinct. Equivalently, any tangent hyperplane is tangential to the graph at a single point. 
\end{definition}

By direct calculation, $W$ has a unique tangent hyperplane property if and only if the map $a\to\ \binom{D_aW(a)}{W(a)-D_aW(a)\,a}$ is injective. For example, this is always the case when the gradient map $a\to D_aW$ is injective; e.g., when $W(a)=a^\top H a+b^\top a$ with a non-degenerate matrix $H.$ The following is true. 

\begin{proposition}\label{main convexity-discrete} Suppose that $G=a-g(\go)$ and $W(a,\go)=W(a)$ has a unique tangent hyperplane property. Let $X\subset \gO$ be an open set suppose that $g$ is injective on $X$ and $g(X)$ is convex. Then, the set $g(\gO_k\cap X)$ is convex for each $k.$ In particular, $\gO_k\cap X$ is connected.\footnote{Furthermore, the map $x\to D_aW(a(g^{-1}(x)))$ is monotone increasing on $g(\gO_k\cap X).$ A map $F$ is monotone increasing if $(x_1-x_2)^\top (F(x_1)-F(x_2))\ge 0$ for all $x_1,x_2.$} 
\end{proposition}

In the case when $g(\go)=\go,$ Proposition \ref{main convexity-discrete} is a direct extension of \cite{DworczakMartini2019} to multiple dimensions. 
It is easy to see that the unique tangent hyperplane property is a stronger condition than the affine closeness of \cite{DworczakMartini2019}. However, in multiple dimensions, the geometry of the problem is much more complex and we conjecture that affine closeness is not anymore sufficient for monotone partitions.\footnote{When $g\not=\go,$ things get more complex. Injectivity of the $g(\go)$ map and convexity of its image are crucial for the connectedness of the $\gO_k$ regions. Without injectivity, even bounding the number of connected components of $\gO_k$ is non-trivial. These effects become particularly strong in the limit when $K\to\infty,$ where lack of injectivity in the map $g$ may lead to a breakdown of even minimal regularity properties of the optimal map. }

\section{Finite Partitions: Proofs}

There are four time periods, $t\ =\ 0-,\ 0,\ 0+,\ 1.$ The information designer (the sender) believes that the state $\go$ (the private information of the sender) is a random vector taking values in  $\gO\subset \R^L$, an open subset of $\R^L$, and distributed with a density $\mu_0(\go)$ that is strictly positive on $\gO.$ There are $N\ge 1$ agents (receivers) who share the same prior $\mu_0(\go).$ 

Following \cite{KamGenz2011}, we assume that the sender is able to \emph{commit} to an information design at time $t=0-,$ before the state $\go$ is realized. 
The sender learns the realization of the state $\go$ at time $t=0,$ while the receivers only learn it at time $t=1.$ The sender's objective is then to decide how much, and what kind of, information about $\go$ to reveal to the receivers at time $t=0$.

\begin{definition}[Finite Information Design] An information design is a probability space $\cK$ (hereinafter signal space) and a probability measure $\cP$ on $\cK\times \gO.$ An information design is $K$-finite if the signal space $\cK$ has exactly $K$ elements: $|\cK|=K.$ An information design is finite if it is $K$-finite for some $K\in\N.$ In this case, without loss of generality we assume that $\cK\ =\ \{1,\cdots,K\}.$
\end{definition}

Once the receivers observe a signal $k\in \{1,\cdots,K\}$, they update their beliefs about the probability distribution of $\go$ using the Bayes rule. To do this, the receivers just need to know $\pi_k(\go)$,  the probability of observing signal $k$ given that the true state is $\go.$\footnote{That is, $\pi_k(\go)=\Prob(k|\go).$ By the Bayes rule, $\Prob(\go|k)\ =\ \frac{\pi_k(\go)\mu_0(\go)}{\int_\gO\pi_k(\go)\mu_0(\go)d\go}.$ The joint distribution is then $\cP(k,\go\in A)\ =\ \int_{\go\in A}\pi_k(\go)\mu_0(\go)d\go\,.$} As such, a $K$-finite information design can be equivalently characterized by a set of measurable functions $\pi_k(\go),\ k=\{1,\cdots,K\}$ satisfying conditions $\pi_k(\go)\in[0,1]$ and $\sum_k\pi_k(\go)\ =\ 1$ with probability one. 

Intuitively, an information design is a map from the space $\gO$ of possible states to a ``dictionary'' of $K$ messages, whereby the sender commits to a precise rule of selecting a signal from the dictionary for every realization of $\go.$ In principle, it is possible that this rule involves randomization, whereby, for a given $\go,$ the sender randomly picks a signal from a non-singleton subset of messages in the dictionary.
An information design does not involve randomization if and only if it is a partition of the state space $\gO.$ 

\begin{definition}[Randomization]\label{dfn-main}
We say that information design involves randomization if $\{\go:\ \pi_k(\go)\not\in\{0,1\}\}$ has a positive Lebesgue measure for some $k$ . We say that information design is a partition if $\cP(\pi_k(\go)\in\{0,1\})=1$ for all $k=1,\cdots,K.$ In this case,
\begin{equation}\label{part-1}
\cup_{k=1}^K \{\go:\ \pi_k(\go)=1\}
\end{equation}
is a Lebesgue-almost sure partition of $\gO$ in the sense that $\gO\setminus \cup_{k=1}^K \{\go:\ \pi_k(\go)=1\}$ has Lebesgue measure zero, and the subsets of the partition \eqref{part-1} are Lebesgue-almost surely disjoint.
\end{definition}

We use $\bar\pi\ =\ (\pi_k(\go))_{k=1}^K\in [0,1]^K$ to denote the random $K$-dimensional vector representing the information design. As we show below, a key implication of this setting is that, with a continuous state space and under appropriate regularity conditions, randomization is never optimal, and hence optimal information design is always given by a {\it partition}. While this result might seem intuitive, its proof is non-trivial and is based on novel techniques that, to the best of our knowledge, have never been used in the literature before. It is this result that is key our subsequent analysis of the unconstrained problem. 

\subsection{Receivers}\label{der-g}

At time $t=0,$ upon observing a signal $k,$ each agent (receiver) $n=1,\cdots,N$ selects an action $a_n\in\R^m$ to maximize the expected utility function
\[
E[U_n(a_{n},a_{-n},\go)|k]\,,
\]
where we use $a_{-n}\ =\ (a_i)_{i\not=n}\in \R^{(N-1)m}$ to denote the vector of actions of other agents. We denote by  $a=(a_n)_{n=1}^N\in \R^M,\ M=Nm$ the vector of actions of all agents. A Nash equilibrium action profile $a(k)\ =\ (a_n(k))_{n=1}^N$ is a solution to the fixed point system
\begin{equation}\label{sys-1}
a_n(k)\ = \ \arg\max_{a_n} E[U_n(a_{n},a_{-n}(k),\go)|k]\,.
\end{equation}
We use $C^2(\gO)$ to denote the set of functions that are twice continuously differentiable in $\gO$. We will also use $D_a$ and $D_{aa}$ to denote the gradient and the Hessian with respect to the variable $a.$ Let 
\begin{equation}
G_n(a,\go)\ \equiv\ D_{a_n}U_n(a_{n},a_{-n},\go)\,.
\end{equation}

\begin{assumption}\label{ac-discrete} There exists an integrable majorant $Y(\go)\ge 0$ such that $Y(\go)\ge U_n(a_n,a_{-n},\go)$ for all $a\in \R^{M},\ \go\in\gO,\ n=1,\cdots,N.$ 
The function $U_n(a_n,a_{-n},\go)\in C^2(\R^m\times\R^{Nm}\times\gO)$ is strictly concave in $a_n$, and is such that $\lim_{\|a_n\|\to\infty}U_n=-\infty$.\footnote{This is a form of Inada condition ensuring that the optimum is always in the interior. An integrable majorant is needed to apply Fatou lemma and conclude that $\lim_{\|a_n\|\to\infty}E[U_n(a_n,a,\go)|k]= -\infty$ always.}

Furthermore, the map $G\ =\ (G_n)_{n=1}^N\,:\ \R^{M}\times \gO\to\R^{M}$ satisfies the following conditions:  
\begin{itemize}
\item $G$ is uniformly monotone in $a$ for each $\go$ so that $\eps\|z\|^2\ \le -z^\top D_aG(a,\go)z\le\eps^{-1}\|z\|^2$ for some $\eps>0$ and all $z\in \R^M;$\footnote{Strict monotonicity is important here. Without it, there could be multiple equilibria.} 

\item the unique solution $a_*(\go)$ to $G(a_*(\go),\go)=0$ is square integrable: $E[\|a_*(\go)\|^2]<\infty.$
\end{itemize}
\end{assumption}

To state the main result of this section ---the optimality of partitions---we need also the following definition.

\begin{definition}
We say that functions $\{f_1(\go),\cdots,f_{L_1}(\go)\},\ \go\in\gO,$ are linearly independent modulo $\{g_1(\go),\cdots,g_{L_2}(\go)\}$ if there exist no real vectors $h\in \R^{L_1},\ k\in \R^{L_2}$ with $\|h\|\not=0,$ such that
\[
\sum_i h_i f_i(\go)\ =\ \sum_j k_j g_j(\go)\qquad for\ all\ \go\in\go\,.
\]
In particular, if $L_1=1,$ then $f_1(\go)$ is linearly independent modulo $\{g_1(\go),\cdots,g_{L_2}(\go)\}$ if $f_1(\go)$ cannot be expressed as a linear combination of $\{g_1(\go),\cdots,g_{L_2}(\go)\}.$
\end{definition}

We also need the following technical condition.

\begin{definition}\label{main-ass-indep} We say that $W, G$ are in a generic position if for any fixed $a,\tilde a\in \cR^N,\ a\not=\tilde a$, the function $W(a,\go)-W(\tilde a,\go)$ is linearly independent modulo $\left\{\{G_n(a,\go)\}_{n=1}^N,\{G_n(\tilde a,\go)\}_{n=1}^N\right\}$;
\end{definition}

$W, G$ are in generic position for generic functions $W$ and $G$.\footnote{The set of $W,G$ that are not in generic position is nowhere dense in the space of continuous functions.} We will also need a key property of real analytic functions\footnote{A function is real analytic if it can be represented by a convergent power series in the neighborhood of any point in its domain.} that we use in our analysis (see, e.g., \cite*{HugonnierMalamudTrubowitz2012}).

\begin{proposition}\label{zero-go} If a real analytic function $f(\go)$ is zero on a set of positive Lebesgue measure, then $f$ is identically zero. Hence, if real analytic functions $\{f_1(\go),\cdots,f_{L_1}(\go)\}$ are linearly dependent modulo $\{g_1(\go),\cdots,g_{L_2}(\go)\}$ on some subset $I\subset\gO$ of positive Lebesgue measure, then this linear dependence also holds on the whole $\gO$ except, possibly, a set of Lebesgue measure zero.
\end{proposition}

Using Proposition \ref{zero-go}, it is possible to prove the main result of this section:

\begin{theorem}[Optimal finite information design]\label{mainth1-analytic} There always exists an optimal $K$-finite information design $\bar\pi^*$ which is a partition. Furthermore, if $W,\ G$ are real analytic in $\go$ for each $a$ and are in generic position, then any $K$-finite optimal information design is a partition.
\end{theorem}

\begin{proof}[Proof of Lemma \ref{existence}] First, by uniform monotonicity, the map 
\[
a\to\ F(a)= a+\gd E[G(a,\go)|k]
\]
is a contraction for sufficiently small $\gd.$ Indeed, by monotonicity, 
\[
\|F(a_1)-F(a_2)\|^2\ \le\ \|a_1-a_2\|^2-2\eps\gd\|a_1-a_2\|^2\ +\ \gd^2\eps^{-2} \|a_1-a_2\|^2\,. 
\]
As a result, there exists a unique equilibrium by the Banach fixed point theorem. Then, with $a=a(k),$
\begin{equation}
\begin{aligned}
&E[(a_*(\go)-a)^\top\,G(a,\go)|k]\ =\ E[(a_*(\go)-a)^\top\,(G(a,\go)-G(a_*(\go),\go))|k]\\ 
&\ge\ \eps\,E[\|a_*(\go)-a\|^2|k]\ \ge\ \eps(E[\|a_*(\go)\|^2+2\|a\|\|a_*(\go)\| |k]\ +\ \|a\|^2)\,. 
\end{aligned}
\end{equation}
At the same time, 
\begin{equation}
\begin{aligned}
&E[(a_*(\go)-a)^\top\,G(a,\go)|k]\ =\ E[a_*(\go)^\top\,G(a,\go)|k]\ \le\ \eps^{-1} E[\|a_*(\go)\|\,\|a-a_*(\go)\||k]\\
& =\ \eps^{-1}(E[\|a_*(\go)\|^2|k]+\|a\|E[\|a_*(\go)\||k])
\end{aligned}
\end{equation}
and the claim follows. 
\end{proof}

\begin{proof}[Proof of Theorem \ref{mainth1}] 
The fact that utility  is bounded and depends smoothly on the information design follows by the same arguments as in the proof of Lemma \ref{lem-approx}.

Existence of an optimal information design then follows trivially from compactness. Indeed, since $\pi_k(\go)\in [0,1]$, the are square integrable and, hence, compact in the weak topology of $L_2(\mu_0).$ The identity $\sum_k \pi_k=1$ is trivially preserved in the limit. Continuity of utility  in $\pi_k$ follows directly from the assumed integrability and regularity, hence the existence of an optimal design. 

The equilibrium conditions can be rewritten as
\begin{equation}
E_{\mu_s}[G(a(s),\go)|s]\ =\ 0\,.
\end{equation}
Here,
\[
\mu_s(\go)\ =\ \frac{\pi(s|\go)\mu_0(\go)}{\int \pi(s|\go)\mu_0(\go) d\go}
\]
and hence
\[
E_{\mu_k}[G(a(s),\go)]\ =\ \frac{\int \pi(k|\go)\mu_0(\go) G(a(k),\ \go)d\go}{\int \pi(k|\go)\mu_0(\go) d\go}\,.
\]
By assumption, equilibrium $a$ depends continuously on $\{\pi_k\}.$ Since the map
\[
(\{\pi_k\},\ \{a_k\})\ \to\ \left\{ \int \pi_k(\go)\mu_0(\go) G(a(k,\eps),\ \go)d\go\right\}
\]
is real analytic, and has a non-degenerate Jacobian with respect to $a,$ the assumed continuity of $a$ and the implicit function theorem imply that $a$ is in fact real analytic in $\{\pi_k\}.$ To compute the Frechet differentials of $a(s),$ we take a small perturbation $\eta(\go)$ of $\pi_k(\go)$. By the regularity assumption and the Implicit Function Theorem,
\[
a(k,\eps)\ =\ a(k)+\eps a^{(1)}(k)+0.5\eps^2a^{(2)}(k)\ +\ o(\eps^2)
\]
for some $a^{(1)}(k),\ a^{(2)}(k)\,.$ Let us rewrite
\begin{equation}
\begin{aligned}
&0\ =\ \int (\pi_k(\go)+\eps \eta(\go))\mu_0(\go) G(a(k,\eps),\ \go)d\go\\
&=\ \int (\pi_k(\go)+\eps \eta(\go))\mu_0(\go) G(a(k)+\eps a^{(1)}(k)+0.5\eps^2a^{(2)}(k),\ \go)d\go\\
&\approx\ \Bigg(
\int \pi_k(\go) \mu_0(\go)\Bigg(G(a(k),\ \go)+G_a (\eps a^{(1)}(k)+0.5\eps^2a^{(2)}(k))\\
&+0.5 G_{aa}(\eps a^{(1)}(k), \eps a^{(1)}(k))
\Bigg)d\go\\
&+\eps \int \eta(\go)\mu_0(\go) \Bigg(G(a(k))+G_a\eps a^{(1)}(k)
\Bigg)d\go
\Bigg)\\
&=\ \Bigg(\eps
\Bigg(
\int \pi_k(\go) \mu_0(\go)G_a a^{(1)}(k)d\go+\int \eta(\go)\mu_0(\go) G(a(k))d\go\Bigg)\\
&+0.5\eps^2
\Bigg(\int \pi_k(\go) \mu_0(\go)[G_a a^{(2)}(k)+G_{aa}(a(k),\go)( a^{(1)}(k), a^{(1)}(k))]d\go\\
&+2\int \eta(\go)\mu_0(\go) G_a(a(k),\go) a^{(1)}(k)d\go\Bigg)
\Bigg)
\end{aligned}
\end{equation}
As a result, we get
\begin{equation}\label{x1}
\begin{aligned}
&a^{(1)}(k)\ =\ -\bar G_a(k)^{-1}\,\int \eta(\go)\mu_0(\go) G(a(k),\go)d\go,\ \bar G_a(k)\ =\ \int \pi_k(\go) \mu_0(\go)G_a d\go\,,
\end{aligned}
\end{equation}
while
\begin{equation}\label{x2}
\begin{aligned}
&a^{(2)}(k)\ =\ -\bar G_a(k)^{-1}\,\Bigg(
\int \pi_k(\go) \mu_0(\go)G_{aa}(a(k),\go)( a^{(1)}(k), a^{(1)}(k))d\go\\
& +\ 2\int \eta(\go)\mu_0(\go) G_a(a(k),\go) a^{(1)}(k)d\go
\Bigg)\,.
\end{aligned}
\end{equation}
Consider the utility  function
\begin{equation}
\begin{aligned}
&\bar W(\pi)\ =\ E[W(a(s),\go)]\ =\ \sum_k \int_\gO W(a(k),\go)\pi_k(\go)\mu_0(\go)d\go\,.
\end{aligned}
\end{equation}
Suppose that the optimal information structure is not a partition. Then, there exists a subset $I\subset\gO$ of positive $\mu_0$-measure and an index $k$ such that $\pi_k(\go)\in (0,1)$ for $\mu_0$-almost all $\go\in I.$ Since $\sum_i\pi_i(\go)=1$ and $\pi_i(\go)\in[0,1],$ there must be an index $k_1\not=k$ and a subset $I_1\subset I$ such that $\pi_{k_1}(\go)\in (0,1)$ for $\mu_0$-almost all $\go\in I_1.$ Consider a small perturbation $\{\tilde\pi(\eps)\}_i$ of the information design, keeping $\pi_i,\ i\not=k,k_1$ fixed and changing $\pi_k(\go)\to \pi_k(\go)+\eps \eta(\go),\ \pi_{k_1}(\go)\to\pi_{k_1}(\go)-\eps(\go)$ where $\eta(\go)$ in an arbitrary bounded function with $\eta(\go)=0$ for all $\go\not\in I_1.$ Define $\eta_k(\go)=\eta(\go),\ \eta_{k_1}(\go)=-\eta(\go),$ and $\eta_i(\go)=0$ for all $i\not=k,k_1.$ A second-order Taylor expansion in $\eps$ gives
\begin{equation}\label{w-expan1}
\begin{aligned}
&\sum_{i} \int_\gO W(a(i,\eps),\go) (\pi_i(\go)+\eps \eta_i(\go))\mu_0(\go)d\go\\
&\approx\  \int_\gO \Bigg(W(a(i),\go)+W_a(a(i),\go)(\eps a^{(1)}(i)+0.5\eps^2a^{(2)}(i))\\
&+0.5 W_{aa}(a(i),\go)\eps^2( a^{(1)}(i), a^{(1)}(i))\Bigg) (\pi_i(\go)+\eps \eta_i(\go))\mu_0(\go)d\go\\
&=\ \bar W(\pi)\ +\ \eps\sum_i \Bigg(
\int_\gO (W(a(i),\go) \eta_i(\go)+ W_a(a(i),\go)a^{(1)}(i)\pi_i(\go))\mu_0(\go)d\go
\Bigg)\\
&+0.5\eps^2
 \sum_i \int_\gO\Big( W_{aa}(a(i),\go)( a^{(1)}(i), a^{(1)}(i))\pi_i(\go)\\
 &+W_a(a(i),\go) a^{(2)}(i) \pi_i(\go)
 +W_a(a(i),\go) a^{(1)}(i)\eta_i(\go)
 \Big) \mu_0(\go)d\go
\end{aligned}
\end{equation}
Since, by assumption, $\{\pi_i\}$ is an optimal information design, it has to be that the first order term in \eqref{w-expan1} is zero, while the second-order term is always non-positive. We can rewrite the first order term as
\begin{equation}\label{w-expan}
\begin{aligned}
&\sum_i \Bigg(
\int_\gO (W(a(i),\go) \eta_i(\go)+ W_a(a(i),\go)a^{(1)}(i)\pi_i(\go))\mu_0(\go)d\go
\Bigg)\\
&=\ \sum_i \int_\gO \Bigg(W(a(i),\go)\\
& -\  \Big(\int W_a(a(i),\go_1)\pi_i(\go_1)\mu_0(\go_1)d\go_1
\Big)\bar G_a(i)^{-1}\, G(a(i),\go)
\Bigg)
\eta_i(\go)\mu_0(\go)d\go
\end{aligned}
\end{equation}
and hence it is zero for all considered perturbations if and only if
\begin{equation}\label{w-expan2}
\begin{aligned}
&W(a(k),\go)\ -\ \Big(\int W_a(a(k), \go_1)\pi_k(\go_1)\mu_0(\go_1)d\go_1
\Big)\bar G_a(k)^{-1}\, G(a(k),\go)\\
&=\ W(a(k_1),\go)\ -\ \Big(\int W_a(a(k_1),\go)\pi_{k_1}(\go_1)\mu_0(\go_1)d\go_1
\Big)\bar G_a(k_1)^{-1}\, G(a(k_1),\go)
\end{aligned}
\end{equation}
Lebesgue-almost surely for $\go\in I_1.$ By Proposition \ref{zero-go}, \eqref{w-expan2} also holds for all $\go\in\gO.$ Hence, by Assumption \ref{main-ass-indep}, $a(k)=a(k_1),$ which contradicts our assumption that all $a(k)$ are different. 

Let now $W_n(a,\go)$ be a sequence of utility functions in generic positions, uniformly converging to $W(a,\go).$ Let $\{\gO_k(n)\}_{k=1}^K$ be the respective partitions, and $\pi_k(n)={\bf 1}_{\gO_k(n)}.$ Passing to a subsequence, we $\pi_k(n)\to\pi_k^*$ for each $k$ when $n\to\infty.$ Passing to a subsequence once again, we may assume that $a_n\to a_*.$ Now, for any $K$-finite information design $\{\tilde\pi_k\},$ 
\[
W(\{\tilde\pi_k\})\ =\ \lim_{n\to\infty} W_n(\{\tilde\pi_k\})\ \le\ \lim_{n\to\infty} W_n(\{\pi_k(n)\})\ =\ W(\{\pi_k^*\})
\]
where the last result follows from uniform continuity of $W_n$, uniform convergence, and compactness of $\gO.$ 
\end{proof}

\subsection{The Structure of Optimal Partitions}

The goal of this section is to provide a general characterization of the ``optimal clusters" in Theorem \ref{mainth1}.

We use $D_aG(a,\go)\in \R^{M\times M}$ to denote the Jacobian of the map $G$, and, similarly, $D_aW(a,\go)\in \R^{1\times M}$ the gradient of the utility  function $W(a,\go)$ with respect to $a.$ For any vectors $x_k\ \in\ \R^{M},\ k=1,\cdots,K$ and actions $\{a(k)\}_{k=1}^K,$ let us define the partition
\begin{equation}\label{partitions1}
\begin{aligned}
\gO_k^*(\{x_\ell\}_{\ell=1}^K,\{a_\ell\}_{\ell=1}^K)\ &=\
\Bigg\{
\go\ \in\ \gO\ :\ W(a(k),\ \go)-x_k^\top G(a(k),\ \go)\\
& =\ \max_{1\le l\le K} \left(W(a(l),\ \go)\ -\ x_l^\top G(a(l),\ \go)\right)
\Bigg\}
\end{aligned}
\end{equation}
Equation \eqref{partitions1} is basically the first-order condition for sender's optimization problem, whereby $x_k$ are the Lagrange multipliers of agents' participation constrains.

\begin{theorem}\label{regular-partition} Any optimal partition in Theorem \ref{mainth1} satisfies the following conditions:
\begin{itemize}
\item local optimality holds: $\gO_k\ =\ \gO_k^*(\{x_\ell\}_{\ell=1}^K,\{a_\ell\}_{\ell=1}^K)$ with $x_k^\top\ =\ \bar D_aW(k)(\bar D_aG(k))^{-1}\,,$
where we have defined for each $k=1,\cdots,K\,$
\begin{equation}
\begin{aligned}
&\bar D_aW(k)\ =\ \int_{\gO_k} D_aW(a(k),\go)\mu_0(\go)d\go\,,\ \bar D_aG(k)\ =\ \int_{\gO_k} D_{a}G(a(k),\go)\mu_0(\go)d\go
\end{aligned}
\end{equation}
\item the actions $\{a(k)\}_{k=1}^K$ satisfy the fixed point system
\begin{equation}\label{gak1}
\int_{\gO_k} G(a(k),\go)\mu_0(\go)d\go\ =\ 0,\ k=1,\,\cdots,\,K\,.
\end{equation}
\item the boundaries of $\gO_k$ are a subset of the variety\footnote{This variety is real analytic when so are $W$ and $G.$ A real analytic variety in $\R^m$ is a subset of $\R^m$ defined by a set of identities $f_i(\go)=0,\ i=1,\cdots,I$ where all functions $f_i$ are real analytic. If at least one of functions $f_i(\go)$ is non-zero, then a real analytic variety is always a union of smooth manifolds and hence has a Lebesgue measure of zero. When $W,G$ are real analytic and are in generic position, the variety $\left\{\go\in \R^m\ :\ W(a(k),\ \go)\ -\ x_k^\top G(a(k),\ \go)\ =\ W(a(l),\ \go)\ -\  x_l^\top G(a(l),\ \go)\right\}$ has a Lebesgue measure of zero for each $k\not=l.$}
\begin{equation}\label{indiff}
\cup_{k\not=l}\left\{\go\in \R^m:W(a(k),\go)-x_k^\top G(a(k),\go)=W(a(l),\go)-x_l^\top G(a(l),\go)\right\}\,.
\end{equation}
\end{itemize}
\end{theorem}

A key insight of Theorem \ref{regular-partition} comes from the characterization of the different clusters of an optimal partition. The sender has to solve the problem of maximizing utility by inducing the desired actions, $(a_n),$ of economic agents for every realization of $\go.$ Ideally, the sender would like to induce $a_*\ =\ \arg\max_a W(a,\go)\,.$ However, the ability of the sender to elicit the desired actions is limited by the participation constraints of the receivers – that is, the map from the posterior beliefs induced by communication to the actions of the receivers. Indeed, while the sender can induce any Bayes-rational beliefs (i.e., any posteriors consistent with the Bayes rule; see \cite{KamGenz2011}), she has no direct control over the actions of the receivers. The degree to which these constraints are binding is precisely captured by the Lagrange multipliers $x(a),$ so that the sender is maximizing the Lagrangian $\max_a(W(a,\go)-x(a)^\top G(a,\go)).$  Formula \eqref{partitions1} shows that, inside the cluster number $k$, the optimal action profile maximizes the respective Lagrangian. The boundaries of the clusters are then determined by the indifference conditions \eqref{indiff}, ensuring that at the boundary between regions $k$ and $l$ the sender is indifferent between the respective action profiles $a_k$ and $a_l.$

Several papers study the one-dimensional case (i.e., when $L=1$ so that $\go\in \R^1$) and derive conditions under which the optimal signal structure is a monotone partition into intervals. Such a monotonicity result is intuitive, as one would expect the optimal information design to only pool nearby states. The most general results currently available are due to \cite{Hopenhayn2019} and \cite{DworczakMartini2019},\footnote{See also \cite{mensch2018}.} but they cover the case when sender's utility (utility  function in our setting) only depends on $E[\go]\in \R^1.$\footnote{This is equivalent to $G(a,\go)\ =\ a-\go$ in our setting. In this case, formula \eqref{gak1} implies that the optimal action is given by $a(k)=E[\go|k].$} Under this assumption, \cite{DworczakMartini2019} derive necessary and sufficient conditions guaranteeing that the optimal signal structure is a monotone partition of $\gO$ into a union of disjoint intervals. \cite{Arielietal2020} (see, also, \cite{kleiner2020extreme}) provide a full solution to the information design problem when $a(k)=E[\go|k]$ and, in particular, show that the partition result does not hold in general when the signal space is continuous. 
Theorem \ref{regular-partition} proves that a $K$-finite optimal information design is in fact always a partition when the state space is continuous and the signal space is discrete. However, no general results about the monotonicity of this partition can be established without imposing more structure on the problem.\footnote{Of course, as \cite{DworczakMartini2019} and \cite{Arielietal2020} explain, even in the one-dimensional case the monotonicity cannot be ensured without additional technical conditions. No such conditions are known in the multi-dimensional case. \cite{DworczakMartini2019} present an example with four possible actions $(K=4)$ and a two-dimensional state space $(L=2)$ for which they are able to show that the optimal information design is a partition into four convex polygons. See, also, \cite{DworzakKolotilin2019}.}

Consider the optimal information design of Theorem \ref{regular-partition} and define the piece-wise constant function\,
\begin{equation}
a(\go)\ =\ \sum_k a(k) {\bf 1}_{\go \in \gO_k}\,.
\end{equation}
Since utility  can be written as 
\[
E[W(a(\go),\go)],
\]
this function encodes all the properties of the optimal information design.

\begin{proof}[Proof of Theorem \ref{regular-partition}]
Suppose a partition $\go\ =\ \cup_k \gO_k$ is optimal. By regularity, equilibrium actions satisfy the first order conditions
\[
\int_{\gO_k} G(a(k),\ \go)\mu_0(\go)d\go\ =\ 0\,.
\]
Consider a small perturbation, whereby we move a small mass on a set $\cI\subset \gO_k$ to $\gO_l.$ Then, the marginal change in $a_n(k)$ can be determined from
\begin{equation}
\begin{aligned}
&0\ =\ \int_{\gO_k} G(a(k),\ \go)\mu_0(\go)d\go\ -\ \int_{\gO_k\setminus\cI} G(a(k,\cI),\ \go)\mu_0(\go)d\go\\
&\approx\ -\int_{\gO_k} D_aG(a(k),\ \go)\Delta a(k)\ \mu_0(\go)d\go\ + \int_{\cI} G(a(k),\ \go)\mu_0(\go)d\go\,,
\end{aligned}
\end{equation}
implying that the first order change in $a$ is given by
\[
\Delta a(k)\ \approx\ (\bar D_aG(k))^{-1}\int_{\cI} G(a(k),\ \go)\mu_0(\go)d\go\,.
\]
Thus, the change in utility  is\footnote{Note that $D_aW$ is a horizontal (row) vector.}
\begin{equation}
\begin{aligned}
&\Delta W\ =\ \int_{\gO_k}W(a(k),\ \go)\mu_0(\go)d\go\ -\ \int_{\gO_k\setminus \cI}W(a(k,\cI),\ \go)\mu_0(\go)d\go\\
&+\int_{\gO_l}W(a(l),\ \go)\mu_0(\go)d\go\ -\ \int_{\gO_l\cup \cI}W(a(l,\cI),\ \go)\mu_0(\go)d\go\\
&\approx\ -\int_{\gO_k}D_aW(a(k),\ \go)\Delta a(k)\mu_0(\go)d\go+\int_{\cI}W(a(k),\ \go)\mu_0(\go)d\go\\
&-\int_{\gO_l}D_aW(a(l),\ \go)\Delta a(l)\mu_0(\go)d\go-\int_{\cI}W(a(l),\ \go)\mu_0(\go)d\go\\
&=\  -\bar D_aW(k)(\bar D_aG(k))^{-1}\int_{\cI} G(a(k),\ \go)\mu_0(\go)d\go+\int_{\cI}W(a(k),\ \go)\mu_0(\go)d\go\\
&+\bar D_aW(l)(\bar D_aG(l))^{-1}\int_{\cI} G(a(l),\ \go)\mu_0(\go)d\go-\int_{\cI}W(a(l),\ \go)\mu_0(\go)d\go\,.
\end{aligned}
\end{equation}
This expression has to be non-negative for any $\cI$ of positive Lebesgue measure. Thus,
\begin{equation}
\begin{aligned}
& -\bar D_aW(k)(\bar D_aG(k))^{-1}G(a(k),\ \go)+W(a(k),\ \go)\\
&+\bar D_aW(l)(\bar D_aG(l))^{-1} G(a(l),\ \go)\ -\ W(a(l),\ \go)\ \ge\ 0
\end{aligned}
\end{equation}
for Lebesgue almost any $\go\in\gO_k.$
\end{proof}

\begin{proof}[Proof of Proposition \ref{main convexity-discrete}] First, we note that $y=\binom{y_1}{y_2}\in \hat g(\gO_k\cap X)$ where $y_1\in \R^L$ if and only if 
\[
W(a(k))-x_k^\top (a(k)-y_1)\ =\ \max_{1\le l\le K} (W(a(l))\ -\ x_l^\top (a(l)-y_1))
\]
and $y_1\in \hat g(X).$ Both sets are convex and hence so is their intersection. To show monotonicity of $D_aW(a(\hat g^{-1}(y)))$, pick a $y,z$ such that $y,y+z\in g(\gO_k\cap X).$ By convexity, $y+tz \in \hat g(\gO_k\cap X)$ for all $t\in [0,1].$ Our goal is to show that 
\[
(D_aW(a(g^{-1}(y+z)))-D_aW(a(g^{-1}(y)))z\ \ge\ 0\,. 
\]
Since $a$ is constant inside each $\gO_k,$ it suffices to show this inequality when $y$ and $y+z$ are infinitesimally close to the boundary between two regions, $\gO_{k_1}$ and $\gO_{k_2}.$ Let $y$ belong to that boundary and $y+\eps z\in \gO_{k_2}$. Then, 
\[
W(a(k_2))-D_aW(a(k_2))(a(k_2)-(y+\eps z))\ \ge\ W(a(k_1))-D_aW(a(k_1))(a(k_1)-(y+\eps z))
\]
and 
\[
W(a(k_2))-D_aW(a(k_2))(a(k_2)-y)\ =\ W(a(k_1))-D_aW(a(k_1))(a(k_1)-y)
\]
Subtracting, we get the required monotonicity. 
\end{proof}

\section{Finite Partitions: The Small Uncertainty Limit} 

The structure of the optimal partition (Theorem \ref{regular-partition}) can be complex and non-linear.\footnote{In general, the boundaries of the sets $\gO_k$ might be represented by complicated hyper-surfaces, and some of $\gO_k$ might even feature multiple disconnected components.}  One may ask whether it is possible to ``linearize" these partitions, just as one can linearize equilibria in complex, non-linear economic models, assuming the deviations from the steady state are small. As we show below, this is indeed possible.

Everywhere in this section, we make the following assumption.
\begin{assumption}\label{ass-eps} There exists a small parameter $\eps$ such that the functions defining the equilibrium conditions, $G,$ and the utility  function, $W,$ are given by $G(a,\eps\go)$ and $W(a,\eps\go).$
\end{assumption}

Parameter $\eps$ has two interpretations. First, it could mean small deviations from a steady state (as is common in the literature on log-linear approximations). Second, $\eps$ could be interpreted as capturing the sensitivity of economic quantities to changes in $\go.$ In the limit when $\eps=0,$ equilibrium does not depend on shocks to $\go.$ We use $a^0=a_*(0)$ to denote this ``steady state" equilibrium. By definition, it is given by the unique solution to the system $G(a^0,0)\ =\ 0,$ and the corresponding utility  is $W(a^0,0).$

\begin{assumption}[The information relevance matrix]\label{oonondeg}
We assume the matrix $\cD(0)$ with
\begin{equation}
\begin{aligned}
&\cD(\go)\ \equiv\ D_{\go\go} (W(a_*(\go),\go))\ -\ W_{\go\go}(a_*(\go),\go)
\end{aligned}
\end{equation}
is non-degenerate. We refer to $\cD$ as the information relevance matrix.
\end{assumption}

We are now ready to state the main result of this section, showing how the optimal linearized partition can be characterized explicitly in terms of the information relevance matrix $\cD$.

\begin{theorem}[Linearized partition]\label{mainLimit} Under the hypothesis of Theorem \ref{mainth1} and Assumptions \ref{ass-eps} and \ref{oonondeg}, let $\{\gO_k(\eps)\}_{k=1}^K$ be the corresponding optimal partition. Then, for any sequence $\eps_l\to 0,\ l>0,$ there exists a sub-sequence $\eps_{l_j},\ j>0,$  such that the optimal partition $\{\gO_k(\eps_{l_j})\}_{k=1}^K$ converges to an almost sure partition $\{\tilde \gO_k^*\}_{k=1}^K$ satisfying
\[
\tilde \gO_k^*\ =\ \{\go\in \gO:\ (M_1(k)-M_1(l))^\top \cD(0) \go\ >\ 0.5(M_1(k)^\top \,\cD(0) \, M_1(k)-M_1(l)^\top \,\cD(0) \, M_1(l))\ \forall\ l\not=k\}\,,
\]
where we have defined $M_1(k)\ \equiv\ E[\go|\tilde \gO_k^*]\,.$ In particular, for this limiting partition, each set $\tilde \gO_k^*$ is convex. If the matrix $\cD$ from Assumption \ref{oonondeg} is negative semi-definite, then all sets $\tilde\gO_k^*$ are empty except for one; that is, it is optimal to reveal no information.
\end{theorem}

Theorem \ref{mainLimit} implies that the general problem of optimal information design converges to a quadratic moment persuasion when $\eps$ is small. The matrix $\cD(0)$ of Assumption \ref{oonondeg} incorporates information both about the hessian of $H$ and about other partial derivatives of $G$. Moment persuasion setting corresponds to the case when $D_{\go a}G=0.$ One interesting effect we observe is that, in general, non-zero partial derivatives $D_{\go a}G$ may have a major impact on the structure of the $\cD$ matrix. In particular, when $\eps$ is sufficiently small and $K$ is sufficiently large, we are in a position to link the number of positive eigenvalues of $\cD$ to the dimension of the support of optimal policies.

When the policy-maker sends signal $k$, the receivers learn that $\go\in \gO_k$. As a result, the receivers' posterior estimate of the conditional mean of $\go$ is then given by
\begin{equation}\label{M1}
M_1(\gO_k)\ \equiv\ E[\go|\go\in \gO_k]\ = \frac{\int_{\gO_k} \go\mu_0(\go)d\go}{P_k}\ \in\ \R^m\,,
\end{equation}
where
\[
P_k\ =\ \cP(k)\ =\ \int_{\gO_k}\mu_0(\go)d\go\,.
\]


Define
\begin{equation}
\cG\ \equiv\ (D_aG(a^0,0))^{-1}D_\go G(a^0,0)\ \in \R^{M\times L}\,.
\end{equation}

The following lemma follows by direct calculation.
\begin{lemma}\label{ak1} For any sequence $\eps_\nu \to 0,\ \nu\in \Z_+,$ there exists a sub-sequence $\eps_{\nu_j},\ j>0,$ such that the optimal partitions $\{\gO_k(\eps_{\nu_j})\}_{k=1}^K$ converge to a limiting partition $\{\gO_k(0)\}_{k=1}^K\,$ as $j\to\infty.$  In this limit,
\begin{equation}
a_k(\eps_{\nu_j})\ =\ a^0_k\ -\ \eps_{\nu_j} \cG\,M_1(\gO_k(0))\ +\ o(\eps_{\nu_j})\,.
\end{equation}
\end{lemma}
Lemma \ref{ak1} provides an intuitive explanation for the role of the matrix $\cG.$ Namely, in the linear approximation, the receivers' action is given by a linear transformation of $E[\go|k],$ the first moment of $\go$ given the signal: $a_k\ \approx\ a^0_k\ -\ \cG E[\go|k].$ Thus, the matrix $-\cG$ captures how strongly receivers' actions respond to changes in beliefs.

\begin{proof}[Proof of Lemma \ref{ak1}]  Trivially, the set of partitions is compact and hence we can find a subsequence $\{\gO_k(\eps_j)\}$ converging to some partition $\{\gO_k(0)\}$ in the sense that their indicator functions converge in $L_2.$ We have
\begin{equation}
\begin{aligned}
&0\ =\ \int_{\gO_k(\eps)} G(a(k,\eps),\eps\go)\mu_0(\go)d\go\ =\ \int_{\tilde\gO_k(\eps)} G(a(k,\eps),\go)\mu_0(\go)d\go
\end{aligned}
\end{equation}
Now,
\begin{equation}\label{id-b}
\begin{aligned}
&0\ =\ \int_{\gO_k(\eps)} G(a(k,\eps),\eps\go)\mu_0(\go)d\go\\
&=\ G(a(k,\eps),0)M(\gO_k(\eps))\ +\ \eps\,D_\go G(a(k,\eps),0)\,M_1(\gO_k(\eps))\ +\ O(\eps^2)\,.
\end{aligned}
\end{equation}
Let us show that $a(k,\eps)-a(k,0)\ =\ O(\eps).$ Suppose the contrary. Then there exists a sequence $\eps_m\to0$ such that $\|a(k,\eps)-a(k,0)\|\eps^{-1}\ \to\infty.$ We have
\begin{equation}\label{id-b1}
\begin{aligned}
&G(a(k,\eps),0)\ -\ G(a(k,0),0)\ =\ \int_0^1D_aG(a(k,0)+t(a(k,\eps)-a(k,0)))(a(k,\eps)-a(k,0))dt\\
&\ge\ c \|a(k,\eps)-a(k,0)\|
\end{aligned}
\end{equation}
for some $c>0$ due to the continuity and non-degeneracy of $D_aG(0)=D_aG(a(k,0)).$ Dividing \eqref{id-b} by $\eps,$ we get a contradiction.

Define
\[
 a^{(1)}(k)\ \equiv\ -D_aG(0)^{-1}D_\go G(a(k),0)M_1(\gO_k(0))\ =\ -\cG\,M_1(\gO_k(0))\,.
\]
Let us now show that $a(k,\eps)-a(k,0)\ =\ \eps a^{(1)}(k)\ +\ o(\eps).$ Suppose the contrary. Then, $\|\eps^{-1}(a(k,\eps)-a(k,0))-a^{(1)}(k)\|\ >\ c$ for some $c>0$ along a sequence of $\eps\to0.$ By \eqref{id-b1},
\begin{equation}\label{id-b2}
\begin{aligned}
&0\ =\ \int_{\gO_k(\eps)} G(a(k,\eps),\eps\go)\mu_0(\go)d\go\\
&=\ G(a(k,\eps),0)M(\gO_k(\eps))\ +\ \eps\,D_\go G(a(k,\eps),0)\,M_1(\gO_k(\eps))\ +\ O(\eps^2)\\
&=\ \eps D_aG(0) \eps^{-1}(a(k,\eps)-a(k,0))M(\gO_k(\eps))\ +\ \eps\,D_\go G(a(k),0)\,M_1(\gO_k(\eps))\ +\ O(\eps^2)\,,
\end{aligned}
\end{equation}
and we get a contradiction taking the limit as $\eps\to0.$
\end{proof}

\begin{proof}[Proof of Theorem \ref{mainLimit}] We have 
\begin{equation}\label{set11}
\begin{aligned}
&\gO_k(\eps)\ =\  \{\go\in\gO:\  -\bar D_aW(k,\eps)(\bar D_aG(k,\eps))^{-1}G(a(k,\eps),\ \eps\go)+W(a(k,\eps),\ \eps\go)\\ >\ 
&-\bar D_aW(l,\eps)(\bar D_aG(l,\eps))^{-1} G(a(l,\eps),\ \eps\go)\ +\ W(a(l,\eps),\ \eps\go)\ \forall\ l\not=k.\}
\end{aligned}
\end{equation}
The proof of the theorem is based on the following technical lemma. 

\begin{lemma}\label{pert} We have 
\begin{equation}
\begin{aligned}
& -\bar D_aW(k,\eps)(\bar D_aG(k,\eps))^{-1}G(a(k,\eps),\ \eps\go)+W(a(k,\eps),\ \eps\go)\\
&=\ W(0)\ -\ 0.5M_1(k)^\top \cD \eps^2 \go\ +\ 0.5\eps^2 M_1(k)^\top \,\cD \, M_1(k)\ +\ \eps W_\go(0)\go\\
&+0.5\eps^2\go^\top W_{\go\go}(0)\go\ -D_aW(0) D_aG(0)^{-1}(D_\go G(0)\go +0.5\go^\top G_{\go\go}(0)\go) +\ o(\eps^2)\,. 
 \end{aligned}
 \end{equation}
\end{lemma}
\begin{proof}
We have
\begin{equation}
\begin{aligned}
&\bar D_aW(k,\eps)\ =\ \int_{\gO_k(\eps)}D_aW(a(k,\eps),\eps\go)\mu_0(\go)d\go\\
& =\ \int_{\gO_k(\eps)}(D_aW(0)+\eps \go^\top D_\go W(0)^\top+\eps a^{(1)}(k)^\top D_{aa}^2W(0)+o(\eps))\mu_0(\go)d\go\\
&\ =\ (D_aW(0)+\eps (a^{(1)}(k))^\top D_{aa}^2W(0) \ +\ \eps\, M_1(\gO_k(0))^\top (D_\go W(0))^\top )M(\gO_k(\eps))\,\ +\ o(\eps)\ \in \R^{1\times M}\,.
\end{aligned}
\end{equation}
At the same time, an analogous calculation implies that
\begin{equation}
\begin{aligned}
&\bar D_aG(k,\eps)\ =\ (D_aG(0)+\eps (a^{(1)}(k))^\top D_{aa}^2G(0) \ +\ \eps\, D_\go G(0)M_1(\gO_k(0)))M(\gO_k(\eps))\,\ +\ o(\eps)
\end{aligned}
\end{equation}
Here, $D_aG(0)=(\partial G_i/\partial a_j)$ and
\[
(D_\go G(0)M_1(\gO_k(0)))_{i,j}\ =\ \sum_k\frac{\partial^2 G_i}{\partial a_j\partial\go_k}M_{1,k}\,,
\]
and, similarly,
\[
((a^{(1)}(k))^\top D_{aa}^2G(0))_{i,j}\ =\ \sum_l (a^{(1)}(k))_l \frac{\partial^2 G_l}{\partial a_i\partial a_j}\ \in \R^{M\times M}\,.
\]
Thus, 
\begin{equation}
\begin{aligned}
&M(\gO_k(\eps))\bar D_aG(k,\eps)^{-1}\\
&=\ D_aG(0)^{-1}-D_aG(0)^{-1}\eps \Big((a^{(1)}(k))^\top D_{aa}^2G(0) \ +\ \eps\,  D_\go G(0)M_1(\gO_k(0))\Big)D_aG(0)^{-1}\ +\ o(\eps)\,,
\end{aligned}
\end{equation}
and therefore
\begin{equation}
\begin{aligned}
&\bar D_aW(k,\eps)(\bar D_aG(k,\eps))^{-1}\ =\ D_aW(0) D_aG(0)^{-1}\\
&+\ \eps (M_1^\top D_\go W(0)^\top D_aG(0)^{-1}\,+\, (a^{(1)}(k))^\top D_{aa}^2W(0)D_aG(0)^{-1})\\
&-\eps D_aW(0)D_aG(0)^{-1} \Big((a^{(1)}(k))^\top D_{aa}^2G(0) \ +\  D_\go G(0)M_1(\gO_k(0))\Big)D_aG(0)^{-1} +\ o(\eps)\\
&=\ D_aW(0) D_aG(0)^{-1}\\
&+\ \eps (M_1^\top D_\go W(0)^\top D_aG(0)^{-1}\,- M_1^\top\cG^\top D_{aa}^2W(0)D_aG(0)^{-1})\\
&-\eps D_aW(0)D_aG(0)^{-1} \Big(- M_1^\top\cG^\top D_{aa}^2G(0) \ +\ D_\go G(0)M_1\Big)D_aG(0)^{-1} +\ o(\eps)\\
&=\ D_aW(0) D_aG(0)^{-1}+\eps\Gamma+o(\eps)\,,
\end{aligned}
\end{equation}
where
\begin{equation}
\begin{aligned}
&\Gamma\ =\ M_1^\top D_\go W(0)^\top D_aG(0)^{-1}\\
&- M_1^\top\cG^\top D_{aa}^2W(0)D_aG(0)^{-1}-D_aW(0)D_aG(0)^{-1} \Big(- M_1^\top\cG^\top D_{aa}^2G(0) \ +\ D_\go G(0)M_1\Big)D_aG(0)^{-1}\,.
\end{aligned}
\end{equation}
Define
\[
 \tilde a^{(1)}(k,\eps)\ \equiv\ \eps^{-1}(a(k,\eps)-a(k,0))\ =\ a^{(1)}(k)\ +\ o(1)\,.
\]
Let also
\begin{equation}
\begin{aligned}
&G^{(2)}(k)\ \equiv\ 0.5\eps^2(a^{(1)}(k)^\top D_{aa}^2G(0)a^{(1)}(k) +\  2\go^\top D_\go G(0) a^{(1)}(k)\ +\ \go^\top G_{\go\go}(0)\go)\,,
\end{aligned}
\end{equation}
so that
\[
G(a(k,\eps),\ \eps\go)\ -\ (\eps D_aG(0) \tilde a^{(1)}(k,\eps)+\eps D_\go G(0)\go)\ =\ \eps^2G^{(2)}(k)\ +\ o(\eps^2)\,,
\]
where we have used that $G(0)=0.$ While we cannot prove that $\eps (\tilde a^{(1)}(k)-a^{(1)}(k))=o(\eps^2),$ we show that this term cancels out. We have
\begin{equation}
\begin{aligned}
&-\bar D_aW(k,\eps)(\bar D_aG(k,\eps))^{-1}G(a(k,\eps),\ \eps\go)+W(a(k,\eps),\ \eps\go)\\
&\approx\ -\bar D_aW(k,\eps)(\bar D_aG(k,\eps))^{-1}\Big(\eps D_aG(0) \tilde a^{(1)}(k,\eps)+\eps D_\go G(0)\go +\eps^2G^{(2)}(k)\ +\  o(\eps^2)\Big)\\
&+\Bigg(W(0)+\eps D_aW(0) \tilde a^{(1)}(k,\eps)+\eps W_\go(0)\go \\
&+0.5\eps^2
\Big(
(a^{(1)}(k))^\top  D_{aa}^2W(0) a^{(1)}(k)+ \go^\top  W_{\go,\go}(0)\go+2 (a^{(1)}(k))^\top D_\go W(0)  \go
\Big)\ +\ o(\eps^2)
\Bigg)\\
&=\ -
\Big(
D_aW(0) D_aG(0)^{-1}\ +\ \eps \Gamma+o(\eps)\Big)\\
&\times \Big(\eps D_aG(0) \tilde a^{(1)}(k,\eps)+\eps D_\go G(0)\go +\eps^2G^{(2)}(k)\ +\  o(\eps^2)\Big)\\
&+\Bigg(W(0)+\eps D_aW(0) \tilde a^{(1)}(k)+\eps W_\go(0)\go \\
&+0.5\eps^2
\Big(
(a^{(1)}(k))^\top  D_{aa}^2W(0) a^{(1)}(k)+ \go^\top  W_{\go,\go}(0)\go+2( a^{(1)}(k))^\top  D_\go W(0) \go
\Big)+o(\eps^2)
\Bigg)\\
\end{aligned}
\end{equation}
\begin{equation}
\begin{aligned}
&=\ W(0)\\
&+\ \eps
\Bigg(
-D_aW(0) D_aG(0)^{-1}
\Big(D_aG(0) \tilde a^{(1)}(k,\eps)+D_\go G(0)\go\Big)\ +
D_aW(0) \tilde a^{(1)}(k,\eps)+ W_\go(0)\go
\Bigg)\\
&+\eps^2
\Bigg(-D_aW(0) D_aG(0)^{-1}G^{(2)}(k)
- \Gamma  \Big(D_aG(0) a^{(1)}(k)+D_\go G(0)\go\Big)\\
&+0.5
\Big(
(a^{(1)}(k))^\top  D_{aa}^2W(0) a^{(1)}(k)+ \go^\top  W_{\go,\go}(0)\go+2  a^{(1)}(k)^\top  D_\go W(0)\go
\Big)
\Bigg)\ +\ o(\eps^2)\,.
\end{aligned}
\end{equation}
Thus, the terms with $\tilde a^{(1)}(k,\eps)$ have cancelled out. 
We have
\begin{equation}
\begin{aligned}
&\Gamma  \Big(D_aG(0) a^{(1)}(k)+D_\go G(0)\go\Big)\ =\ \Big(
M_1^\top D_\go W(0)^\top D_aG(0)^{-1}\\
&- M_1^\top\cG^\top D_{aa}^2W(0)D_aG(0)^{-1}-D_aW(0)D_aG(0)^{-1} \Big(- M_1^\top\cG^\top D_{aa}^2G(0) \ +\   D_\go G(0)M_1\Big)D_aG(0)^{-1}
\Big)\\
&\times  D_\go G(0)(\go-M_1)\\
&=\  \Big(
M_1^\top D_\go W(0)^\top - M_1^\top\cG^\top D_{aa}^2W(0)-D_aW(0)D_aG(0)^{-1} \Big(- M_1^\top\cG^\top D_{aa}^2G(0) \ +\  D_\go G(0)M_1\Big)\Big)\\
&\times  \cG(\go-M_1)\ =\ M_1^\top \cD_1\cG(\go-M_1)\,,
\end{aligned}
\end{equation}
where
\[
\cD_1\ =\ D_\go W(0)^\top-D_aW(0)D_aG(0)^{-1} D_\go G(0)-(\cG^\top D_{aa}^2W(0)-\cG^\top D_aW(0)D_aG(0)^{-1} D_{aa}^2G(0))\in \R^{L\times M}
\]
and where the three-dimensional tensor multiplication is understood as follows:
\begin{equation}
\begin{aligned}
&M_1^\top D_aW(0)D_aG(0)^{-1} D_\go G(0)\ =\ \sum_k M_{1,k}D_aW(0)D_aG(0)^{-1} D_aG_{\go_k}(0)\\
&M_1^\top \cG^\top  D_aW(0)D_aG(0)^{-1} D_\go G(0)\ =\ \sum_k (\cG M_{1})_k D_aW(0)D_aG(0)^{-1} D_aG_{a_k}(0)\,.
\end{aligned}
\end{equation}
Rewriting, we get
\begin{equation}
\begin{aligned}
&W(0)\ +\ \eps
\Bigg(
-D_aW(0) D_aG(0)^{-1}D_\go G(0)\go+ W_\go(0)\go
\Bigg)\\
&+\eps^2
\Bigg(-D_aW(0) D_aG(0)^{-1}\eps^2G^{(2)}(k)-M_1^\top \cD_1\cG(\go-M_1)
\\
&+0.5
\Big(
(a^{(1)}(k))^\top  D_{aa}^2W(0) a^{(1)}(k)+ \go^\top  W_{\go,\go}(0)\go+2  a^{(1)}(k)^\top  D_\go W(0)\go
\Big)
\Bigg)\ +\ o(\eps^2)\\
&=\ W(0)\ +\ \eps
\Bigg(
-D_aW(0) \cG\go+ W_\go(0)\go
\Bigg)\\
&+\eps^2
\Bigg(-D_aW(0) D_aG(0)^{-1}\eps^2G^{(2)}(k)-M_1^\top \cD_1\cG(\go-M_1)\\
&+0.5
\Big(
M_1^\top \cG^\top D_{aa}^2W(0)\cG M_1+ \go^\top  W_{\go,\go}(0)\go-2  (\cG M_1)^\top  D_\go W(0)\go
\Big)
\Bigg)\ +\ o(\eps^2)\,.
\end{aligned}
\end{equation}
Now,
\[
\eps^2G^{(2)}(k)\ =\ 0.5(M_1^\top \cG^\top D_{aa}^2G(0)\cG M_1\ -\  2(\cG M_1)^\top D_\go G(0)\go +\go^\top G_{\go\go}(0)\go )\,.
\]
Thus, the desired expression is given by
\begin{equation}
\begin{aligned}
&W(0)\ +\ \eps
\Bigg(
-D_aW(0) \cG\go+ W_\go(0)\go
\Bigg)\ +\ \eps^2(0.5 M_1^\top\cA M_1+M_1^\top\cB \go\ +\ \go^\top \mathcal C \go)
\end{aligned}
\end{equation}
where we have defined
\begin{equation}
\begin{aligned}
&\cA\ \equiv\ - D_aW(0) D_aG(0)^{-1}\cG^\top D_{aa}^2G(0)\cG+2\cD_1\cG+\cG^\top D_{aa}^2W(0)\cG \\
&=\  - D_aW(0) D_aG(0)^{-1}\cG^\top D_{aa}^2G(0)\cG+2\Big(D_\go W(0)^\top-D_aW(0)D_aG(0)^{-1} D_\go G(0)\\
&-(\cG^\top D_{aa}^2W(0)-D_aW(0)D_aG(0)^{-1}\cG^\top D_{aa}^2G(0))\Big)\cG+\cG^\top D_{aa}^2W(0)\cG\\
&=\ \cG^\top D_aW(0) D_aG(0)^{-1} D_{aa}^2G(0)\cG-\cG^\top D_{aa}^2W(0)\cG\\
&+2(D_\go W(0)^\top\cG -\cG^\top D_aW(0)D_aG(0)^{-1} D_\go G(0))\in \R^{L\times L}\\
&\cB\ \equiv\  \cG^\top D_aW(0) D_aG(0)^{-1} D_\go G(0)-\cD_1\cG-\cG^\top D_\go W(0)\\
&=\  \cG^\top D_aW(0) D_aG(0)^{-1} D_aG^\top_\go(0)-\Big(\cG^\top D_\go W(0)-\cG^\top D_aW(0)D_aG(0)^{-1} D_\go G(0)\\
&-(\cG^\top D_{aa}^2W(0)\cG -D_aW(0)D_aG(0)^{-1}\cG^\top D_{aa}^2G(0))\cG\Big)-D_\go W(0)^\top \cG
\end{aligned}
\end{equation}
Here, the first term is given by
\[
(D_aW(0) D_aG(0)^{-1} D_aG^\top_\go(0))_{i,j}\ =\ \sum_k (D_aW(0) D_aG(0)^{-1} )_k \frac{\partial^2 G_k}{\partial a_i\partial \go_j}\,,
\]
and the claim follows by a direct (but tedious) calculation. 
\end{proof}

The desired convergence is then a direct consequence of Lemma \ref{pert}. Indeed, by compactness, we can pick a converging sub-sequence and Lemma \ref{pert} implies that, in the limit, a point $\go$ satisfies inequalities \eqref{set11} if and only if $\go\in \tilde \gO_k^*$. 
\end{proof}

\section{Proof of Theorems \ref{mainth-limit} and \ref{monge}}

In this section, we prove that a pure optimal policy (see Definition  \ref{dfn1}) solving the the unconstrained persuasion problem always exists. We do this by passing to the limit in Theorem \ref{mainth1}. The proof of convergence is non-trivial due to additional complications created by the potential non-compactness of the set $\gO.$\footnote{Note that all existing models of Bayesian persuasion (with the exception of \cite{tamura2018Bayesian}) assume that $\gO$ is compact. This precludes many practical applications where the distributions (such as, e.g., the Gaussian distribution) do not have compact support.}

\begin{lemma}\label{lem-approx} When $K\to\infty,$ maximal sender's utility attained with $K$-finite optimal information designs converges to the maximal utility  attained in the full, unconstrained problem of Definition \ref{dfn1}. 
\end{lemma}

\begin{proof}[Proof of Lemma \ref{lem-approx}] The proof requires some additional arguments because $\gO$ is not necessarily compact. 
First, consider an increasing sequence of compact sets $X_n=\{\go:g(\go)\le n\}$ such that $X_n$ converge to $\gO$ as $n\to\infty.$  For any measure $\mu,$ let $\mu_X$ be its restriction on $X.$ Let $a_n=a(\mu_{X_n}).$ The first observation is that Assumptions \ref{integrability} and \ref{ac} imply that $a_n\to a$ uniformly as $n\to \infty$. Indeed, 
\[
\int_{X_n} G(a_n,\go)d\mu(\go)\ =\ \int_{\gO} G(a,\go)d\mu(\go)=0
\]
implies that 
\begin{equation}
\begin{aligned}
&\int_{X_n} (G(a_n,\go)-G(a,\go))d\mu(\go)\ =\ \int_{\gO\setminus X_n} G(a,\go)d\mu(\go)\\
&\le\ \int_{\gO\setminus X_n} \eps^{-1}\|a-a_*(\go)\|d\mu(\go)\ \le\ 2\eps^{-1}\mu(\gO\setminus X_n)^{1/2}\left(\int_{\gO\setminus X_n} \|a_*(\go)\|^2d\mu(\go)\right)^{1/2}\\
&\ \le\ \eps^{-1}(\mu(\gO\setminus X_n)+\int_{\gO\setminus X_n} \|a_*(\go)\|^2d\mu(\go))\,.
\end{aligned}
\end{equation}
Multiplying by $(a-a_n)$, we get 
\[
\eps\,\|a-a_n\|^2 (1-\mu(\gO\setminus X_n))\ \le\  \|a-a_n\| \eps^{-1}(\mu(\gO\setminus X_n)+\int_{\gO\setminus X_n} \|a_*(\go)\|^2d\mu(\go))
\]
Furthermore, by Lemma \ref{existence}, $\|a-a_n\|\le\ 2\left(\int_{\gO} g(\go)d\mu(\go)\right)^{1/2}\ \le\ 1+\int_{\gO} g(\go)d\mu(\go)$ and therefore 
\[
\|a-a_n\|\ \le\ C\Bigg( \mu(\gO\setminus X_n)(1+\int_{\gO} g(\go)d\mu(\go))+\int_{\gO\setminus X_n} g(\go)d\mu(\go)
\Bigg)
\] 
for some constant $C>0.$ Now, pick a $\tau\in \Delta(\Delta(\gO)).$ Since the function $q(x)={\bf 1}_{x>n}$ is monotone increasing in $x$, we get 
\[
\mu(\gO\setminus X_n)\int_{\gO} g(\go)d\mu(\go)\ =\ \int_{\gO} q(g(\go))d\mu(\go) \int_{\gO} g(\go)d\mu(\go)\ \le\ \int_{\gO} q(g(\go))g(\go)d\mu(\go)\ =\ \int_{\gO\setminus X_n} g(\go)d\mu(\go)\,
\]
and therefore 
\[
\|a-a_n\|\ \le\ C\int_{\gO\setminus X_n}(1+2g(\go))d\mu(\go)\,.
\]
Then, we have by the Jensen inequality that 
\begin{equation}
\begin{aligned}
&|\bar W(\mu)-\bar W(\mu_{X_n})|\ \le\ \int_{\gO\setminus X_n} |W(a(\mu),\go)|d\mu(\go)\ +\ \int_{X_n}|W(a(\mu),\go)-W(a_n(\mu),\go)|d\mu(\go)\\
&\ \le\ \int_{\gO\setminus X_n}(g(\go) f(\int_{\gO} g(\go)d\mu(\go)))d\mu(\go)\\
&+\ \int_{\gO}\|a(\mu)-a_n(\mu)\| (g(\go) f(\int_{\gO} g(\go)d\mu(\go)))d\mu(\go)\\
&\ \le\ \int_{\gO\setminus X_n}g(\go)d\mu(\go)\,\int_\gO f(g(\go))d\mu(\go)\\
&+\ \|a(\mu)-a_n(\mu)\|  \int_{\gO} g(\go)d\mu(\go)\,\int_\gO f(g(\go))d\mu(\go)\,.
\end{aligned}
\end{equation}
Since the function $q(x)=x{\bf 1}_{x>n}$ is monotone increasing in $x$ and $f$ is monotone increasing, we get 
\[
 \int_{\gO} g(\go)d\mu(\go)\,\int_\gO f(g(\go))d\mu(\go)\ \le\ \int_\gO g(\go) f(g(\go))d\mu(\go)
\]
and therefore, by the same monotonicity argument, 
\begin{equation}
\begin{aligned}
&\|a(\mu)-a_n(\mu)\|  \int_{\gO} g(\go)d\mu(\go)\,\int_\gO f(g(\go))d\mu(\go)\ \le\ C\int_{\gO\setminus X_n}(1+2g(\go))d\mu(\go)\,\int_\gO g(\go) f(g(\go))d\mu(\go)\\
&\le\ C\int_{\gO\setminus X_n}(1+2g(\go)) g(\go) f(g(\go))d\mu(\go)\,.
\end{aligned}
\end{equation}
Similarly, 
\begin{equation}
\begin{aligned}
&\int_{\gO\setminus X_n}g(\go)d\mu(\go)\,\int_\gO f(g(\go))d\mu(\go)\ =\ \int_{\gO}q(g(\go))d\mu(\go)\,\int_\gO f(g(\go))d\mu(\go)\\ 
&\le\ \int_\gO q(g(\go))f(g(\go))d\mu(\go)\ =\ \int_{\gO\setminus X_n}g(\go)f(g(\go))d\mu(\go)\,.
\end{aligned}
\end{equation}
Therefore, by the Fubini Theorem, 
\begin{equation}
\begin{aligned}
&|\int_{\Delta(\mu)} (\bar W(\mu)-\bar W(\mu_{X_n}))d\tau(\mu)|\\
& \le\ \int_{\gD(\gO)}\int_{\gO\setminus X_n}g(\go)f(g(\go))d\mu(\go)d\tau(\mu)\\
&+\  \int_{\gD(\gO)}C\int_{\gO\setminus X_n}(1+2g(\go)) g(\go) f(g(\go))d\mu(\go)d\tau(\mu)\\
&=\ \int_{\gO\setminus X_n}(g(\go)f(g(\go))+C(1+2g(\go)) g(\go) f(g(\go)))d\mu_0(\go)\,.
\end{aligned}
\end{equation}
Thus, Assumption \ref{integrability} implies that we can restrict our attention to the case when $\gO=X_n$ is compact. 

In this case, the Prokhorov Theorem implies that $\Delta(\gO)$ is compact in the weak* topology and this topology is metrizable. Thus, for any $\eps>0,$ we can decompose $\Delta(\gO)=Q_1\cup\cdots\cup Q_K$ where all $Q_k$ have diameters less than $\eps.$ We can now approximate $\tau$ by $\tilde\tau=\sum_k \nu_k \gd_{\mu_k}$ with $\mu_k=\int_{Q_k} \mu d\tau(\mu)/\nu_k$ and $\nu_k= \int_{Q_k} d\tau(\mu).$ Clearly, $\int \mu d\tilde\tau(\mu)=\mu_0,$ and therefore it remains to show that $\bar W$ is continuous in the weak* topology. 

To this end, suppose that $\mu_n\to\mu$ in the weak* topology. Let us first show $a_n=a(\mu_n)\to a(\mu)=a.$ Suppose the contrary.  Since $\gO$ is compact and $G$ is continuous and bounded, Lemma \ref{existence} implies that $a_n$ are uniformly bounded. Pick a subsequence such that $\|a_n-a\|>\eps$ for some $\eps>0$ and subsequence $a_n\to b$ for some $b\not=a.$ Since $G(a_n,\go)\to G(b,\go)$ uniformly on $\gO,$ we get a contradiction because 
\[
\int G(a_n,\go)d\mu_n-\int G(b,\go)d\mu\ =\ \int (G(a_n,\go)-G(b,\go))d\mu_n\ +\ \int G(b,\go)d(\mu_n-\mu)\,.
\]
The second term converges to zero because of weak* convergence. The first term can be bounded by 
\[
|\int (G(a_n,\go)-G(b,\go))d\mu_n|\ \le\ C \|a_n-b\|
\]
and hence also converges to zero. Thus, $\int G(b,\go)d\mu=\int G(a,\go)d\mu=0$, implying that $a=b$ by the strict monotonicity of the map $G.$ The same argument implies the required continuity of $\bar W(\mu).$ 
\end{proof}

\begin{proof}[Proof of Theorem \ref{mainth-limit}] For each finite $K,$ the optimal solution $(a_K(\go),x_K(a(\go)))$ stays uniformly bounded and hence there exists a subsequence converging in $L_2(\gO;\mu_0)$ and in probability to a limit $(a(\go),x(a(\go))).$ By continuity and Lemma \ref{lem-approx}, $E[W(a_K(\go),\go)]$ converges to the maximum in the unconstrained problem and, hence, by the same continuity argument, $a(\go)$ is an  optimal policy without randomization. Since \eqref{olicy} holds true for $a_K$ for each finite $K$, convergence in probability implies that \eqref{x-top} also holds in the limit. Indeed, 
\[
c(a_K(q),\go;x_K)\ \ge\ c(a_K(\go),\go;x_K)
\]
holds for almost all $\go$ and all $q$ with probability one, and hence it also holds in the limit with probability one (due to convergence in probability). Clearly, for each finite $K$ the function $c(a_K(\go),\go;x_K)$ is smooth in each region $\gO_k$ and is continuous at the boundaries. Since $a,x$ stay bounded and $W,G$ are smooth and $G$ is compact, the functions are uniformly Lipschitz continuous and the Arzela-Ascoli theorem implies that so is the limit (passing to a subsequence if necessary). Finally, to prove that 
\[
E[G(a(\go),\go)|a(\go)]\ =\ 0
\]
it suffices to prove that 
\[
E[G(a(\go),\go)f(a(\go))]\ =\ 0
\]
for a countable dense set of test functions, which follows by passing to a subsequence. 

To verify all the required integrability to apply Lebesgue dominated convergence, Assumption \ref{integrability} implies that we just need to check that $E[D_aG(a,\go)|a]^{-1}$ is uniformly bounded. Since, by assumption, $\|D_aG(a,\go)\|$ is uniformly bounded, we just need to check that eigenvalues of $E[D_aG(a,\go)|a(\go)]$ are uniformly bounded away from zero. 

Indeed, let $\eps=\inf_{a,\go,z,\|z\|=1}-z\top D_aG(a,\go)z >0.$ If $\gl$ is an eigenvalue of $E[-D_aG(a,\go)|a]$ with a normed eigenvector $z$, then 
\[
\gl\ =\ z^\top E[-D_aG(a,\go)|a]z\ =\ E[-z^\top D_aG(a,\go)z|a]\ \ge\eps\,.
\]
To prove that \eqref{x-top} holds, we note that it suffices to show that 
\[
E[(x(a(\go))^\top E[D_aG(a,\go)|a(\go)]-E[D_aW(a,\go)|a(\go)])f(a(\go))]\ =\ 0
\]
for a countable, dense set of smooth test functions $f.$ The latter is equivalent to 
\[
E[(x(a(\go))^\top D_aG(a(\go),\go)-D_aW(a(\go),\go))f(a(\go))]\ =\ 0
\]
and the claim follows by continuity by passing to a subsequence. Finally, the last identity follows from \eqref{key}. Finally, the fact that $a$ is Borel-measurable follows from the known fact that for any Lebesgue-measurable $a(\go)$ there exists a Borel measurable modification of $a$ coinciding with $a$ for Lebesgue-almost every $\go.$ 
\end{proof}

\begin{proof}[Proof of Theorem \ref{monge}] Let $(a(\go),\ x(a))$ be an optimal policy and let 
\[
\phi^c(a)\ =\ \inf_\go (c(a,\go;x)-\phi_\Xi(\go;x))
\]
Pick an $a\in \Xi.$ Since $a\in \Xi,$ there exists a $\tilde\go$ such that $a=a(\tilde\go)$ and hence 
\[
\phi^c(a)\ =\ \inf_\go (c(a,\go;x)-\phi_\Xi(\go;x))\ \le\ c(a,\tilde\go;x)-\phi_\Xi(\go;x)\ =\ 0\,.
\]
Thus, 
\[
\int\phi^c(a(\go))\mu_0(\go)d\go\ =\ 0\,. 
\]
At the same time, 
\[
c(a,\go;x)-\phi_\Xi(\go;x)\ =\ c(a,\go;x)\ -\ \inf_{b\in \Xi}c(b,\go;x)\ \ge 0\,.
\]
Thus, $\phi^c(a)=0$ for all $a\in \Xi.$ Now, by the definition of $\phi^c,$ we always have 
\[
\phi^c(a)\ +\ \phi_\Xi(\go;x)\ \le\ c(a,\go)\
\]
for an optimal policy. Let $\gamma$ be the measure on $\Xi\times \gO$ describing the joint distribution of $\chi=a(\go)$ and $\go.$ Then, 
\[
\int c(a,\go)\gamma(a,\go)\ =\ \int c(a(\go),\go)\mu_0(\go)d\go\ =\ \int \phi_\Xi(\go;x)\mu_0(\go)d\go\ =\ \int \phi_\Xi(a)d\nu(a)\,,
\]
Pick any measure $\pi$ from the Kantorovich problem. Then, 
\begin{equation}
\begin{aligned}
&\int c(a,\go)d\gamma(a,\go)\ =\ \int(\phi_\Xi(\go;x))d\gamma(a,\go)\\
& =\  \int\phi_\Xi(\go;x)\mu_0(\go)d\go\ +\ \int\phi^c(a(\go))\mu_0(\go)d\go\\
&=\ \int(\phi_\Xi(\go;x)+\phi^c(a(\go)))\mu_0(\go)d\go\\
&=\ \int(\phi_\Xi(\go;x)+\phi^c(a))d\pi(a,\go)\ \le\ \int c(a,\go)d\pi(a,\go)
\end{aligned}
\end{equation}
Thus, $\gamma$ minimizes the cost in the Kantorovich problem. 
\end{proof}

We complete this section with a direct application of a theorem of \cite{gangbo1995habilitation} and \cite{levin1999abstract}. Recall that a function $u(\go)$ is called $c$-convex if and only if $u=(u^{\tilde c})^c$ where 
\[
u^{\tilde c}(a)\ =\ \sup_{\go\in \gO}(-c(a,\go)-u(\go)),\ v^c(\go)\ =\ \sup_{a\in \R^M}(-c(a,\go)-v(a))
\]
The class of $c$-convex functions has a lot of nice properties (see, e.g., \cite{mccann2011five}). In particular, a $c$-convex function $u$ is twice differentiable Lebesgue-almost everywhere and satisfies $|Du(\go)|\ \le\ \sup_a |D_\go c(a,\go)|$ and $D_{aa}^2u(\go)\ \ge\ \inf_a-D_{\go\go} c(,\go).$ The following is true (see, e.g., \cite{mccann2011five}).

\begin{corollary}\label{gradentcor} Suppose that $c$ is jointly continous in $(a,\go)$ and that the map $a\ \to\ D_\go c(a,\go)$ is injective for Lebesgue-almost every $\go\in\gO.$ Let $a=Y(\go, p)$ be the unique solution to $D_\go c(a,\go)=-p$ for any $p.$ Then, there exists a locally Lipschitz $c$-convex function $u:\gO\to \R$ such that $a(\go)\ =\ Y(\go, Du(\go)).$ 
\end{corollary}

\section{Proofs for the Case of Moment Persuasion: Basic Propeties and Convexity of Pools}

\begin{proof}[Proof of Proposition \ref{main convexity}] By Proposition \ref{main convexity-discrete}, the map $x\to D_aW(a(x))$ is monotone increasing in $x$ for the discrete approximations. Since this property is preserved in the continuous limit, the limiting $a$ also has $x\to D_aW(a(x))$ monotone increasing. The claims now follow because level sets of a monotone increasing map are  convex.
\end{proof}

\begin{proof}[Proof of Theorem \ref{cor-moment}] Integrability condition (by the same argument as in the proof of Lemma \ref{lem-approx}) implies that all the convergence arguments are justified. The convexity claim is then a direct consequence of Proposition \ref{main convexity}. 

\end{proof}

\begin{proof}[Proof of Theorem \ref{converse}] The first claim follows from Corollary \ref{main-conditions} in the Appendix. The proof of sufficiency closely follows ideas from \cite{kramkov2019optimal}. 

Let $a(\go)$ be a policy satisfying the conditions Theorem \ref{converse}. Note that, in terms of the function $c,$ our objective is to show (see \eqref{key}) that 
\[
\min_{all\ feasible\ policies\ b(\go)}E[c(b(\go),g(\go))]\ =\ E[c(a(\go),g(\go))]. 
\]
Next, we note that the assumed maximality implies that $c(a(\go),g(\go))\ =\ \phi_\Xi(g(\go))\le 0$ for all $\go.$ Now, for any feasible policy $b(\go),$ we have $E[g(\go)|b(\go)]=b(\go)\in conv(g(\gO))$ and therefore, for any fixed $a\in \R^M,$ we have 
\begin{equation}
\begin{aligned}
&E[c(a,g(\go))\ -\ c(b(\go),g(\go))|b(\go)]\\
& =\ E[W(g(\go))\ -\ W(a)\ +\ D_aW(a)\,(a-g(\go))\\
&-(W(g(\go))\ -\ W(b(\go))\ +\ D_aW(b(\go))\,(b(\go)-g(\go)))]\\
&\ =\ W(b(\go))-W(a)+D_aW(a)(a-b(\go))\ =\ c(a,b(\go))\,.
\end{aligned}
\end{equation}
Taking the infinum over a dense, countable set of $a,$ we get 
\[
\inf_a E[c(a,g(\go))\ -\ c(b(\go),g(\go))|b(\go)]\ =\ \phi_\Xi(b(\go))\le 0
\]
and therefore 
\begin{equation}\label{last-c}
\begin{aligned}
&E[c(a(\go),g(\go))\ -\ c(b(\go),g(\go))|b(\go)]\ =\ E[\inf_{a\in \Xi}c(a,g(\go))\ -\ c(b(\go),g(\go))|b(\go)]\\ 
&\le\ \inf_{a\in\Xi}E[c(a,g(\go))\ -\ c(b(\go),g(\go))|b(\go)]= \inf_{a\in \Xi} c(a,b(\go))\le 0\,.
\end{aligned}
\end{equation}
and therefore, integrating over $b,$ we get  
\[
E[c(a(\go),g(\go))]\ \le\  E[c(b(\go),g(\go))]\,. 
\]
The proof is complete. 
\end{proof}

\begin{proposition}\label{uniqueness} Let $a(\go)$ be an optimal policy with support $\Xi.$ Let also $Q_\Xi$ be the set $\{b\in \R^M:\ \phi_\Xi(b)=0\}.$ Then, if $\Xi$ is $X$-maximal for some set $X$ and $\Xi\subseteq X,$ we have $\Xi\subseteq Q_\Xi.$ Furthermore,  if $\tilde a$ is another optimal policy with support $\tilde\Xi,$ then   $\tilde\Xi\subseteq Q_{\Xi}$.

If $\Xi=Q_\Xi$ and $\arg\min_{a\in \Xi}c(a,b)$ is a singleton for all $b\in conv(g(\gO)),$ then the optimal policy is unique. 
\end{proposition}

We conjecture that the conditions in Proposition \ref{uniqueness} hold generically and hence optimal policy is unique for generic $W.$ This is indeed the case in the explicitly solvable examples discussed in Section \ref{sec:conceal}. 

\begin{proof}[Proof of Proposition \ref{uniqueness}] Since $\Xi$ is $W$-monotone, we have $c(a,b)\ge 0$ for all $a,b\in \Xi$ and hence $\phi_\Xi(b)\ge 0$ for all $b\in \Xi.$ Thus, if $\Xi\subset X,$ we get $\phi_\Xi(b)=0$ on $\Xi.$ Let now $a$ be an optimal policy. First we note that \eqref{last-c} implies that $\phi_\Xi(b(\go))=0$ almost surely for any optimal policy $b(\go).$ 

If $\Xi=Q_\Xi,$ we get that $\tilde\Xi\subset \Xi$ and hence $\phi_\Xi(b) \le \phi_{\tilde\Xi}(b)$ for all $b.$  Thus, 
\[
\int c(b(\go),g(\go))\mu_0(\go)d\go\ \ge\ \int\phi_{\tilde\Xi}(g(\go))\mu_0(\go)d\go\ \ge \int\phi_{\Xi}(g(\go))\mu_0(\go)d\go\,.
\]
Since both policies are optimal, we must have $\phi_\Xi=\phi_{\tilde\Xi},$ and the singleton assumption implies that $a(\go)=\tilde a(\go).$ 
\end{proof}

\section{Proofs: Optimal Information Manifold}

\subsection{First Order Conditions}

\begin{proposition}\label{delta-dev} Let $\gamma$ be the joint distribution of $(a,\go)$ for an optimal information design. Then, 
\[
\int (x(a)^\top G(a,\go)- W(a,\go))d\eta\ +\ \int W(\tilde a_*,\go)d\eta(\R^M,\go)\ \le\ 0
\]
for every measure $\eta$ such that $\Supp(\eta)\subset\Supp(\gamma)\,$ such that $\int f(\|a\|^2)g^2(\go)d\eta(a,\go) <\infty.$ 
In the case of moment persuasion, 
\begin{equation}\label{kramkov2019optimal-12}
\int (D_aW(a)(a -g(\go))-W(a))d\eta\ +\ W(\int g(\go)d\eta)\le 0
\end{equation}
\end{proposition}

\begin{proof}[Proof of Proposition \ref{delta-dev}] We closely follow the arguments and notation in \cite{kramkov2019optimal}. Let $\gamma$ be the joint distribution of the random variables $\go$ and $a(\go).$ We first establish \eqref{kramkov2019optimal-12} for a Borel probability measure $\eta$ that has a bounded density with respect to $\gamma.$ Then, the general result follows by a simple modification of the argument in the proof of Theorem A.1 in \cite{kramkov2019optimal}. Let
\[
V(a,\go)\ =\ \frac{d\eta}{d\gamma}(a,\go)\,.
\]
We choose a non-atom $q\in \R^M$ of $\mu(da)\ =\ \gamma(da,\R^L)$ and define the probability measure
\[
\zeta(da,d\go)\ =\ \gd_q(da)\eta(\R^M,d\go)\,,
\]
where $\gd_q$ is the Dirac measure concentrated at $q.$ For sufficiently small $\eps>0$ the probability measure
\[
\tilde\gamma\ =\ \gamma\ +\ \eps (\zeta-\eta)
\]
is well-defined and has the same $\go$-marginal $\mu_0(\go)$ as $\gamma$. Let $\tilde a$ be the optimal action satisfying 
\[
\tilde\gamma (G(\tilde a,\go)|\tilde a)\ =\ 0\,.
\]
The optimality of $\gamma$ implies that 
\begin{equation}\label{kramkov2019optimal-13}
\int W(\tilde a,\go)d\tilde\gamma\ \le\ \int W(a,\go)d\gamma\,. 
\end{equation}
By direct calculation, 
\begin{equation}
\begin{aligned}
&0\ =\ \tilde\gamma (G(\tilde a,\go)|a)\\
&=\ {\bf 1}_{a\not=q}\frac{\int G(\tilde a,\go)d((\gamma|a)-\eps (\eta|a))}{\int d(\gamma-\eps \eta)}\ +\ {\bf 1}_{a=q}\int G(\tilde a,\go) d\eta(\R^M,\go)\\
&=\ {\bf 1}_{a\not=q}\frac{\int G(\tilde a,\go)d(\gamma|a)-\eps\int G(\tilde a,\go)d (\eta|a)}{1-\eps U(a)}\ +\ {\bf 1}_{a=q}\int G(\tilde a,\go) d\eta(\R^M,\go)
\end{aligned}
\end{equation}
where $U(a)=\gamma(V(a,\go)|a)\,.$
Now, we know that 
\[
\int G(a,\go)d(\gamma|a)\ =\ 0,
\]
and the assumed regularity of $G$ together with the implicit function theorem imply that 
\[
\tilde a(a)\ =\ a\ +\ \eps Q(a)\ +\ O(\eps^2)
\]
if $a\not=q$ and
\[
\tilde a\ =\ \tilde a_*\,,
\]
where $\tilde a_*$ is the unique solution to 
\[
\int G(\tilde a_*,\go) d\eta(\R^M,\go)\ =\ 0
\]
for $a=q.$ Here, 
\begin{equation}
\begin{aligned}
&0\ =\ O(\eps^2)\ +\ \int G(a\ +\ \eps Q(a),\go)d(\gamma|a)-\eps\int G(a,\go) V(a,\go) d (\gamma|a)\\
&=\ O(\eps^2)\ +\ \eps \int D_aG(a,\go) d(\gamma|a) Q(a)-\eps\int G(a,\go) V(a,\go) d (\gamma|a)
\end{aligned}
\end{equation}
so that 
\[
Q(a)\ =\ \left(\int D_aG(a,\go) d(\gamma|a)\right)^{-1}\int G(a,\go) V(a,\go) d (\gamma|a)\,. 
\]
Thus, 
\begin{equation}
\begin{aligned}
&\int W(\tilde a(a),\go)d\tilde \gamma\ =\ \int W(\tilde a(a),\go)(1-\eps V(a,\go))d\gamma+\eps \int W(\tilde a_*,\go)d\eta(\R^M,\go)\\
&=\ O(\eps^2)\ +\ \int W(a,\go)d\gamma\ +\ \eps\Bigg(\int (D_aW(a,\go) Q(a) - V(a,\go)) d\gamma\ +\ \int W(\tilde a_*,\go)d\eta(\R^M,\go)
\Bigg)
\end{aligned}
\end{equation}
In view of \eqref{kramkov2019optimal-13}, the first-order term is non-positive:
\[
\int (D_aW(a,\go) Q(a) - W(a,\go)V(a,\go)) d\gamma\ +\ \int W(\tilde a_*,\go)d\eta(\R^M,\go)\ \le\ 0\,.
\]
Substituting, we get 
\[
\int (x(a)^\top \int G(a,\go) V(a,\go) d (\gamma|a)\ - W(a,\go)V(a,\go)) d\gamma\ +\ \int W(\tilde a_*,\go)d\eta(\R^M,\go)\ \le\ 0\,,
\]
which is equivalent to 
\[
\int (x(a)^\top G(a,\go)- W(a,\go))d\eta\ +\ \int W(\tilde a_*,\go)d\eta(\R^M,\go)\ \le\ 0
\]
In the case of a moment persuasion, we get 
\[
Q(a)\ =\ a U(a)\ -\ R(a)\,,
\]
where we have defined 
\[
U(a)\ =\ \gamma(V(a,\go)|a),\ R(a)\ =\ \gamma(g(\go)V(a,\go)|a)\,,
\]
and 
\[
\tilde a_*\ =\ \int g(\go)d\eta\,.
\]
Thus, we get 
\begin{equation}
\begin{aligned}
&0\ \ge\ \int (D_aW(a) Q(a) - W(a)V(a,\go)) d\gamma\ +\ W(\tilde a_*)\\
&=\ \int (D_aW(a)(a U(a)\ -\ R(a))-W(a)V(a,\go))d\gamma\ +\ W(\tilde a_*)\\
&=\ \int (D_aW(a)(a -g(\go))-W(a))d\eta\ +\ W(\int g(\go)d\eta)\,. 
\end{aligned}
\end{equation}
\end{proof}

\begin{lemma}\label{super-G} Let $a_*(\go)$ be the unique solution to $G(a_*(\go),\go)=0.$ Then, at any optimal $(a(\go),\,x(a(\go))$ with $x(a(\go))=\bar D_aW(a(\go))\bar D_aG(a(\go))^{-1}$ we have 
\[
x(a(\go))^\top G(a(\go),\go)\ -\ W(a(\go),\go)\ +\ W(a_*(\go),\go)\ \le\ 0
\]
Furthermore, 
\[
c(a(\go),\go;x)\ =\ \min_{a\in\Xi}\,c(a,\go;x)\,
\]
almost surely.
\end{lemma}

\begin{proof} The first claim follows by selecting $\eta=\gd_{(a(\go),\go)}.$ The second one follows by selecting $\eta=t\gd_{a_1}{\bf 1}_{\gO_1}\gamma|a_1+(1-\kappa t)\gd_{a_2}\gamma|a_2$ for some open set $\gO_1$ and $\kappa=\gamma(\gO_1|a_1)$. In this case, we get 
\begin{equation}\label{aux10}
\begin{aligned}
&t\int (x(a_1)^\top G(a_1,\go)- W(a_1,\go)){\bf 1}_{\gO_1}d\gamma(\go|a_1)\ +\ (1-\kappa t)\int (x(a_2)^\top G(a_2,\go)- W(a_1,\go))d\gamma(\go|a_2)\\
& +\ \int W(\tilde a_*,\go)(t{\bf 1}_{\gO_1}d\gamma|a_1+(1-\kappa t)d\gamma|a_2)\ \le\ 0\,. 
\end{aligned}
\end{equation}
where $\tilde a_*(t)$ is uniquely determined by 
\[
t\int G(\tilde a_*(t),\go){\bf 1}_{\gO_1}d(\gamma|a_1)+(1-\kappa t)\int G(\tilde a_*(t),\go)d(\gamma|a_2)\ =\ 0\,.
\]
Clearly, \eqref{aux10} is equivalent to 
\begin{equation}\label{aux11111}
\begin{aligned}
&t\int (W(\tilde a_*(t),\go)-W(a_1,\go)+x(a_1)^\top G(a_1,\go)){\bf 1}_{\gO_1} d(\gamma|a_1)\\
&+(1-\kappa t)\int (W(\tilde a_*(t),\go)-W(a_2,\go)+x(a_2)^\top G(a_2,\go))d(\gamma|a_2)\ \le\ 0\,. 
\end{aligned}
\end{equation}
Assuming that $t$ is small, we get 
\[
\tilde a_*(t)\ =\ a_2+t \hat a\ +\ o(t),\ \hat a\ =\ -\bar D_aG(a_2)^{-1}\int G(a_2,\go){\bf 1}_{\gO_1}d(\gamma|a_1)
\]
and hence
\begin{equation}\label{aux11}
\begin{aligned}
&0\ge t\int (W(\tilde a_*(t),\go)-W(a_1,\go)+x(a_1)^\top G(a_1,\go)){\bf 1}_{\gO_1} d(\gamma|a_1)\\
&+(1-\kappa t)\int (W(\tilde a_*(t),\go)-W(a_2,\go)+x(a_2)^\top G(a_2,\go))d(\gamma|a_2)\\
&=\ t\int (W(a_2,\go)-W(a_1,\go)+x(a_1)^\top G(a_1,\go)){\bf 1}_{\gO_1} d(\gamma|a_1)\\
&+t\bar D_aW(a_2)\hat a\ +\ o(t)\\
&=\ t\int (W(a_2,\go)-W(a_1,\go)+x(a_1)^\top G(a_1,\go)){\bf 1}_{\gO_1} d(\gamma|a_1)\\
&-t\bar D_aW(a_2)\bar D_aG(a_2)^{-1}\int G(a_2,\go){\bf 1}_{\gO_1}d(\gamma|a_1)\ +\ o(t)\\
&=\ t\int(c(a_1,\go;x)-c(a_2,\go;x)){\bf 1}_{\gO_1}d(\gamma|a_1)\ +\ o(t)\,.
\end{aligned}
\end{equation}
Since $\gO_1$ is arbitrary, we get that 
\[
c(a_1,\go;x)\ \le\ c(a_2,\go;x)
\]
almost surely with respect to $\gamma|a_1.$ 
\end{proof}

\section{Proofs: Properties of Pools}

\begin{corollary}\label{main-conditions} In a moment persuasion setup, let $a(\go)$ be an optimal information design. Then, for any two points $\go_1,\go_2,$ we have 
\begin{equation}\label{28}
\begin{aligned}
&W(t g(\go_1)+(1-t)g(\go_2))\ +\ t (D_aW(a(\go_1))(a(\go_1) -g(\go_1))-W(a(\go_1)))\\ 
&+\ (1-t)(D_aW(a(\go_2))(a(\go_2) -g(\go_2))-W(a(\go_2)))\ \le\ 0\,. 
\end{aligned}
\end{equation}
In particular, in the case of pooling (when $a(\go_2)=a(\go_1)=a),$ we get 
\begin{equation}\label{popool}
W(t g(\go_1)+(1-t)g(\go_2))\ +\ t (D_aW(a)(a -g(\go_1))-W(a))\ +\ (1-t)(D_aW(a)(a -g(\go_2))-W(a))\ \le\ 0
\end{equation}

In the case of $\go_1=\go_2,$ we just get 
\begin{equation}\label{simple-max}
 D_aW(a(\go))(a(\go) -g(\go))\ \le\ W(a(\go))\ -\ W(g(\go))\,.
\end{equation}
Furthermore, 
\begin{equation}\label{conv-W}
W(t a_1+(1-t)a_2)\ \le\ tW(a_1)\ +\ (1-t) W(a_2)\ ,\ a_1,\ a_2\in \Supp(a)\,. 
\end{equation}
In particular, $\Supp(a)$ is a $W$-convex set. 
\end{corollary}

\begin{proof}[Proof of Corollary \ref{main-conditions} and \ref{dimension-pool}] The first claim follows from the choice $\eta=t\gd_{(a(\go_1),\go_1)}+(1-t)\gd_{(a(\go_2),\go_2)}$ in Proposition \ref{delta-dev}. The second one follows from Lemma \ref{super-G}. Monotonicity of the set follows by evaluating the inequality at $t\to 0.$ 
\eqref{conv-W} follows from $\eta=t\gd_{a_1}+(1-t)\gd_{a_2}.$
Maximality follows from \eqref{popool}. 

Now, to prove Corollary \ref{dimension-pool}, we note that \eqref{popool} implies that 
\[
W(x)\ \le\ W(a)\ +\ D_aW(a)(x-a)
\]
for all $x\in conv(Pool(a)).$ 

Without loss of generality, we may assume that $a\in int(conv(g(Pool(a)))).$ Indeed, if not then it is on a face of this set, but since $a$ is the expectation of $g(\go)$, then $\gamma(\go|a)$ is supported on this face (which is itself convex) and hence the dimension claim follows trivially as the dimension is then even lower. 
Since $a\in conv(Pool(a)),$ this implies that $(\eps x +(1-\eps) a-a)^\top D_{aa}W(a)(\eps x +(1-\eps) a-a)\ \le\ 0$ for all $x\in conv(Pool(a))$ and small $\eps.$ The claim follows  from the eigenvalue interlacing theorem. 
\end{proof}

The first claim of Corollary \ref{dimension-pool} is intuitive. By Theorem \ref{cor-moment}, since $\Xi$ is $conv(g(\gO))$-maximal and $a=\cP_\Xi(g(\go))$, we have that $W(x)\ -W(a)\ -\ D_aW(a)(x-a)=c(a,x)\ =\ \min_{\tilde a\in\Xi}c(\tilde a,x)\le 0$ for all $x\in g(Pool(a))$ and it is possible to extend this argument to $conv(g(Pool(a))).$ Thus, $a$ is always a ``peak" of the utility function $W$ inside $conv(g(Pool(a)))$, and $W(x)$ stays below the tangent hyperplane at that peak. The second result, $(x-a)^\top D_{aa}W(a)(x-a)\le 0$, follows directly from $c(a,x)\le 0$ by Taylor approximation. 
The last claim follows from the observation that $a$ always belongs to the interior of the convex set $conv(g(Pool(a)))$ and hence $z'D_{aa}W(a)z\le 0$ should hold for all $z$ in the minimal-dimensional subspace to which $conv(g(Pool(a)))-a$ belongs.\footnote{For example, a line segment passing through the origin in $\R^2$ belong to $\R^2$, but the minimal-dimensional subspace to which it belongs is a line.}

\begin{proof}[Proof of Theorem \ref{thm-unif}] The proof follows directly from $W$-convexity: picking a small ball and using the non-degeneracy of the Hessian, we get tat $W$-convexity implies locally a condition analogous to \eqref{good-bad} for points on the manifold, and then the claim follows by the same arguments as for the quadratic case. 
\end{proof}

We now proceed by showing that the dimension bound is, in fact, exact. 

\begin{proposition}\label{regularity} Let $a(\go)$ be a pure optimal policy and $\Xi$ the corresponding optimal information manifold. Suppose that $D_{aa}W(a)$ is non-degenerate and that, for any $\eps>0,$ $conv(g(Pool(B_\eps(a))))\subset\R^M$ has positive Lebesgue measure. Then, for sufficiently small $\eps>0,$ $\Xi\cap B_\eps(a)$ has Hausdorff dimension $\nu(a).$ 
\end{proposition}

\begin{proof}[Proof of Proposition \ref{regularity}]
First, we show that $\tilde A_{-\nu}\subset \R^{\nu}$ has positive Lebesgue measure. Without loss of generality, we may assume that $X$ is bounded. Pick an $a\in \supp\{a(\go)\}$ and a small ball $B_\eps(a)$ of radius $\eps$ around $a.$ Let $\beta=a(X)\cap B_\eps(a)$ and $Z=conv(g(a^{-1}(\beta))).$ By \eqref{kramkov2019optimal-12}, for any $a_1$ and $a_2$ and any $x_i\in conv(g(a^{-1}(a_i))$ and any $t_i\in [0,1],\ t_1+t_2=1$ we have 
\begin{equation}\label{28-1}
\begin{aligned}
&W(t_1 x_1+t_2 x_2)\ +\ \sum_i t_i (D_aW(a_i)(a_i-x_i)-W(a_i))\ \le 0
\end{aligned}
\end{equation}
Since $a_i\in conv(g(a^{-1}(a_i)),$ we have $y_i=a_i(1-\eps_1)+\eps_1 x_i\in conv(g(a^{-1}(a_i))$ and therefore \eqref{28-1} also holds if we replace $x_i$ with $y_i.$ Then, the Taylor approximation plus the two-times continuous differentiability imply 
\begin{equation}
\begin{aligned}
&D_aW(a_i)(a_i-y_i)-W(a_i)\ =\ -W(y_i)+0.5( y_i-a_i)^\top D_{aa}W(a_i)(y_i-a_i)+ o(\eps^2)\\
&=-(W(a)+D_aW(a)(y_i-a)+0.5(y_i-a)^\top D_{aa}W(a)( y_i-a))\\
&+0.5(y_i-a_i)^\top (D_{aa}W(a)+o(\eps))(y_i-a_i)+o(\eps^2)\\
&=\ -W(a)-D_aW(a)(y_i-a)\ -\ 0.5( y_i-a)^\top D_{aa}W(a)(y_i-a)+0.5(y_i-a_i)^\top D_{aa}W(a)(y_i-a_i)
\end{aligned}
\end{equation}
and \eqref{28-1} implies that, with $\bar y=\sum_i t_i y_i,$ we have 
\begin{equation}\label{28-2}
\begin{aligned}
&0\ \ge\ W(a)\ +\ D_aW(a)(\bar y-a)\ +\ 0.5(\bar y-a)^\top D_{aa}W(a)(\bar y-a)\ +\ o(\eps^2) \\
&+\ \sum_i t_i (-W(a)-D_aW(a)(y_i-a)\ -\ 0.5( y_i-a)^\top D_{aa}W(a)(y_i-a)+0.5(y_i-a_i)^\top D_{aa}W(a)(y_i-a_i))\\
&=\ 0.5(\bar y-a)^\top D_{aa}W(a)(\bar y-a)\\
&+\sum_i t_i (-\ 0.5( y_i-a)^\top D_{aa}W(a)(y_i-a)+0.5(y_i-a_i)^\top D_{aa}W(a)(y_i-a_i))\ +\ o(\eps^2)\\
&=\ 0.5(\bar y-a)^\top D_{aa}W(a)(\bar y-a)\ +\ o(\eps^2)\,. 
\end{aligned}
\end{equation}
Let $H$ be such that $H\le D_{aa}W(a)$ for all $a\in B_\eps(a).$ 
Thus, we get that 
\[
(\bar x -a)^\top H (\bar x-a)\ \le\ K\eps 
\]
for all $\bar x$ in the convex hull of $g(X).$ Without loss of generality, we may assume that $H=\diag({\bf 1}_\nu,-{\bf 1}_{M-\nu})$ and $a=0.$ Then, this condition implies 
\[
\|x_1\|^2\ \le\ K\eps\,+\|x_2\|^2
\]

The following Lemma follows by similar arguments as above. 

\begin{lemma} Suppose that $X$ is a convex set satisfying 
\[
\|x_1\|^2\ \le\ K\eps\,+\|x_2\|^2
\]
Then, its Hausdorff measure is of the order $\eps^\nu.$ 
\end{lemma}

Suppose now that $\beta$ has measure zero. Then, we can select a covering with a union of small balls.

Then, inequality \eqref{28} implies that for any $x=t g(\go_1)+(1-t)g(\go_2)\in X,$ we have 
\begin{equation}\label{28-1-1}
\begin{aligned}
&W(t g(\go_1)+(1-t)g(\go_2))\ +\ t (D_aW(a(\go_1))(a(\go_1) -g(\go_1))-W(a(\go_1)))\\ 
&+\ (1-t)(D_aW(a(\go_2))(a(\go_2) -g(\go_2))-W(a(\go_2)))\ \le\ 0\,. 
\end{aligned}
\end{equation}
Since $a(\go_1)\in conv(g(a^{-1}(a(\go_1)))),$ we have that $(1-\eps)a(\go_1)+\eps g(\go_1)\in $
\[
W(a_1)\ =\ W(a)+D_aW(a)(a_1-a)
\]

\begin{lemma}\label{lem7} For any set $Y\subset\R^\nu$ of positive Lebesgue measure. Then, $span(\{y_1-y_2:\ y_1,\ y_2\in Y\})=\R^\nu$ 
\end{lemma}

\begin{proof} Suppose the contrary. Then, $dim (span(\{y_1-y_2:\ y_1,\ y_2\in Y\}))<\nu.$ Our proof is based on an application of the famous Frostman's lemma (see, e.g., \cite{mattila1999geometry}). 

\begin{lemma}[Frostman's lemma] 
Define the $s$-capacity of a Borel set $A$ as follows:
\[
C_s(A)\ =\ \sup\left\{
\left(
\int_{A\times A}\frac{d\mu(x)d\mu(y)}{\|x-y\|^s}
\right)^{-1}:\ \mu\ \text{is a Borel measure and $\mu(A)=1$}
\right\}\,.
\]
(Here, we take $\inf\emptyset=\infty$ and $1/\infty=0.$) Then, the Hausdorff dimension $\dim_H(A)$ is given by 
\[
\dim_H(A)\ =\ \sup\{s\ge 0:\ C_s(A)\ >\ 0\}\,. 
\]
\end{lemma}

Without loss of generality, we may assume that $Y$ is bounded. 
Since $Y\subset\R^\nu$ has positive Lebesgue measure, its Hausdorff dimension is $\nu.$ Let us embed $S=(span(\{y_1-y_2:\ y_1,\ y_2\in Y\}))$ into $\R^{\nu-1}$. Then, picking a ball $B$ in $\R^{\nu-1}$ such that $B\subset S$ we get that $S\subset B-B$ and hence 
\begin{equation}
\begin{aligned}
&C_s(Y)\ =\ \sup\left\{
\left(
\int_{Y\times Y}\frac{d\mu(x)d\mu(y)}{\|x-y\|^s}
\right)^{-1}:\ \mu\ \text{is a Borel measure and $\mu(Y)=1$}
\right\}\\
& \le\ \sup\left\{
\left(
\int_{B\times B}\frac{d\mu(x)d\mu(y)}{\|x-y\|^s}
\right)^{-1}:\ \mu\ \text{is a Borel measure and $\mu(Y)=1$}
\right\}\ =\ C_s(B)
\end{aligned}
\end{equation}
and hence $\dim_H(Y)\le \dim_H(B)=\nu-1,$ which is a contradiction. The proof of Lemma \ref{lem7} is complete.  

\end{proof}

The proof of Proposition \ref{regularity} is complete. 

\end{proof}

\begin{proof}[Proof of Corollary \ref{char-pools}] The first part follows directly from Proposition \ref{regularity}. To prove the second part, we notice that $c(f(\theta_*), g(\go))\ \le\ c(f(\theta), g(\go))$ for all $\theta\in\Theta.$ By the Rademacher Theorem, $f(\theta)$ is almost everywhere differentiable, and hence we can assume that $f$ is differentiable on the whole of $\Theta.$ Suppose that there exists a subset $\hat\Theta$ of positive Lebesgue measure such that, for each $\theta\in \hat\Theta$, there exists an $\go(\theta)$ such that the first-order condition \eqref{downward-sloping-pools} does not hold: $D_\theta c(f(\theta), g(\go))|_{\go=\go(\theta)}\not=0.$ Without loss of generality, by the compactness of the unit ball, using an $\eps$-net argument, and passing to a subset we may assume that all gradients of $D_\theta c(f(\theta), g(\go))|_{\go=\go(\theta)}$ look approximately in the same direction (i.e., close to the same point on the unit ball). 
This will lead to a violation of the optimality condition $c(f(\theta_*), g(\go))\ \le\ c(f(\theta), g(\go))$ by Lemma \ref{lem7} because the latter implies that, in fact, the set of directions in which we can deviate spans the whole $\R^\nu.$ 
\end{proof}

\section{Proofs of Proposition \ref{conv-bound} and Corollary \ref{concave-bounded}}

\begin{proof}[Proof of Proposition \ref{conv-bound}] Suppose first that $W$ is quadratic. Then, the claim follows directly from the maximality of $\Xi:$ If there exists a hyperplane such that $\Theta$ is on one side of it, then it is possible to extend $f$ preserving its Lipschitz constant beyond this hyperplane using the standard Lipschitz extension argument from the Kirszbraun theorem. See, \cite{kirszbraun1934zusammenziehende}.\footnote{We just pick one point on the other side of the hyperplane and extend $f$ to this point as in \cite{kirszbraun1934zusammenziehende}. Hence, $\Xi$ cannot be maximal.}

Suppose now on the contrary that there exists an unbounded, $W$-convex set $\Xi.$ Let $a_k\to\infty$ be a sequence of points in $\Xi$ and let $\theta_k=a_k/\|a_k\|.$ Passing to a subsequence, we may assume that $\theta_k\to\theta_*.$ Then, 
\[
\inf_{t\in[0,1]}(a_{k+1}-a_k)'D_{aa}W(a_kt+a_{k+1}(1-t))(a_{k+1}-a_k)\ \ge\ 0\,. 
\]
Passing to a subsequence, we may assume that $\|a_{k+1}\|=2^k \|a_k\|$ and the whole sequence stays in $\cC(\theta_*,\eps).$ This is a contradiction. 
\end{proof}

\begin{proof}[Proof of Corollary \ref{concave-bounded}] We just need to show that the conditions of Proposition \ref{conv-bound} are satisfied. We have $D_aW(a)\ =\ 2\varphi'(a^\top H a)Ha$ and $D_{aa}W(a)\ =\ 4\varphi''(a^\top H a)H a a^\top H\ +\ 2\varphi'(a^\top H a) H.$ Thus, 
\[
b^\top D_{aa}Wb\ =\ 4\varphi''(a^\top H a) (a^\top H b)^2\ +\ 2\varphi'(a^\top H a) b^\top H b
\]
and the claim follows because $a^\top H b \approx b^\top H b\approx a^\top H a$ and the first (negative) term dominates when $a^\top H b\to\infty.$ 
\end{proof}

\section{Proofs for Applications Section \ref{sec-applications}}

\begin{proof}[Proof of Proposition \ref{Rayo-Segal1}] 
By direct calculation, 
\begin{equation}\label{cab1-1}
\begin{aligned}
&c(a,b)\ =\ W(b)\ -\ W(a)\ +\ D_aW(a)\,(a-b)\\
& =\ \sum_i (b_1G_i(b_{i+1})-a_1G_i(a_{i+1}))+\sum_i \Big(G_i(a_{i+1})(a_1-b_1)\ +\ a_1G_i'(a_{i+1})(a_{i+1}-b_{i+1})\Big)\\
&=\ 
b_1\sum_i (G_i(b_{i+1})-G_i(a_{i+1}))-a_1\sum_i G_i'(a_{i+1})(b_{i+1}-a_{i+1})\,.
\end{aligned}
\end{equation}
Let $a_{1+}=(a_{i+1})_{i=1}^N$ and $G(a_{1+})=(G_i(a_{1+i}))_{i=1}^N.$ Then, $D_aW\ =\ \binom{G(a_{1+})^\top {\bf 1}}{a_1 G'(a_{1+})}$ and\footnote{We use ${\bf 1}$ to denote a vector of ones.} 
\begin{equation}
D_{aa}W(a)\ =\ \begin{pmatrix}
0&G'(a_{1+})^\top\\
G'(a_{1+})&a_1\diag(G''(a_{1+}))
\end{pmatrix}
\end{equation}
By direct calculation, $D_aW$ is injective if all $G_i''$ the the same sign.\footnote{We have that $a=(D_aW)^{-1}(x)$ satisfies  $\sum_{i=1}^N G(a_{i+1})=x_1,\ a_{i+1}=(G_i')^{-1}(x_{i+1}/a_1)$ and hence there is a unique $a_1$ solving $\sum_{i=1}^N G((G_i')^{-1}(x_{i+1}/a_1))=x_1.$} In this case, convexity of pools can be guaranteed by Corollary \ref{main convexity}. Let $\nu$ be the number of $G_i$ with $G_i''>0.$ By direct calculation, $D_{aa}W(a)$ always has exactly $\nu+{\bf 1}_{q(a)>0}$ positive eigenvalues. As above, convex acceptance probabilities incentivise information revelation and a larger $\nu$ makes it optimal for the sender to reveal more information. 

In the first case, by direct calculation, we have 
\begin{equation}
Df\ =\ \binom{D_af}{I}
\end{equation}
and, hence, 
\[
Df^\top D_{aa}W(a) Df\ =\ f(a)\,\diag(G''(a))\ +\ (G' D_af^\top+D_af\,(G')^\top)\ \ge\ 0\,.
\]
In particular, diagonal elements are 
\[
2G_i'(a_i)f_{a_i}+G_i''(a_i)f\ =\ 2((G_i')^{1/2}f)_{a_i}\ \ge\ 0\,,
\]
implying the required monotonicity. Furthermore, it also implies that 
\[
Q=\diag(f(a)\,G''(a))^{-1}\ +\ \diag(f(a)\,G''(a))^{-1}(G' D_af^\top+D_af\,(G')^\top)\diag(f(a)\,G''(a))^{-1}\ \ge\ 0
\]
and hence the matrix 
\[
\begin{pmatrix}
(G')^\top QG'& (G')^\top Q D_af\\
(G')^\top Q D_af &(D_af)^\top Q D_af
\end{pmatrix}\ \ge\ 0
\]
that is
\[
\begin{pmatrix}
A+2AB&B+AC+B^2\\
B+AC+B^2&C+2CB
\end{pmatrix}\ \ge\ 0\,,
\]
where we have defined 
\begin{equation}
\begin{aligned}
&A=(G'(a))^\top \diag(f(a)G''(a))^{-1}G',\ B=(G'(a))^\top \diag(f(a)G''(a))^{-1}D_af,\\
&C=(D_af)^\top \diag(f(a)G''(a))^{-1}D_af
\end{aligned}
\end{equation}
In particular,
\[
1+(G'(a))^\top \diag(f(a)G''(a))^{-1}D_af\ >\ 0\,.
\]
The pool equation takes the form 
\[
Df^\top D_aaW(a)(\binom{f(a)}{a}\ -\ \binom{\pi}{v})\ =\ 0
\]
which is equivalent to the system 
\[
(f(a)\,\diag(G''(a))+D_af\,(G'(a))^\top)(a-v)\ =\ (\pi-f(a))G'(a)
\]
By the Sherman-Morrison formula, 
\begin{equation}
\begin{aligned}
&(f(a)\,\diag(G''(a))+D_af\,(G'(a))^\top)^{-1}\\
&=\ \diag(f(a)G''(a))^{-1}-\frac{\diag(f(a)G''(a))^{-1}D_af\,(G'(a))^\top \diag(f(a)G''(a))^{-1}}{(1+(G'(a))^\top \diag(f(a)G''(a))^{-1}D_af)}
\end{aligned}
\end{equation}
implying that 
\[
a_i-v_i\ =\ (\pi-f(a))\kappa_i(a)
\]
where 
\[
\kappa_i(a)\ =\ (f(a)G_i''(a_i))^{-1}\Big(G_i'(a_i)-D_{a_i}f\,\frac{((G'(a))^\top \diag(f(a)G''(a))^{-1}G'(a))}{(1+(G'(a))^\top \diag(f(a)G''(a))^{-1}D_af)}\Big)\,. 
\]
Then, 
\begin{equation}
\begin{aligned}
&\sum_i G_i'(a_i)\kappa_i(a)\ =\ ((G'(a))^\top \diag(f(a)G''(a))^{-1}G'(a))\\
&-((G'(a))^\top \diag(f(a)G''(a))^{-1}D_af)\frac{((G'(a))^\top \diag(f(a)G''(a))^{-1}G'(a))}{(1+(G'(a))^\top \diag(f(a)G''(a))^{-1}D_af)}\\
&=\ A-B\frac{A}{1+B}\ =\ \frac{A}{1+B}\ >\ 0\,. 
\end{aligned}
\end{equation}
In the concave $G$ case, we have 
\[
0\ \le\ (1,D_af^\top)\begin{pmatrix}
0&G'(a_{1+})^\top\\
G'(a_{1+})&a_1\diag(G''(a_{1+}))
\end{pmatrix}\binom{1}{D_af}\ =\ \sum_i (2f_i'(a_1)G_i'(f_i(a_1))+a_1(f_i'(a_1))^2G_i''(f_i(a_1)))\,.
\]
In particular, $D_af^\top G'(a)\ge 0.$ 
The pool equation is 
\[
(1,D_af^\top)\begin{pmatrix}
0&G'(a_{1+})^\top\\
G'(a_{1+})&a_1\diag(G''(a_{1+}))
\end{pmatrix}(\binom{a_1}{f(a_1)}\ -\ \binom{\pi}{v})\ =\ 0\,,
\]
that is 
\[
G'(a_{1+})^\top (f(a_1)-v)\ +\ D_af^\top G'(a) (a_1-\pi)+D_af^\top  a_1G''(a) (f(a_1)-v)\ =\ 0\,
\]
\end{proof}

We now assume that $G_i'(x)>\eps$ for some $\eps>0$ and define  $\varphi_i(b)$ to be the unique monotone increasing to the differential equation 
\begin{equation}
\varphi_i(x)'\ =\  (G_i'(\varphi_i(x)))^{-1/2},\ \varphi_i(0)=0\,, \varphi(b)\ =\ (\varphi_i(b_i))_{i=1}^N:\ \R^N\to\R^N\,.
\end{equation}
Define $\widetilde f(a)\ =\ \diag((G'(a))^{1/2})\,f(a)$ and $\widehat f(x)\ =\ \widetilde f(\varphi(x))\,.$
Note that when $N=1,$ we have $\widetilde f(a)=(G'(a))^{1/2} f(a)$ and $\tilde f(a)$ is monotonic if and only if so is $\widehat f.$ However, in multiple dimensions this is not the case anymore: It might happen that $\widehat f$ is monotonic, while $\widetilde f$ is not. 
Let $a_{N+}=(a_{N+i})_{i=1}^N\in \R^N,\ a_{-N}=(a_{i})_{i=1}^N$ be the vectors of expected values of the different products for the sender and the receivers, respectively.  The following is true. 

\begin{proposition} \label{arbitraryGi} There always exists a pure optimal policy $a(\go).$ For each such policy, there exists a map $f=(f_i)_{i=1}^N:\ \R^N\to\R^N$ such that $a_{-N}(\go)=f(a_{N+}(\go))$ for all $\go$ and, hence, the optimal information manifold $\Xi$ is the $N$-dimensional graph $\{(f(a_{N+}),a_{N+})\}$ of the map. Furthermore, the map $\widehat f(a_{N+})$ is monotone increasing. The pool of Lebesgue-almost every signal $a_{1+}$ is given by the $N$-dimensional hyperplane 
\begin{equation}\label{optimal-rs-2-gen}
Pool(a_{N+})\ =\ \{\binom{\pi}{v}:\ \pi\ =\ \kappa_1(a_{N+})\,v\ +\ \kappa_2(a_{N+})\ for\ all\ i\,\}\ \subset\ R^{2N}\,. 
\end{equation}
where 
\[
\kappa_1(a_{N+})\ =\ -(\diag(G')^{-1}(D_af)^\top \diag(G')  +\diag(f G'' /G'))\ \in \R^{N\times N}
\]
and 
\[
\kappa_2(a_{N+})\ =\ \diag(f G'){\bf 1}\ -\ \kappa_1(a_{N+})a_{N+}\ \in\ \R^N\,. 
\]
Furthermore, the matrix $\diag(G'(a_{N+}))\kappa_1(a_{N+})+0.5 \diag(f G'')$ is negative semi-definite. 
\end{proposition}

Monotonicity of the map $\widehat f$ is the multi-dimensional analog of simple coordinate-wise monotonicity of Propositions \ref{Rayo-Segal} and \ref{Rayo-Segal1}. The intuition behind this monotonicity is clear: It is not optimal to pool two strictly ordered prospects; pooling only makes sense to create substitution between projects of a similar acceptance rate. In our setting, however, the nature of prospects ordering along the optimal information manifold is more subtle. The monotonicity of $\widehat f$ implies that 
\[
(x-y)^\top (\widehat f(x)-\widehat f(y))\ \ge\ 0
\]
for any $x,\ y.$ Thus 
\[
0\ \le\ (\varphi^{-1}(a)-\varphi^{-1}(b))^\top (\widetilde f(a)-\widetilde f(b))\ =\ \sum_i (\varphi_i^{-1}(a_i)-\varphi_i^{-1}(b_i))((G_i'(a_i))^{1/2}f_i(a)-(G_i'(b_i))^{1/2}f_i(b))
\]
That is, for any signal $s,$ the vectors 
\[
\binom{(G_i'(E[v_i|s]))^{1/2} E[\pi_i|s]}{\varphi_i^{-1}(E[v_i|s])}
\]
are aligned across prospects, but only an average across prospects. 

Proposition \ref{arbitraryGi} also implies that linearity of pools is preserved in this multiple prospects case. However, the ``downward sloping" property becomes more subtle. What kind of prospects get pooled together depends on the {\it endogenous substitutability between prospects}. First, suppose that $G_i''\ge 0$ for all $i.$ Then, the matrix $\diag(G'(a_{N+}))\kappa_1(a_{N+})$ is negative semi-definite. As a result, for any two prospects $\binom{\pi}{v},\ \binom{\tilde\pi}{\tilde v}\in Pool(a_{N+})$ we have 
\[
0\ \ge\ (\pi-\tilde\pi)^\top \diag(G'(a_{N+})) (v-\tilde v)\ =\ \sum_i (\pi_i-\tilde\pi_i) \diag(G'(a_{N+i})) (v_i-\tilde v_i)
\]
Thus, it is optimal to bundle multiple prospects as long as at least some of them are sufficiently mis-aligned. This cross-compensation across different prospects differs drastically from the ``single-prospect-with-multiple-receivers" case of Proposition \ref{arbitraryGi} where pooling only happens when sender's payoff is mis-aligned with every single payoff of the receivers.

\begin{proof}[Proof of Proposition \ref{Rayo-Segal2} and \ref{arbitraryGi}] We have 
\begin{equation}\label{cab1-2}
\begin{aligned}
&c(a,b)\ =\ W(b)\ -\ W(a)\ +\ D_aW(a)\,(a-b)\\
& =\ \sum_i (b_i G_i(b_{i+N})-a_i G_i(a_{i+N}))+\sum_i \Big(G_i(a_{i+N})(a_i-b_i)\ +\ a_i G_i'(a_{i+N})(a_{i+N}-b_{i+N})\Big)\\
&=\ \sum_i b_i (G_i(b_{i+N})-G_i(a_{i+N}))-\sum_i a_i G_i'(a_{i+N})(b_{i+N}-a_{i+N})\\
&=\ \sum_i  (f_i(b_{N+})-f_i(a_{N+}))(G_i(b_{i+N})-G_i(a_{i+N}))\\
&+\sum_i f_i(a_{N+}) (G_i(b_{i+N})-G_i(a_{i+N})-G_i'(a_{i+N})(b_{i+N}-a_{i+N}))\\
\end{aligned}
\end{equation}
When $b_{N+}\to a_{N+}$ and $da=b_{N+}-a_{N+},$ we get 
\[
0=\ (da)^\top \diag(G'(a)) D_af\,(da)\ +\ 0.5 (da)^\top \diag(f(a) G''(a)) (da)\,. 
\]
where $(D_af)_{i,j}=\partial f_i/\partial a_j.$ That is, the matrix 
\[
\diag(G'(a)) D_af\ +\ 0.5 \diag(f(a) G''(a))
\]
is positive semi-definite. Therefore, so is the matrix 
\[
Q(a)\ =\ \diag(G'(a)^{1/2}) D_af \diag(G'(a)^{-1/2})\ +\ 0.5 \diag(f(a) G''(a)G'(a)^{-1})
\]
Let now $a=\varphi(x).$ Then, by direct calculation, $Q(\varphi(x))$ is the Jacobian of $\widehat f(x)$ and the claim follows because a map is monotone increasing if and only if its Jacobian is positive semi-definite. Since monotone maps are differentiable Lebesgue-almost surely, we get the first order condition 
\[
\max_{a_{N+}}\Bigg(\sum_i \pi_i (G_i(v_{i})-G_i(a_{i+N}))-\sum_i f_i(a_{N+}) G_i'(a_{i+N})(v_i-a_{i+N})\Bigg)
\]
takes the form 
\[
-\diag(G')\pi\ -\ ((D_af)^\top \diag(G')  +\diag(f G'')) (v-a)\ +\ \diag(f G'){\bf 1}\ =\ 0
\]
and the claim follows. 
\end{proof}

\begin{proof}[Proof of Proposition \ref{tamura}] 
In this case, Theorem \ref{converse} implies that, for any optimal policy, $\Xi$ has to be monotonic, meaning that $(a_1-a_2)^\top H(a_1-a_2)\ge 0$ for all $a_1,a_2\in \Xi.$ The question we ask is: Under what conditions is $a(\go)=A\go$ with some matrix $A$ of rank $M_1\le M$ is optimal with $g(\go)=\go.$ Clearly, it is necessary that $\mu_0$ have linear conditional expectations,\footnote{This is, e.g., the case for all elliptical distributions, but also for many other distributions. See \cite{wei1999linear}.} $E[\go|A\go]\ =\ A\go$.  But then, since $E[A\go|A\go]=A\go,$ we must have $A^2=A,$ so that $A$ is necessarily a projection. Maximal monotonicity implies that $Q=A^\top H A$ is positive semi-definite, and\footnote{Here, $Q^{-1}$ is the Moore-Penrose inverse.} 
\[
\phi(b)\ =\ \min_{\go}(b^\top Hb+\go^\top A^\top H (A \go-2b))\ =\ b^\top HAQ^{-1}A^\top A(AQ^{-1}A^\top H -2Id)b
\]
with the minimizer $Q^{-1}A^\top H b.$ Thus, $A$ satisfies the fixed point equation $A\ =\ Q^{-1}A^\top H$ and hence $A^\top=HAQ^{-1}.$ Furthermore, maximality of $\Xi$ implies that 
\[
H\ +\ HAQ^{-1}A^\top A(AQ^{-1}A^\top H -2Id)\ =\ H-A^\top A
\]
is negative semi-definite. As a result, $(Id-A^\top)(H-A^\top A)(Id-A)=(Id-A^\top)H(Id-A)$ is also negative semi-definite, implying that $A$ and $Id-A$ ``perfectly split" positive and negative eigenvalues of $H$. Here, it is instructive to make two observations: First, optimality requires that $a(\go)$ ``lives" on positive eigenvalues of $H$. Second, maximality (the fact that $\phi(b)\le 0$ for all $b$) requires that $A$ absorbs all positive eigenvalues, justifying the term ``maximal". \end{proof}

\section{The Integro-Differential Equation}

\begin{proposition}\label{prop-change-variables} Let $F$ be a  bijective, bi-Lipshitz map,\footnote{A map $F$ is bi-Lipschitz if both $F$ and $F^{-1}$ are Lipschitz continuous.} $F:\ X\to \gO$ for some open set $X\subset \R^M.$ Let also $M_1\le M$ and $x=(\theta,r)$ with $\theta\in X_1$, the projection of $X$ onto $\R^{M_1}$ and $r\in X_2,$ the projection of $X$ onto $\R^{L-M_1}.$ Define 
\begin{equation}\label{cond-f}
\begin{aligned}
&f(\theta)\ =\ f(\theta;F)\ \equiv\ \frac{\int_{X_{2}} |\det(D_\theta F(\theta,r))|\mu_0(F(\theta,r))\,g(F(\theta,r))dr}
{\int_{X_{2}} |\det(D_\theta F(\theta,r))|\mu_0(F(\theta,r))dr}\,.
\end{aligned}
\end{equation}
Suppose that $f$ is an injective map, $f:X_1\to \R^{M}$ and define 
\begin{equation}\label{opt-dif}
\phi(b)\ =\ \min_{\theta\in X_1}\{W(b)-W(f(\theta))+D_aW(f(\theta))^\top(f(\theta)-b)\}\,.
\end{equation}
Suppose also that the min in \eqref{opt-dif} for $b=g(F(\theta,r))$ is attained at $\theta$ and that
$\phi(b)\ \le\ 0\ \forall\ b\in conv(g(\gO)).$ Then, $a(\go)=f((F^{-1}(\go))_1)$ is an optimal policy. If $x_1=\arg\min$ in \eqref{opt-dif} with $b=F(\theta,r)$ for all $\theta,r$ and $\phi(b)<0$ for all $b\in conv(g(\gO))\setminus \Xi$ with $\Xi=f(X_1),$ then the optimal policy is unique. 

If $f$ is Lipshitz-continuous and the minimum in \eqref{opt-dif} is attained at an interior point, we get a system of second order partial integro-differential equations for the $F$ map: 
\begin{equation}\label{pde}
D_{\theta}f(\theta)^\top D_{aa}W(f(\theta))(f(\theta)-g(F(\theta,r)))\ =\ 0\,.
\end{equation}
\end{proposition}

\begin{proof}[Proof of Proposition \ref{prop-change-variables}] Let $\Xi=f(X_1)\cap conv(g(\gO)).$ 

Clearly, $\Xi$ is an $M_1$-dimensional manifold, and we need to verify that $a(\go)=f((F^{-1}(\go))_1)$ satisfies the three conditions of  Theorem \ref{converse}: 
\begin{itemize}
\item $a(\go)\ =\ E[g(\go)|a(\go)]$

\item $a(\go)\ \in \cP_\Xi(g(\go))$ for all $\go$

\item $\Xi$ is $conv(g(\gO))$-maximal. 
\end{itemize}
The first condition is equivalent to \eqref{cond-f} by the change of variables formula. The second condition is equivalent to the fact that the minimum in \eqref{opt-dif} is attained for $b=g(F(\theta,r))$. The third condition follows from the fact that $\phi(b)\ \le\ 0\ \forall\ b\in conv(g(\gO)).$
\end{proof}

\newpage

\bibliographystyle{aer}
\bibliography{bibliography}

\newpage

\appendix

\end{document}